\documentclass[10pt]{article} 
\usepackage[preprint]{tmlr}

\usepackage[
    hypertexnames=false,
    hidelinks
    ]{hyperref}
\usepackage{url}

\usepackage{amsmath}
\usepackage{amssymb}
\usepackage{mathtools}
\mathtoolsset{showonlyrefs,showmanualtags} 
\usepackage{amsthm}
\usepackage{bbm}
\usepackage{tikz}
\usetikzlibrary{arrows.meta, positioning, calc}
\usepackage{tikz-cd}

\usepackage{thmtools}
\usepackage{thm-restate}

\usepackage{zref-clever} 

\declaretheorem[name=Theorem, numberwithin=section]{theorem}
\declaretheorem[name=Proposition, sibling=theorem]{proposition}
\declaretheorem[name=Lemma, sibling=theorem]{lemma}
\declaretheorem[name=Corollary, sibling=theorem]{corollary}
\declaretheorem[name=Definition, sibling=theorem, style=plain]{definition}
\declaretheorem[name=Assumption, sibling=theorem, style=plain]{assumption}
\declaretheorem[name=Remark, sibling=theorem, style=remark]{remark}

\usepackage[textsize=tiny]{todonotes}
\usepackage{enumitem}

\newcommand{\E}{\mathbb{E}}
\newcommand{\V}{\mathbb{V}}
\newcommand{\R}{\mathbb{R}}

\newcommand{\bN}{\mathbb{N}}

\newcommand{\bP}{\mathbb{P}}

\newcommand{\cA}{\mathcal{A}}
\newcommand{\cB}{\mathcal{B}}

\newcommand{\cD}{\mathcal{D}}

\newcommand{\cF}{\mathcal{F}}
\newcommand{\cG}{\mathcal{G}}

\newcommand{\cN}{\mathcal{N}}
\newcommand{\cO}{\mathcal{O}}
\newcommand{\cP}{\mathcal{P}}

\newcommand{\cR}{\mathcal{R}}
\newcommand{\cS}{\mathcal{S}}
\newcommand{\cT}{\mathcal{T}}
\newcommand{\cU}{\mathcal{U}}

\newcommand{\cX}{\mathcal{X}}
\newcommand{\cY}{\mathcal{Y}}
\newcommand{\cZ}{\mathcal{Z}}

\newcommand{\rbar}{\bar r}

\newcommand{\thetahat}{\hat{\theta}}

\newcommand{\thetaE}{\theta^{\mathsf{E}}}
\newcommand{\SigmaE}{\Sigma^{\mathsf{E}}}
\newcommand{\thetastar}{\theta_{\star}}
\newcommand{\Hstar}{\mathsf{H}_{\star}}
\newcommand{\dstar}{d_{\star}}
\newcommand{\lambdastar}{\lambda_{\star}}
\newcommand{\phihat}{\hat{\phi}}
\newcommand{\rtest}{r^{\mathsf{test}}}

\DeclareMathOperator*{\argmin}{arg\,min}
\DeclareMathOperator*{\esssup}{ess\,sup}

\newcommand{\ind}{\mathbbm{1}}

\DeclareMathOperator{\tr}{tr}

\DeclareMathOperator{\ima}{im}
\newcommand{\eqdef}{=\vcentcolon}
\newcommand{\defeq}{\vcentcolon=}

\newcommand{\diff}{\mathop{}\!\mathrm{d}}

\usepackage{accents}

\newcommand{\DKL}{D_{\mathsf{KL}}}

\newcommand{\DHel}{D_{\mathsf{H}}}

\newcommand{\Nlog}{N_{\mathsf{log}}}

\DeclarePairedDelimiter{\bc}{\{}{\}}      
\DeclarePairedDelimiter{\br}{(}{)}         
\DeclarePairedDelimiter{\bs}{[}{]}        
\DeclarePairedDelimiter{\abs}{\lvert}{\rvert} 
\DeclarePairedDelimiter{\norm}{\lVert}{\rVert}
\DeclarePairedDelimiter{\ip}{\langle}{\rangle}

\newcommand{\lambdamin}{\lambda_{\mathsf{min}}}

\newcommand{\cSA}{\cS\times\cA}

\newcommand{\muhat}{\hat{\mu}}
\newcommand{\piE}{\pi^{\mathsf{E}}}

\newcommand{\pihat}{\hat\pi}
\newcommand{\rhat}{\hat r}

\newcommand{\pistar}{\pi^{\star}}
\newcommand{\pstar}{p^{\star}}
\newcommand{\ptraj}{p_{\mathsf{traj}}}
\newcommand{\lambdatraj}{\lambda_{\mathsf{traj}}}
\newcommand{\Jstar}{J^{\star}}
\newcommand{\Vstar}{V^{\star}}
\newcommand{\Qstar}{Q^{\star}}

\newcommand{\LIRL}{L^{\mathsf{MM}}}

\newcommand{\LMLE}{L^{\mathsf{MLE}}}

\newcommand{\LhatIRL}{\hat{L}^{\mathsf{MM}}}

\newcommand{\LhatMLE}{\hat{L}^{\mathsf{MLE}}}

\newcommand{\Lhat}{\hat{L}}

\newcommand{\Ehat}{\hat{\E}}

\newcommand{\gap}{\operatorname{Gap}}

\newcommand{\DE}{\cD^{\mathsf{E}}}

\newcommand{\proj}{\operatorname{Proj}}

\newcommand{\varepsilonopt}{\varepsilon_{\mathsf{opt}}}

\newcommand{\ie}{\textnormal{i.e.}}
\newcommand{\eg}{\textnormal{e.g.}}
\newcommand{\cf}{\textnormal{cf.}}

\title{Fast Rates for Inverse Reinforcement Learning}

\author{\name Andreas Schlaginhaufen \email andreas.schlaginhaufen@epfl.ch \\
      \addr \addr Sycamore, EPFL
      \AND
      \name Maryam Kamgarpour \email maryam.kamgarpour@epfl.ch \\
      \addr Sycamore, EPFL}

\def\month{MM}  
\def\year{YYYY} 
\def\openreview{\url{https://openreview.net/forum?id=XXXX}} 

\begin{document}

\maketitle
\lhead{}

\begin{abstract}
We establish novel structural and statistical results for
entropy-regularized min-max inverse reinforcement learning
(Min-Max-IRL) in finite-horizon MDPs with
Borel state and action spaces. We show that
maximum likelihood estimation (MLE) and Min-Max-IRL are equivalent at
the population level, and at the empirical level under deterministic
dynamics. For linear reward classes, we leverage pseudo-self-concordance
of the Min-Max-IRL loss to prove that both the excess trajectory-level KL
divergence and the squared parameter error in the Hessian norm decay
at the fast rate $\cO(n^{-1})$, where $n$ is the number of expert
trajectories. A local minimax lower bound matches the parameter-error
rate up to logarithmic factors in the well-specified deterministic
setting. Our guarantees apply under misspecification and require no
uniform state-coverage assumption. We further extend reward-identifiability
results to general Borel spaces and compare our results with MLE-based guarantees.
\end{abstract}

\addtocontents{toc}{\protect\setcounter{tocdepth}{-1}}

\section{Introduction}

In many sequential decision-making problems, including robotic manipulation or language model alignment, the reward
function is unknown, difficult to specify, or too sparse to be
practical, and it is easier to demonstrate desired behavior. This motivates the imitation learning setting, where the goal is to infer a policy from a dataset of expert demonstrations. The most direct approach to imitation learning is behavioral cloning~(BC), which fits a policy
to the data by supervised learning without requiring access to
the environment. Inverse reinforcement learning~(IRL) instead recovers a reward that induces an optimal policy that imitates the expert's behavior within the MDP. By
parameterizing rewards rather than policies, IRL can encode
structural priors, such as linearity in
features, sparsity, or stability
\citep{kalman1964linear}.

As expert demonstrations are costly, a central question in imitation
learning is how many expert trajectories, $n$, are required to imitate the
expert well in a given metric. This is a
classical statistical estimation problem. For BC, the statistical
picture is fairly complete:
\citet{foster2024behavior} show that
maximum likelihood estimation (MLE) ensures that the trajectory-level squared
Hellinger distance between the expert and the imitation policy decays as
$\cO(n^{-1})$ for well-specified and sufficiently learnable policy classes, and
that this Hellinger guarantee in turn yields tight bounds on the performance gap under an unknown test reward.

Since IRL has access to the MDP while BC does not, one might expect stronger statistical guarantees. However, existing IRL results give slow
$\cO(n^{-1/2})$ rates with quadratic horizon dependence for the excess risk \citep{syed2007game, schlaginhaufen2023identifiability, zeng2023demonstrations}, require a uniform state-coverage assumption for reward parameter recovery \citep{schlaginhaufen2024towards}, and are largely confined to finite state-action
spaces. Moreover, entropy regularization, central to modern IRL algorithms
\citep{ziebart2010modeling, ho2016generative, garg2021iq} and to reward
identifiability \citep{cao2021identifiability}, is poorly understood from a statistical perspective.

In this paper, we close this gap by establishing improved statistical results for entropy-regularized IRL in general Borel state and action spaces. We establish a connection
between MLE and entropy-regularized min-max IRL, and show that for linear reward classes the latter achieves fast $\cO(n^{-1})$ rates for the excess
trajectory-level KL risk and parameter estimation. Our upper bounds apply under misspecification, require no uniform state-coverage assumption, and can yield horizon-independent rates in favorable settings. A matching local minimax lower bound shows that these rates cannot be further improved in the well-specified deterministic setting. Furthermore, we also extend existing results on reward identifiability from the tabular setting to general state and action spaces, and derive likelihood-ratio-based guarantees for MLE IRL and compare them with our min-max IRL bounds.

\paragraph{Contributions}
Our main contributions can be summarized as follows:
\begin{enumerate}[leftmargin=*]
    \item 
    In Section~\ref{sec:equivalences}, we show that MLE and
    entropy-regularized min-max IRL are equivalent at the population level
    (Theorem~\ref{thm:equivalences}). At the empirical level, the
    equivalence holds under deterministic dynamics, whereas under stochastic
    dynamics an additional correction term arises from the randomness of the
    state transitions. We also show that min-max IRL is a
    convex optimization problem, while MLE IRL is nonconvex in general. In
    Appendix~\ref{app:sec:identifiability_and_potential_shaping}, we further
    extend reward-identifiability results to Borel state and action spaces.

    \item 
    In Section~\ref{sec:statistical_guarantees}, we derive improved
    statistical guarantees for min-max IRL with linear reward classes in
    Borel state and action spaces. In particular, we establish
    pseudo-self-concordance \citep{ostrovskii2021finite} of the entropy-regularized min-max IRL
    loss to prove that both the excess trajectory-level KL risk and the
    squared parameter error in the Hessian norm decay as $\cO(n^{-1})$ in the
    number of expert demonstrations $n$ (Theorem~\ref{thm:irl_fast_rate}).
    Our results apply under misspecification, and we make no uniform state-coverage
    assumption.

    \item
    In Section~\ref{sec:lower_bound}, we prove a local minimax lower bound of
    order $\Omega(n^{-1})$ for the squared parameter error in the Hessian norm
    (Theorem~\ref{thm:irl_lower_bound}), matching the upper bound up to
    logarithmic factors in the well-specified deterministic setting.
    Furthermore, Section~\ref{sec:imitation_implications} compares our result
    with guarantees for MLE and discusses implications for the performance gap.
\end{enumerate}

\section{Problem Statement}\label{sec:problem_statement}

\subsection{Notation}
For a measurable space $\cX$, we use $\Delta_{\cX}$ to denote the set of probability measures on $\cX$. For $P\in\Delta_{\cX}$ and measurable $f:\cX\to\R$ and $g:\cX\to\R^d$, we write $\ip*{f,P}\defeq\int f \diff P=\E_{X\sim P}\bs{f(X)}$, whenever the integral is well defined, and $\operatorname{Cov}_{X\sim P}\bs{g(X)}\defeq\E_{X\sim P}\bs{\br{g(X)-\mu}\br{g(X)-\mu}^{\top}}$ with $\mu = \E_{X\sim P}\bs{g(X)}$, whenever $g$ is square-integrable. For two measures $P, Q $, we write $P\ll Q$ if $P$ is absolutely continuous with respect to $Q$, and $\diff P/\diff Q$ for the density of $P$ with respect to $Q$. For $x,y\in\R^d$, $\ip*{x,y}\defeq x^\top y$ denotes the standard inner product and $\norm*{x}\defeq\sqrt{\ip*{x,x}}$ the Euclidean norm. For symmetric matrices $A,B\in\R^{d\times d}$, we write
$A\succeq B$ and $A\succ B$ for the Loewner order, and
$I_d$ for the identity matrix in $\R^{d\times d}$.
For two nonnegative sequences $a_n,b_n$, we write $a_n\lesssim b_n$
if $a_n\leq C b_n$ for some constant $C>0$, and
$a_n\asymp b_n$ if $a_n\lesssim b_n$ and $b_n\lesssim a_n$.

\subsection{MDP Setup}
We consider a finite-horizon MDP defined by a tuple
$\br{T, \cS, \cA, (\bP_t)_{t=0}^{T-1}, r}$, where $T\in\bN$ is the horizon,
$\cS$ and $\cA$ are the (potentially infinite) state and action spaces,
$\bP_0 \in \Delta_{\cS}$ is the initial state distribution,
$\bP_t:\cSA\to\Delta_{\cS}$ is the transition law at time $t=1,\hdots,T-1$, and
$r = (r_t)_{t=1}^T$ is a sequence of reward functions $r_t:\cSA\to\R$.
A Markov policy is a sequence $\pi=(\pi_t)_{t=1}^T$ of stochastic kernels
$\pi_t:\cS\to\Delta_{\cA}$, and it induces a trajectory law $\bP^{\pi}$ via
$s_1\sim\bP_0$, $a_t\sim\pi_t(\;\cdot \mid s_t)$, and
$s_{t+1}\sim\bP_t(\;\cdot \mid s_t,a_t)$. We write $\E^{\pi}$ for the expectation under $\bP^{\pi}$, and for measurable $f:\cS\to\R$, we write
$\br{\bP_t f}(s,a)\defeq\E_{s'\sim\bP_t(\;\cdot \mid s,a)}\bs{f(s')}$ for the
expectation over the next state. The time-$t$ occupancy measure
$\mu_t^{\pi}\in\Delta_{\cSA}$ is the marginal law of $(s_t,a_t)$ under
$\bP^{\pi}$. We denote
$\ip{r,\mu^{\pi}}\defeq\sum_{t=1}^T\ip{r_t,\mu^{\pi}_t} = \sum_{t=1}^T \E^{\pi}\bs{r_t(s_t, a_t)}$. Throughout, $\cS$ and $\cA$ are Borel
subsets of complete separable metric spaces, equipped with their Borel
$\sigma$-fields, and all rewards, policies, and transition kernels are
measurable, and rewards are in addition bounded. For ease of exposition, we treat these regularity conditions as standing assumptions and defer measure-theoretic details to Appendix~\ref{app:sec:preliminaries}.

Given a finite reference measure\footnote{For example, the counting measure for a finite action space or the Lebesgue measure for a compact one.} $\lambda$ on $\cA$ and $\beta \geq 0$, we consider the entropy-regularized objective
\begin{equation}\label{eq:regularized_return}
    J^{\beta}(r,\pi) \defeq \ip*{r,\mu^{\pi}} + \beta\,H(\pi),
    \qquad
    H(\pi) \defeq \sum_{t=1}^T \E^{\pi}\bs{H(\pi_t(\;\cdot \mid s_t))},
\end{equation}
where $H:\Delta_{\cA}\to [-\infty, \infty]$ denotes the entropy defined by $H(P) = - \int \log \br{\diff P/ \diff \lambda} \diff P$ if $P \ll \lambda $ and $H(P) = - \infty$ otherwise. We write $V_{t, r}^{\pi, \beta}(s)=\sum_{k=t}^T \E^{\pi}\bs{r_k(s_k, a_k) + \beta H(\pi_k(\;\cdot \mid s_k)) \mid s_t =s}$ and $Q_{t, r}^{\pi, \beta}(s, a) = r_t(s,a) + \br{\bP_t V_{t+1, r}^{\pi, \beta} }(s,a)$ for the associated value and Q-functions, and $J^{\star, \beta}(r), V^{\star, \beta}_r$, $Q^{\star, \beta}_r$ for the corresponding optimal values. For $\beta>0$, there is a unique
Bellman-optimal policy $\pi^{\star,\beta}_r$ (Appendix~\ref{app:sec:optimality}),
which we call \emph{soft-optimal}. It admits densities $\br{p^{\star,\beta}_{t,r}}_{t=1}^T$ with respect to $\lambda$ given by the explicit
Gibbs form,
\begin{equation}\label{eq:soft_opt_density}
    \pi^{\star,\beta}_{t,r}(\diff a \mid s) = p^{\star,\beta}_{t,r}(a \mid s)\,\lambda(\diff a),
    \qquad
    p^{\star,\beta}_{t,r}(a \mid s) = \exp \br*{\beta^{-1}\br*{Q^{\star,\beta}_{t,r}(s,a) - V^{\star,\beta}_{t,r}(s)}}.
\end{equation}
Beyond uniqueness of the optimal policy, $\beta > 0$ also ensures that the values $J^{\star,\beta}(r), V^{\star,\beta}_r, Q^{\star,\beta}_r$ are differentiable in $r$
(Appendix~\ref{app:sec:der_optimal_value}). For readability, we henceforth fix
$\beta>0$ and drop the superscript $\beta$, reserving the superscript $0$ for
unregularized quantities, that is, we write for example $V^{\pi,0}_{t,r}$ or $Q_{t,r}^{\pi,0}$.

\subsection{Imitation Learning}
Given a dataset of expert demonstrations
\begin{equation}
    \DE = \bc{\tau^i}_{i=1}^n, \quad \tau^i = \br*{s_1^i,a_1^i, \hdots, s_T^i, a_T^i} \sim \bP^{\piE} \text{ i.i.d.},
\end{equation}
generated by an unknown expert policy $\piE$, the goal of imitation learning is to recover a policy that \textit{imitates} the expert. We consider two approaches. Behavioral cloning
(BC) fits a sequence of policy densities $p=(p_t)_{t=1}^T$ from a class $\cP$
directly to the data, where each $p\in\cP$ induces a policy
via $\pi_{t, p}(\diff a \mid s)=p_t(a \mid s)\,\lambda(\diff a)$. Inverse reinforcement
learning (IRL) instead recovers a reward $r$ from a class $\cR$ such that the
soft-optimal policy $\pistar_r$ imitates the expert. We write $\pihat$ for the policy
recovered from an estimate $\hat p$ or $\rhat$, and call the problem
\emph{well-specified} if the expert trajectory law $\bP^{\piE}$ is realizable, \ie, there is
$p\in\cP$ with $\bP^{\piE}=\bP^{\pi_p}$, or $r\in\cR$ with $\bP^{\piE}=\bP^{\pistar_r}$, and
\emph{misspecified} otherwise.

The quality of imitation can be measured in different ways. We next introduce the two imitation metrics that we will analyze in this paper.

\paragraph{Trajectory-Level Divergences}
A natural measure of imitation is a divergence between the trajectory distributions $\bP^{\piE}$ and $\bP^{\pihat}$ induced by the expert and recovered policy.  We consider the Kullback--Leibler (KL) divergence and the squared Hellinger distance,
\begin{equation}\label{eq:traj_divergences}
    \DKL(\bP^{\piE}, \bP^{\pihat}), \quad \text{and} \quad \DHel^2(\bP^{\piE},\bP^{\pihat}),
\end{equation}
defined as $\DKL(P, Q) \defeq \int \log\br*{p/q} \,\diff P$ and $\DHel^2(P,Q) \defeq \int \br*{\sqrt{p}-\sqrt{q}}^2 \,\diff m$, where $p$ and $q$ are densities of $P$ and $Q$ with respect to a common dominating measure $m$ (\eg, $m = \tfrac{1}{2} \br{P + Q}$).\footnote{Note that $\DHel^2(P,Q)\in[0,2], \DKL(P, Q)\in[0,\infty]$, and $\DHel^2(P,Q) \leq \DKL(P, Q)$. }

\paragraph{Reward Estimation Error}
In IRL, one may additionally ask whether the underlying reward can be identified. Given a reward parametrization $r_\theta$ with $\theta \in\Theta\subseteq \R^d$, we measure reward recovery through the weighted parameter error
\begin{equation}\label{eq:reward_estimation_error}
    \norm*{\thetastar-\thetahat}_A^2,
    \qquad
    \norm*{x}_A^2 \defeq x^\top A x,
\end{equation}
where $\thetastar$ is a reference parameter, usually the population risk minimizer introduced shortly, and $A\in\R^{d\times d}$ is a positive semidefinite weight matrix. As we will see later, a natural choice for $A$ is the Hessian of the IRL loss at $\thetastar$, as it captures the local curvature of the loss and therefore also the identifiability of parameter directions (see Corollary~\ref{cor:fisher_kernel_shaping}).

\subsection{Optimization Objectives}\label{sec:objectives}
Since we observe only the demonstrations $\DE$ and not the expert policy $\piE$ or an underlying reward parameter itself\footnote{In general, we do not
assume that there exists $\thetaE$ such that $\piE$ is optimal for
$r_{\thetaE}$.}, we cannot minimize the imitation metrics above
directly and turn to empirical risk minimization. Given a hypothesis space $\cZ$ and a loss function $\ell: \cZ \times \br*{\cSA}^T \to \R$, the corresponding
\emph{empirical} and \emph{population} risks are
\begin{equation}
    \hat{L}(z) \defeq \Ehat^{\piE}\bs*{\ell(z;\tau)},
    \qquad
    L(z) \defeq \E^{\piE}\bs*{\ell(z;\tau)}.
\end{equation}
Here, $\Ehat^{\piE}$ denotes the empirical expectation associated with $\DE$, defined as $\Ehat^{\piE}\bs*{f(\tau)} \defeq \tfrac{1}{n}\sum_{i=1}^n f(\tau^i)$ for $f:\br*{\cSA}^T \to \R$. Empirical risk minimization recovers a minimizer $\hat{z}$ of $\hat{L}$, while statistical guarantees typically bound the excess population risk,
\begin{equation}
    L(\hat{z}) - \inf_{z\in\cZ} L(z).
\end{equation}
In the following, we focus on two specific empirical risk minimization problems.

\paragraph{Min-Max-IRL}
In this setting, the hypothesis space $\cZ$ is a set of rewards $\cR$ and each reward induces a unique soft-optimal policy $\pistar_r$. Min-Max-IRL seeks a reward in $\cR$ under which the empirical
expert occupancy measure $\muhat^{\piE}_{t}(B) \defeq \Ehat^{\piE}\bs*{\ind\!\br*{(s_t,a_t)\in B}}$ is least suboptimal with respect to the regularized objective~\eqref{eq:regularized_return}:
\begin{equation}
    \min_{r\in\cR}\,\max_{\pi}\;
    \ip{r,\,\mu^{\pi}-\muhat^{\piE}} + \beta H(\pi).
\end{equation}
With the trajectory loss $\ell^{\mathsf{MM}}(r;\tau) \defeq
\Jstar(r) - \sum_{t=1}^T r_t(s_t,a_t)$, this can be rewritten equivalently as empirical risk minimization over $\cR$:
\begin{equation}\label{eq:min_max_irl}\tag{Min-Max-IRL}
    \min_{r\in\cR}\;
    \LhatIRL(r), \qquad \LhatIRL(r)
    \defeq \Ehat^{\piE}\bs*{
    \ell^{\mathsf{MM}}(r;\tau)}
    = \Jstar(r) - \ip*{r,\,\muhat^{\piE}}.
\end{equation}
If $H(\piE)$ is finite, the corresponding excess population risk is
\begin{equation}\label{eq:irl_excess}
    \LIRL(\rhat)-\inf_{r \in \cR}\LIRL(r) = \Jstar(\rhat) - J(\rhat, \piE) - \inf_{r \in \cR}\br*{\Jstar(r) - J(r, \piE)}.
\end{equation}
The above min-max formulation is the basis of both maximum causal entropy IRL~\citep{ziebart2010modeling} and generative adversarial imitation learning~\citep{ho2016generative}.

\paragraph{Maximum Likelihood Estimation}
In this setting, the hypothesis space $\cZ$ is a class of policy densities $\cP$, each $p=(p_t)_{t=1}^T\in\cP$ induces a
policy via
$\pi_{p,t}(\diff a \mid s)=p_t(a \mid s)\lambda(\diff a)$. The loss function is given by the negative log-likelihood $\ell^{\mathsf{MLE}}(p;\tau) \defeq -\sum_{t=1}^T \log p_t(a_t \mid s_t)$, yielding the density estimation problem:
\begin{equation}\label{eq:mle}\tag{MLE}
    \min_{p \in\cP}\;\LhatMLE(p), \qquad \LhatMLE(p)
    \defeq \Ehat^{\piE}\bs*{
    \ell^{\mathsf{MLE}}(p;\tau)}.
\end{equation}
If $H(\piE)$ is finite, the MLE excess risk is
\begin{equation}\label{eq:mle_excess}
    \LMLE(\hat{p})-\inf_{p \in \cP}\LMLE(p) = \DKL\br*{\bP^{\piE}, \bP^{\pihat}} - \inf_{p\in\cP}\DKL\br*{\bP^{\piE}, \bP^{\pi_p}}.
\end{equation}
Maximum likelihood is a standard method for density estimation, and in the behavioral cloning setting \citet{foster2024behavior} show that it enjoys strong finite-sample guarantees if the problem is well-specified and $\cP$ is sufficiently well-behaved.\footnote{The density class $\cP$ needs to have small enough log-covering number (see Definition~\ref{def:log_cover}).}
Lastly, note that \eqref{eq:mle} also gives rise to an IRL algorithm by optimizing over the set of soft-optimal densities \eqref{eq:soft_opt_density}
\begin{equation}\label{eq:soft_opt_densities}
     \cP^{\star}(\cR) \defeq \bc*{p^{\star}_r : r\in\cR};
\end{equation}
we refer to this objective as MLE-IRL.

\section{Structural Equivalences}\label{sec:equivalences}
Our first contribution is to clarify the relationship between \eqref{eq:min_max_irl} and \eqref{eq:mle}.
Part~1 of Theorem~\ref{thm:equivalences} shows that, for a given reward class $\cR$, Min-Max-IRL and MLE-IRL are equivalent at the population level. At the empirical level, this equivalence continues to hold under deterministic dynamics. Conversely, Part~2 shows that MLE over a policy density class $\cP$ with bounded log-densities is equivalent to Min-Max-IRL over the induced reward class
\begin{equation}\label{eq:myopic_rewards}
    \beta\log\cP \defeq \bc*{(r_t)_{t=1}^T : r_t(s,a) = \beta\log p_t(a \mid s),\ p\in\cP}.
\end{equation}

\begin{theorem}[Informal]\label{thm:equivalences}
    Let $\beta > 0$. Consider a reward class~$\cR$ and a policy density class~$\cP$
    with bounded log-densities, and let $\cP^{\star}(\cR)$ and $\beta\log\cP$ be defined as in \eqref{eq:soft_opt_densities} and \eqref{eq:myopic_rewards}.
    Then:
    \begin{enumerate}
        \item \eqref{eq:min_max_irl} over $\cR$ and \eqref{eq:mle}  over $\cP^{\star}(\cR)$ are equivalent at the population level, and they are equivalent at the empirical level
        if $\br{\bP_t}_{t=0}^{T-1}$ are deterministic.
        \item \eqref{eq:min_max_irl} over $\beta\log\cP$ and \eqref{eq:mle} over $\cP$ are equivalent at both the population and empirical level.
    \end{enumerate}
\end{theorem}

\begin{figure}[h]
\centering
\begin{tikzpicture}[
    node distance=3.0cm,
    formulation/.style={
        rectangle, rounded corners, draw, thick,
        minimum width=3.0cm, minimum height=1.3cm,
        align=center, font=\small, inner xsep=4pt
    },
    condition/.style={
        font=\scriptsize, align=center,
        inner sep=1pt
    },
    arr/.style={-{Stealth[length=2.5mm]}, thick}
]
\node[formulation] (mle) {\textbf{\hyperref[eq:mle]{MLE}}\\[2pt]
    $\min\limits_{p \in \cP}\ \LhatMLE(p)$ };
\node[formulation, right=of mle] (irl) {\textbf{\hyperref[eq:min_max_irl]{Min-Max-IRL}}\\[2pt]
    $\min\limits_{r \in \cR}\ \LhatIRL(r)$};

\draw[arr] ([yshift=3pt]mle.east) -- ([yshift=3pt]irl.west)
    node[condition, midway, yshift=8pt] {%
        $\cR = \beta\log\cP$%
    };
\draw[arr] ([yshift=-3pt]irl.west) -- ([yshift=-3pt]mle.east)
    node[condition, midway, yshift=-12pt] {%
        deterministic,\\ $\cP = \cP^{\star}(\cR)$%
    };
\end{tikzpicture}
\caption{Equivalence between MLE and Min-Max-IRL at the empirical level.}
\label{fig:equivalences}
\end{figure}
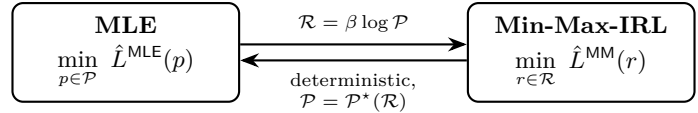

The above equivalences, summarized in Figure~\ref{fig:equivalences}, let us study Min-Max-IRL through the lens of MLE and vice versa. The formal
notion of equivalence of optimization problems is given in
Definition~\ref{def:equivalence}. In particular, it ensures that
minimizers can be recovered from one formulation to the other.

A formal statement of Theorem~\ref{thm:equivalences} requires introducing equivalence classes of
rewards inducing the same soft-optimal density, which we defer to
Theorem~\ref{thm:equivalences_formal} in Appendix~\ref{app:sec:structural_equivalences}. Here we sketch the proof.
\begin{proof}[Proof sketch for Theorem~\ref{thm:equivalences}]
The result rests on the following return decomposition
(Lemma~\ref{lem:return_decomp}): for any trajectory and reward $r$, we have
\begin{equation}
    \underbrace{\Jstar(r) - \sum_{t=1}^T r_t(s_t,a_t)}_{\ell^{\mathsf{MM}}(r; \tau)}
    +
    \sum_{t=0}^{T-1}
    \delta_{t,r}^{\star}(s_t,a_t,s_{t+1}) = \underbrace{- \beta\sum_{t=1}^T
    \log\pstar_{t,r}(a_t \mid s_t)}_{\ell^{\mathsf{MLE}}(\pstar_r; \tau)},
\end{equation}
where $\delta_{t,r}^{\star}(s_t, a_t, s_{t+1}) \defeq V_{t+1,r}^{\star}(s_{t+1}) - (\bP_t V_{t+1,r}^{\star})(s_t,a_t)$ with the convention that $(\bP_0 V_{1,r}^{\star})(s_0,a_0) = \Jstar(r)$.
Since $\E^{\pi}[\delta_{t,r}^{\star}(s_t, a_t, s_{t+1}) \mid s_t,a_t]=0$, taking expectations under $\E^{\piE}$ and
$\Ehat^{\piE}$, respectively, yields
\begin{equation}\label{eq:risk_equivalence}
    \LIRL(r) = \beta \LMLE(\pstar_r),
    \qquad
    \LhatIRL(r)
    \;+\;
    \sum_{t=0}^{T-1}
    \Ehat^{\piE}\bs*{\delta_{t,r}^{\star}(s_t, a_t, s_{t+1})} = \beta \LhatMLE(\pstar_r).
\end{equation}
Part~1 then follows since deterministic dynamics ensure that
$\delta_{t,r}^{\star}= 0$ almost surely, and Part~2 since $p = \pstar_r$ and $V^\star_{t,r}=0$ for the reward $r = \beta \log p$, which forces $\delta_{t,r}^{\star}= 0$ for $r\in\beta\log\cP$.
\end{proof}

A few remarks are in order. First, the empirical equivalence in Part~1 of Theorem~\ref{thm:equivalences} breaks under
stochastic dynamics: the residual~$\delta^\star_{t,r}$
has zero mean under~$\bP^{\pi}$ but not under the empirical trajectory
distribution. Proposition~\ref{prop:convexity} shows that this residual renders MLE-IRL nonconvex, while Min-Max-IRL remains convex. It is therefore unclear whether MLE-IRL is computationally tractable under stochastic dynamics. In contrast, Min-Max-IRL admits efficient algorithms under no-regret oracle assumptions (see Appendix~\ref{app:sec:computation}). Second, the population identity in \eqref{eq:risk_equivalence} shows that the Min-Max-IRL and MLE excess risks coincide up to the factor $\beta$. In light of the excess risk representations in \eqref{eq:irl_excess} and \eqref{eq:mle_excess} above, whenever $H(\piE)$ is finite, we obtain (\cf\ Corollary~\ref{cor:soft_suboptimality})
\begin{equation}\label{eq:excess_risk}
\LIRL(r)-\inf_{r'\in\cR}\LIRL(r')
=
\beta\br*{
\DKL\br*{\bP^{\piE}, \bP^{\pistar_r}}
-
\inf_{r'\in\cR}
\DKL\br*{\bP^{\piE}, \bP^{\pistar_{r'}}}
}.
\end{equation}
This identity makes the trajectory-level KL divergence in \eqref{eq:traj_divergences} a natural metric for evaluating both \eqref{eq:mle} and \eqref{eq:min_max_irl}.

We now turn to statistical guarantees of \eqref{eq:min_max_irl}. In particular, we prove finite-sample bounds both for the trajectory-level divergences \eqref{eq:traj_divergences} and for the squared parameter error \eqref{eq:reward_estimation_error}.

\section{Statistical Guarantees}\label{sec:statistical_guarantees}
\subsection{Setup}
Throughout this section, we restrict our attention to finite-dimensional linear reward classes.
\begin{assumption}[Linear rewards]\label{ass:linear_model}
    Let $\Theta \defeq \bc*{\theta \in \R^d :
    \norm*{\theta} \leq B_{\theta}}$. Assume that
    \begin{equation}
        \cR \defeq \bc*{r_{\theta}:
        r_{t,\theta}(s,a) = \ip*{\theta, \phi_t(s,a)},\;
        \theta \in \Theta},
    \end{equation}
    where $\phi_t\colon\cSA\to\R^d$ are bounded and measurable, satisfying
    $\norm{\sum_{k=t}^T \phi_k} \leq B_{\phi}$ for every $t = 1,\hdots,T$.
\end{assumption}
We may think of the components of $\phi_t(s,a)$ as either hand-designed
reward terms or pretrained representations. Under this linear model, we denote
$\Jstar(\theta)\defeq\Jstar(r_{\theta})$ and
$\pistar_{\theta}\defeq\pistar_{r_{\theta}}$. The loss function corresponding to \eqref{eq:min_max_irl} can then be rewritten equivalently as $\ell^{\mathsf{MM}}(\theta; \tau) \defeq \Jstar(\theta) - \ip*{\theta, \phi(\tau)}$, and the risks as
\begin{equation}\label{eq:lin_irl_risk}
    \LIRL(\theta)
        \defeq \Jstar(\theta) - \ip*{\theta, \phi\br*{\piE}}, \quad
        \LhatIRL(\theta)
        \defeq \Jstar(\theta) - \ip*{\theta, \phihat\br*{\piE}},
\end{equation}
where $\phi(\tau)\defeq \sum_{t=1}^T \phi_t(s_t,a_t)$ denotes the feature return, $\phi(\piE)\defeq \E^{\piE}\bs{\phi(\tau)}$ the expert feature expectation, and $\phihat(\piE)\defeq \Ehat^{\piE}\bs{\phi(\tau)}$ its empirical counterpart. We study the corresponding population and empirical risk
minimizers
\begin{equation}\label{eq:irl_risk_minimizers}
    \thetastar
    \in \argmin_{\theta\in\Theta}
    \LIRL(\theta),
    \qquad
    \thetahat
    \in \argmin_{\theta\in\Theta}
    \LhatIRL(\theta).
\end{equation}

A key role in our analysis is played by the Hessian of the Min-Max-IRL risk,
\begin{equation}
    \mathsf{H}(\theta)
    \defeq
    \nabla^2 \LIRL(\theta) \in \R^{d\times d},
\end{equation}
which captures the curvature of the loss landscape. Since the remaining terms in \eqref{eq:lin_irl_risk} are linear in $\theta$, this curvature is determined entirely by $\Jstar$, that is,
$\mathsf{H}(\theta) = \nabla^2 \LhatIRL(\theta) = \nabla^2_\theta \ell^{\mathsf{MM}}(\theta;\tau) = \nabla^2 \Jstar(\theta)$.
We assume this curvature is strictly positive in every direction.

\begin{assumption}\label{ass:full_rank}
    We assume that for some $\lambdastar > 0$,
    \begin{equation}
        \Hstar \defeq \mathsf{H}(\thetastar) \succeq \lambdastar I_d.
    \end{equation}
\end{assumption}
\begin{remark}
Although stated locally at $\thetastar$, Assumption~\ref{ass:full_rank} is in fact a global identifiability condition. In particular, Corollary~\ref{cor:fisher_kernel_shaping} shows that two reward parameters $\theta, \theta'$ induce the same optimal trajectory laws $\bP^{\pistar_{\theta}} = \bP^{\pistar_{\theta'}}$ if and only if $\theta - \theta' \in \ker \mathsf{H}(\theta)$. Moreover, the subspace $\ker \mathsf{H}(\theta)$, corresponding to unidentifiable parameter directions, is independent of $\theta$. Hence $\mathsf{H}(\thetastar)\succ0$ implies $\ker \mathsf{H}(\theta)={0}$ for every $\theta$, and therefore $\mathsf{H}(\theta)\succ0$ for every $\theta$. Consequently, $\LIRL$ and $\LhatIRL$ are strictly convex, so the minimizers $\thetastar$ and $\thetahat$ are unique. If this condition fails, our analysis can be extended to the quotient space $\R^d/\ker \mathsf{H}(\theta)$ (see Remark~\ref{rem:identifiability_quotient}).
\end{remark}

\subsection{Classical Asymptotic Theory}\label{sec:asymptotics}
Before proceeding with our main results, let us recall what
classical asymptotic theory predicts for the empirical risk minimizer in \eqref{eq:irl_risk_minimizers}. Suppose that $\thetastar$ is an interior point of $\Theta$. Then, under standard
consistency and local smoothness conditions \citep[Theorem~5.23]{van2000asymptotic}, the
empirical risk minimizer $\thetahat$ satisfies
\begin{equation}\label{eq:asymptotic_normality}
    \sqrt{n}\br*{\thetahat - \thetastar}
    \;\xrightarrow{\,d\,}\;
    \cN\br*{0,\; \Hstar^{-1} \mathsf{G}_{\star} \Hstar^{-1}},
    \quad \text{as} \quad n\to\infty,
\end{equation}
where $\xrightarrow{\,d\,}$ denotes convergence in distribution and $\mathsf{G}_{\star} \defeq \operatorname{Cov}_{\tau\sim\bP^{\piE}}\bs{\nabla \ell^{\mathsf{MM}}(\thetastar; \tau)}$ denotes the covariance of the loss gradient at $\thetastar$. For \eqref{eq:min_max_irl} the latter is equal to the covariance of the feature return $\phi(\tau)$,
\begin{equation}
    \mathsf{G}_{\star}  = \operatorname{Cov}_{\tau\sim\bP^{\piE}}\bs*{\phi(\tau)} \eqdef \SigmaE.
\end{equation}
From asymptotic normality \eqref{eq:asymptotic_normality}, it follows, under sufficient integrability and by a second-order Taylor expansion of $L^{\mathsf{MM}}$ around $\thetastar$, that
\begin{equation}\label{eq:asymptotic_risks}
    \E\bs*{\norm*{\thetahat - \thetastar}_{\Hstar}^2}
    = \dfrac{\dstar}{n} + o\br*{n^{-1}},
    \quad
    \E\bs*{L^{\mathsf{MM}}(\thetahat)} - L^{\mathsf{MM}}(\thetastar)
    = \dfrac{\dstar}{2n} + o\br*{n^{-1}}, \quad \text{as} \quad n\to\infty,
\end{equation}
where the expectation is over the data $\DE$ and $\dstar \defeq \tr\br{\mathsf{G}_{\star} \Hstar^{-1}} = \tr(\SigmaE \Hstar^{-1})$ is the so-called effective dimension \citep{ostrovskii2021finite}. For Min-Max-IRL, Proposition~\ref{prop:dstar_decomp} shows that $\dstar=\beta d$ in the well-specified deterministic case. In general, $\dstar$ can be larger, but is always bounded by $\dstar \leq B_{\phi}^2/\lambdastar$.


\subsection{Upper Bound for Min-Max-IRL}
We now show that for \eqref{eq:min_max_irl}, bounds of the form \eqref{eq:asymptotic_risks} hold with high probability and in a nonasymptotic sense. The key technical ingredient is pseudo-self-concordance of the Min-Max-IRL loss (Proposition~\ref{prop:self_concordance}), in the sense of \citet{ostrovskii2021finite}. This property controls the relative variation of the Hessian $\mathsf{H}(\theta)$ and, in our setting, is governed by a bound on the cumulative feature advantage. In particular, let $B_{A_\phi}$ be a constant such that
\begin{equation}\label{eq:advantage_bound}
    \sup_{\theta\in\R^d}\; \norm*{\sum_{t=1}^T \br*{Q_{t,\phi}^{\pistar_{\theta},0}(s_t,a_t) - V_{t,\phi}^{\pistar_{\theta},0}(s_t)}} \leq B_{A_{\phi}}.
\end{equation}
As detailed in Appendix~\ref{app:sec:advantage_and_density_ratio_bound}, it suffices that the bound above holds almost surely. We also show that one may always take $B_{A_{\phi}}=2TB_{\phi}$, while under deterministic dynamics $B_{A_{\phi}}=2B_{\phi}$ suffices. Our main result of this section is as follows.
\begin{restatable}{theorem}{irlfastrate}\label{thm:irl_fast_rate}
    Let $\beta > 0$ and let
    Assumptions~\ref{ass:linear_model} and \ref{ass:full_rank} hold. Define
    \begin{equation}
        \varepsilon_n\br*{\delta} \defeq \dfrac{\dstar \log\br*{\delta^{-1}}}{n} + \dfrac{B_{\phi}^2 \log^2\br*{\delta^{-1}}}{\lambdastar n^2},
    \end{equation}
    and let $\pi_{\star}\defeq\pistar_{\thetastar}, \pihat\defeq\pistar_{\thetahat}$.
    If\;
    $n \gtrsim (\beta \lambda_{\star})^{-1}B_{A_{\phi}} \log(\delta^{-1})\max\bc{\beta^{-1}B_{A_{\phi}} d_{\star},  B_{\phi} }$, then, with probability at least $1-\delta$, the following hold:
    \begin{enumerate}
        \item (Parameter estimation bound)
            \begin{equation}\label{eq:parameter_error_hessian}
                \norm*{\thetahat-\thetastar}_{\Hstar}^2
                \;\lesssim\;
                \varepsilon_n(\delta).
            \end{equation}
        \item (Excess KL risk bound)
            \begin{equation}\label{eq:excess_kl}
                \DKL\!\br*{\bP^{\piE}, \bP^{\pihat}}
                \;\lesssim\;
                \DKL\!\br*{\bP^{\piE},
                \bP^{\pi_{\star}}}
                \;+\;
                \beta^{-1}\varepsilon_n(\delta).
            \end{equation}
        \item (Equivalences)
            \begin{equation}
                \DHel^2(\bP^{\pi_{\star}}, \bP^{\pihat}) \asymp \DKL(\bP^{\pi_{\star}}, \bP^{\pihat}) \asymp \DKL(\bP^{\pihat}, \bP^{\pi_{\star}}) \asymp \beta^{-1} \norm*{\thetahat - \thetastar}_{\Hstar}^2.
            \end{equation}
    \end{enumerate}
\end{restatable}
For fixed $\beta$, the squared parameter error \eqref{eq:parameter_error_hessian} and the excess
trajectory-level KL risk \eqref{eq:excess_kl} both decay at $\cO(\dstar n^{-1})$
with an $\cO(n^{-2})$ remainder (Parts~1 and~2). Part~3 shows that, between $\pi_{star}$ and $\pihat$, the squared
trajectory-level Hellinger, the trajectory-level KL, and the squared
parameter error in the Hessian norm are equivalent up to constants. In the well-specified case, we have $\thetastar = \thetaE$, so the squared
trajectory Hellinger, $\DHel^2(\bP^{\piE},\bP^{\pihat})$, the
trajectory KL, $\DKL(\bP^{\piE}, \bP^{\pihat})$, and the squared
parameter error, $\norm{\thetahat-\thetaE}_{\Hstar}^2$, all decay at
the fast rate $\cO(\dstar n^{-1})$.

Next, we discuss the main technical ideas of the proof. The full proof is provided in Appendix~\ref{app:sec:proof_local_fast_rates}.
\begin{proof}[Proof sketch for Theorem~\ref{thm:irl_fast_rate}]
    The proof follows \citet{ostrovskii2021finite,liu2022confidence}, who
    consider general empirical risk minimization. We make two modifications
    specific to our setting: the parameter set $\Theta$ is bounded, and we use a
    vector Bernstein inequality, which better fits our assumptions, in place of
    sub-Gaussian concentration.

    Let $\Delta^{\theta}_n \defeq \thetahat - \thetastar$ and
    $\Delta^{\phi}_n \defeq \phihat(\piE) - \phi(\piE)$. By \eqref{eq:lin_irl_risk}, we have
    \begin{equation}\label{eq:pf_sketch_perturbation}
        \LhatIRL(\theta) = \LIRL(\theta) - \ip*{\theta, \Delta^{\phi}_n},
    \end{equation}
    so the empirical risk is a linear perturbation of the population risk. The first step in the proof is to establish pseudo-self-concordance of the Min-Max-IRL loss
    (Proposition~\ref{prop:self_concordance}), which ensures that the
    Hessian $\mathsf{H}(\theta)$ varies smoothly in $\theta$. It then follows that on the
    Dikin ellipsoid
    $\Theta_{\rho_\star}(\thetastar)
    \defeq \bc*{\theta : \norm*{\theta - \thetastar}_{\Hstar} \leq \rho_\star}$
    of radius $\rho_\star \defeq \beta\sqrt{\lambdastar}/B_{A_\phi}$, the
    population risk satisfies the gradient monotonicity property,
    \begin{equation}\label{eq:pf_sketch1}
        c\,\norm*{\theta - \thetastar}_{\Hstar}^2
        \leq
        \ip*{\theta - \thetastar,\,
        \nabla\LIRL(\theta) - \nabla\LIRL(\thetastar)},
        \qquad
        \theta \in \Theta_{\rho_\star}(\thetastar),
    \end{equation}
    with $c \defeq 1 - e^{-1}$ (Corollary~\ref{cor:self_concordance}). Furthermore, the first-order optimality conditions for $\thetahat$ and
    $\thetastar$, and Cauchy--Schwarz, give
    \begin{align}
        \ip*{\Delta^{\theta}_n,\,
        \nabla\LIRL(\thetahat) - \nabla\LIRL(\thetastar)}
        &=
        \underbrace{\ip*{\Delta^{\theta}_n,\,
        \nabla\LhatIRL(\thetahat) - \nabla\LIRL(\thetastar)}}_{\leq 0}
        + \ip*{\Delta^{\theta}_n, \Delta^{\phi}_n}\\
        &\leq
        \norm*{\Delta^{\theta}_n}_{\Hstar}\,\norm*{\Delta^{\phi}_n}_{\Hstar^{-1}}.\label{eq:pf_sketch2}
    \end{align}

    On the event $\thetahat \in \Theta_{\rho_\star}(\thetastar)$, combining
    \eqref{eq:pf_sketch1} and \eqref{eq:pf_sketch2} yields
    \begin{equation}\label{eq:pf_sketch3}
        \norm*{\Delta^{\theta}_n}_{\Hstar}
        \leq c^{-1}\,\norm*{\Delta^{\phi}_n}_{\Hstar^{-1}}.
    \end{equation}
    Hence, the parameter estimation error in the Hessian norm, $\norm{\cdot}_{\Hstar}$, is controlled by the concentration of the empirical feature expectation measured in the dual norm, $\norm{\cdot}_{\Hstar^{-1}}$. A vector Bernstein concentration bound (Proposition~\ref{prop:concentration})
    shows that $\norm*{\Delta^{\phi}_n}_{\Hstar^{-1}}^2 \lesssim \varepsilon_n(\delta)$
    with probability at least $1-\delta$, and a localization argument shows that, for large enough $n$, the same concentration event implies $\thetahat \in \Theta_{\rho_\star}(\thetastar)$. This proves the parameter estimation bound of Part~1.

    The excess KL bound follows analogously. By
    optimality of $\thetahat$ for $\LhatIRL$, we have
    \begin{equation}
        \LIRL(\thetahat) - \LIRL(\thetastar)
        =
        \underbrace{\LhatIRL(\thetahat)
        - \LhatIRL(\thetastar)}_{\leq\, 0}
        + \ip*{\Delta^{\theta}_n, \Delta^{\phi}_n}
        \leq
        c^{-1}\,\norm*{\Delta^{\phi}_n}_{\Hstar^{-1}}^2,
    \end{equation}
    where we used Cauchy--Schwarz and \eqref{eq:pf_sketch3}. Together with \eqref{eq:excess_risk}, which identifies the excess
    risk as $\beta$ times the excess trajectory KL, and the vector Bernstein bound, this proves Part~2.

    Finally, the equivalences in Part~3 follow from the local
    geometry on the Dikin ellipsoid. The equivalence of the trajectory-level KL divergences and $\beta^{-1}\norm*{\Delta^{\theta}_n}_{\Hstar}^2$ is a direct consequence of pseudo-self-concordance (Corollary~\ref{cor:self_concordance}; Part~3), and the
    equivalence of KL and squared Hellinger distance follows from a local density-ratio
    bound (Corollary~\ref{cor:self_concordance}; Part~1) and \citet[Lemma~5]{birge1998minimum}.

    The global version, Theorem~\ref{thm:global_fast_rate}, avoids the
    critical sample-size condition,
    but pays additional factors of order $\beta^{-1}B_{A_\phi}B_\theta$,
    as it uses a global rather than localized gradient monotonicity bound over
    $\Theta$.\looseness-1
\end{proof}

\subsection{Lower Bound}\label{sec:lower_bound}
We complement the upper bound above with a local minimax lower bound, showing that the fast rate in Part~1 of Theorem~\ref{thm:irl_fast_rate}
cannot be improved in the well-specified deterministic setting. In particular, we fix a
parameter $\theta_0$ and allow an adversary to choose the expert
parameter $\thetaE$ in an arbitrarily small neighborhood $U$ of $\theta_0$.
The theorem shows that, even in this local regime where $\thetaE$ is known up to a neighborhood, every estimator incurs squared error of
order at least $\beta d n^{-1}$ in the intrinsic Hessian norm, with constant probability.

\begin{restatable}{theorem}{irllowerbound}\label{thm:irl_lower_bound}
    Let $\beta > 0$ and let
    Assumption~\ref{ass:linear_model} hold.
    Fix $\theta_0\in\R^d$ and suppose $\mathsf{H}(\theta_0)\succ0$. Then, for any neighborhood $U\subseteq\R^d$ of $\theta_0$, there exists a universal constant $c > 0$ such that for all sufficiently large $n$ and every estimator $\thetahat$,
    \begin{equation}
        \sup_{\thetaE\in U}
    \Pr\nolimits_{\thetaE}
    \left(
        \norm*{\thetahat - \thetaE}_{\mathsf{H}(\thetaE)}^2
        \geq c \frac{\beta  d}{n}
    \right)
    \geq \dfrac{1}{2}.
    \end{equation}
    Here, $\Pr\nolimits_{\thetaE}$ denotes probability over
the expert trajectories $\tau_1, \hdots, \tau_n$ drawn i.i.d. from
$\bP^{\pistar_{\thetaE}}$, and over any internal randomness
of the estimator $\thetahat$.
\end{restatable}
The proof, provided in Appendix~\ref{app:sec:lower_bound}, follows from an application of Fano's method \citep[Section~15.3]{wainwright2019high}. Since $\dstar=\beta d$ in the well-specified deterministic setting, the lower bound matches the parameter-estimation upper bound in Theorem~\ref{thm:irl_fast_rate} up to logarithmic factors. Under stochastic dynamics, however, the bounds need not match because $\dstar$ may exceed $\beta d$. As discussed in Remark~\ref{rem:mle_effective_dimension}, we expect that an analogous analysis of MLE-IRL could close this gap in the well-specified setting, since its effective dimension is $d$. Yet, because MLE-IRL is nonconvex under stochastic dynamics, it remains unclear whether there exists an IRL algorithm that is both statistically optimal and computationally tractable, even in simple tabular settings.

\subsection{Implications for Imitation Learning}
\label{sec:imitation_implications}

In IRL, reward recovery may be of interest in its own right, for example, when the
goal is to transfer the learned reward to a new environment. When the goal is imitation in the same
environment, the relevant object is the trajectory law induced by the
recovered reward. The trajectory-level KL divergence controlled in
Theorem~\ref{thm:irl_fast_rate} directly measures the mismatch
between the expert trajectory law $\bP^{\piE}$ and the learned trajectory
law $\bP^{\pihat}$. Similar squared Hellinger guarantees have been obtained
for maximum likelihood behavioral cloning by
\citet{foster2024behavior,rohatgi2025computational}. In Appendix~\ref{app:sec:mle_comparison}, we adapt the likelihood-ratio
analysis of \citet{foster2024behavior,rohatgi2025computational} to derive guarantees for
MLE-IRL and compare them with
the Min-Max-IRL fast-rate guarantee of
Theorem~\ref{thm:irl_fast_rate}. In particular, we show that the resulting Hellinger bounds agree up to logarithmic factors in the well-specified deterministic setting, where Min-Max-IRL and MLE-IRL are equivalent, but generally differ under misspecification or stochastic dynamics.

Our guarantees for Min-Max-IRL and the guarantees for
MLE behavioral cloning
\citep{foster2024behavior, rohatgi2025computational}
yield small trajectory-level divergence only when the respective
reward- or policy-induced model class approximates the expert well. In
the well-specified setting with deterministic dynamics, our results
give a rate of order $\cO(dn^{-1})$ for a $d$-dimensional linear reward
class. Similarly, \citet[Corollary~C.4]{foster2024behavior} obtain
the same rate for MLE behavioral cloning with a log-linear policy class
based on $d$-dimensional features. From this perspective, ignoring computational
considerations, IRL is preferable if the expert is realizable by a low-dimensional reward class, but the policy is difficult to parametrize directly, whereas BC is preferable if the expert policy is realizable by a
low-dimensional log-linear class but no reward in $\cR$ induces it as its
soft-optimal policy.

Beyond trajectory-level divergences, another central metric in imitation learning is the performance gap,
\begin{equation}
    \ip*{\rtest,\mu^{\piE}-\mu^{\pihat}},
\end{equation}
under an unknown test reward $\rtest$. \citet{foster2024behavior} show that this gap can be bounded tightly in terms of the squared Hellinger distance, and hence also through
the KL guarantee of Theorem~\ref{thm:irl_fast_rate}. However, if $\rtest\in\cR$, then this gap is controlled directly through the min-max dual of \eqref{eq:min_max_irl}, without requiring approximation of the full
trajectory law in KL or Hellinger distance. Appendix~\ref{app:sec:performance_gap_bounds}
formalizes this guarantee and discusses settings in which this yields an advantage over behavioral cloning.

\section{Conclusion}
We developed a statistical theory of entropy-regularized IRL in finite-horizon MDPs with Borel state and action spaces. We showed that MLE-IRL and Min-Max-IRL coincide at the population level and, under deterministic dynamics, also at the empirical level. Under stochastic dynamics, their empirical objectives differ, and MLE-IRL is generally nonconvex. This clarifies the relationship between entropy-regularized IRL and maximum likelihood estimation.

For linear reward classes, we leveraged pseudo-self-concordance of the Min-Max-IRL loss to establish improved rates for the excess trajectory-level KL risk and the parameter error in the Hessian norm. A local minimax lower bound shows that the parameter-estimation rate is tight up to logarithmic factors in the well-specified deterministic setting. Furthermore, we extended reward-identifiability results to Borel state and action spaces, derived likelihood-ratio-based guarantees for MLE-IRL, discussed bounds for the performance gap, and provided a reduction of Min-Max-IRL to no-regret learning.

Our results open several directions for future research. First, it remains open whether, under stochastic dynamics, there exists an entropy-regularized IRL algorithm that is both statistically optimal and computationally tractable. Second, our self-concordance analysis focuses on entropy-regularized IRL with finite-dimensional linear reward classes. Extending the theory to more general reward classes and to other regularizations would be natural next steps. Finally, our results are theoretical, and
empirical work is needed to examine how they translate to practical IRL problems.

\bibliography{main}
\bibliographystyle{tmlr}

\appendix

\newpage
\addtocontents{toc}{\protect\setcounter{tocdepth}{2}}
{\renewcommand{\contentsname}{Appendix Contents}\tableofcontents}
\newpage

\section{MDP Preliminaries}\label{app:sec:preliminaries}

\subsection{MDP Setup}

Throughout this appendix, we work under the finite-horizon MDP setup
introduced in the main text. We assume that $\cS$ and $\cA$ are
Borel subsets of complete separable metric spaces, equipped with their
Borel $\sigma$-fields $\cB(\cS)$ and $\cB(\cA)$. The initial state distribution $\bP_0$ is a
probability measure on $\cS$. For each $t=1,\ldots,T-1$, the
transition law $\bP_t$ is a stochastic kernel from
$\cS\times\cA$ to $\cS$. That is, $\bP_t(\;\cdot \mid s,a) \in \Delta_{\cS}$ for each $(s,a)\in\cSA$, and
$(s,a)\mapsto \bP_t(C \mid s,a)$ is measurable for every
$C\in\cB(\cS)$. Likewise, all policies considered are stochastic
kernels from $\cS$ to $\cA$. Under these assumptions, there is a unique trajectory law $\bP^{\pi}$ on $(\cS\times\cA)^T$ defined via factorization \citep[Proposition~7.28]{bertsekas1978stochastic},
\begin{equation}
    \bP^{\pi}(\diff \tau) = \bP_0(\diff s_1)\prod_{t=1}^{T-1}\bP_t(\diff s_{t+1} \mid s_t, a_t)\prod_{t=1}^{T}\pi_{t}(\diff a_t \mid s_t), \qquad \tau = \br*{s_1, a_1, \hdots, s_T, a_T}.
\end{equation}

Furthermore, rewards
$r_t:\cS\times\cA\to\R$ are assumed to be bounded and measurable. In particular, for two measurable spaces $\cX$ and $\cY$, we write $B_b(\cX,\cY)$ for the set of bounded measurable functions from $\cX$ to $\cY$, and
$B_b(\cX)\defeq B_b(\cX,\R)$. Hence, bounded measurable reward sequences, $r_1, \hdots, r_T$, live in
\begin{equation}
    B_b^T(\cSA) \defeq \underbrace{B_b(\cSA) \times \dots \times B_b(\cSA)}_{T \text{ times}}.
\end{equation}

\subsection{Policy and Trajectory Densities}
Recall that $\lambda$ is a finite reference measure on $\cA$ with
$0<\lambda(\cA)<\infty$. We write $\pi\ll\lambda$ if
$\pi_t(\;\cdot \mid s)\ll\lambda$ for every $t=1,\ldots,T$ and $s\in\cS$.
For such a policy, let
\begin{equation}
    p_t^\pi(\;\cdot \mid s)
    \defeq
    \frac{\diff\pi_t(\;\cdot \mid s)}{\diff\lambda}
\end{equation}
denote a fixed version of its action density. For the set of trajectories $(\cS\times\cA)^T$, we define the reference measure
\begin{equation}
    \lambdatraj(\diff\tau)
    =
    \bP_0(\diff s_1)
    \prod_{t=1}^{T-1}
    \bP_t(\diff s_{t+1} \mid s_t,a_t)
    \prod_{t=1}^T
    \lambda(\diff a_t).
\end{equation}
Then, for every $\pi\ll\lambda$, the trajectory law $\bP^{\pi}$ admits the
density
\begin{equation}
    \ptraj^\pi(\tau)
    =
    \frac{\diff\bP^{\pi}}{\diff\lambdatraj}(\tau)
    =
    \prod_{t=1}^T p_t^\pi(a_t \mid s_t).
\end{equation}

\subsection{Optimality}\label{app:sec:optimality}
\paragraph{Bellman Optimality Equations}
Recall that for a reference measure $\lambda$ on $\cA$, with $0 < \lambda(\cA) < \infty$, the entropy
$H:\Delta_{\cA}\to [-\infty, \infty]$ is defined as
\begin{equation}
    H(P) = \begin{cases}
        - \int \log \br{\diff P/ \diff \lambda} \diff P, \quad &P \ll \lambda,\\
        - \infty, \quad & \text{otherwise}.
    \end{cases}
\end{equation}

By the Bellman principle of optimality, the optimal value and Q-function satisfy the Bellman optimality equations (see \eg, \citep{geist2019theory})
\begin{equation}\label{eq:bellman_opt_equation}
    V^{\star}_{t,r}(s) = \max_{P\in\Delta_{\cA}} \ip*{Q^{\star}_{t,r}(s,\cdot),\, P} + \beta\,H(P),
    \qquad
    Q^{\star}_{t,r}(s,a) = r_t(s,a) + \br*{\bP_t V^{\star}_{t+1,r}}(s,a),
\end{equation}
with terminal condition $V^{\star}_{T+1,r}= 0$. In particular, the following classical result yields a closed form for the optimal value and policy.
\begin{lemma}[\citealp{donsker1976asymptotic}]\label{lem:donsker}
    For $f\in B_b(\cA)$ and $\beta>0$,
    \begin{equation}
        \beta\log\int_{\cA} e^{\beta^{-1}f(a)}\diff\lambda(a) = \max_{P\in\Delta_{\cA}} \ip*{f, P} + \beta\,H(P),
    \end{equation}
    with maximum attained uniquely at the Gibbs measure $P_{\star}$ with density $\diff P_{\star}/\diff\lambda \propto e^{\beta^{-1}f}$.
\end{lemma}
In light of Lemma~\ref{lem:donsker}, the optimal value has the closed form,
\begin{equation}\label{eq:explicit_opt_value}
    V^\star_{t,r}(s)
    =
    \beta
    \log
    \int_{\cA}
    e^{\beta^{-1}Q^\star_{t,r}(s,a)}
    \,\lambda(\diff a),
\end{equation}
and the corresponding optimal policy,
\begin{equation}\label{eq:explicit_opt_policy}
    \pistar_{t,r}(\diff a \mid s) = \pstar_{t,r}(a \mid s) \lambda(\diff a), \quad \text{with} \quad \pstar_{t,r}(a \mid s) = e^{\beta^{-1}\br*{Q^{\star}_{t,r}(s, a) - V^{\star}_{t,r}(s)}}.
\end{equation}
We call $\pistar_r$ the Bellman-optimal, or soft-optimal, policy for
$r$, and we use the density convention $p_t^{\pistar_r}=p_{t,r}^\star$. Since $r$
is bounded and $\lambda(\cA)<\infty$, backward induction shows that
$V_{t,r}^\star$ and $Q_{t,r}^\star$ are bounded and measurable.
Consequently, the Gibbs formula defines a stochastic kernel.

\begin{remark}[Unregularized setting]
    For $\beta=0$ the Donsker--Varadhan duality is unavailable, and existence of a measurable optimal policy requires standard measurable selection assumptions, \eg, $\cA$ compact, $r_t(s,\cdot)$ upper semicontinuous, and $a\mapsto\bP_t(\;\cdot \mid s,a)$ weakly continuous \citep{hernandez2012discrete}.
\end{remark}

\paragraph{Null Sets of Soft-Optimal Trajectory Laws}

We next record the simple but useful fact that all soft-optimal trajectory laws have the same null sets. Recall that two measures $P$ and $Q$ are equivalent, denoted by $P\sim Q$, if $P\ll Q$ and $Q\ll P$. The next result shows that every soft-optimal trajectory law is equivalent to the trajectory reference measure $\lambdatraj$.

\begin{proposition}
\label{prop:common_support}
Let $\beta>0$. For every $r\in B_b^T(\cSA)$,
\begin{equation}
    \bP^{\pistar_r}
    \sim
    \lambdatraj.
\end{equation}
Consequently, we have $\bP^{\pistar_r}\sim\bP^{\pistar_{r'}}$ for any
$r,r'\in B_b^T(\cSA)$.
\end{proposition}

\begin{proof}
By the trajectory-density factorization,
\begin{equation}
    \ptraj^{\pistar_r}(\tau)
    =
    \prod_{t=1}^T
    p_{t,r}^\star(a_t \mid s_t).
\end{equation}
Since $Q_{t,r}^\star$ and $V_{t,r}^\star$ are finite, the Gibbs formula
implies
\begin{equation}
    0
    <
    p_{t,r}^\star(a \mid s)
    <
    \infty,
\end{equation}
for every $t,s,a$. Hence,
$0<\ptraj^{\pistar_r}(\tau)<\infty$, which proves
$\bP^{\pistar_r}\sim\lambdatraj$. The final claim follows by
transitivity.
\end{proof}

In view of Proposition~\ref{prop:common_support}, we write
$\bP^{\pistar}$-a.s.\ for a property that holds under any, and hence
every, soft-optimal trajectory law. Equivalently, such a property
holds $\lambdatraj$-a.e.

\paragraph{Expected Value Optimality}\label{app:sec:expected_value_optimality}
The Bellman optimal policy $\pistar_r$ is defined as the unique (if $\beta>0$) policy $\pi$ that maximizes the value $V^{\pi}_1(s)$ for every initial state $s$. This is a stronger requirement than optimality for the expected value $J(r,\pi) = \bP_0 V^{\pi}_1$ under a fixed initial distribution $\bP_0$. We denote the set of expected value optimal policies as
\begin{equation}
    \Pi^\star_{\mathsf{EV}}(r)
    \defeq
    \left\{
        \pi:
        J(r,\pi)=J^\star(r)
    \right\}.
\end{equation}
The following result shows that expected-value optimal policies are unique up to $\bP^{\pistar}$-null sets and, in particular, induce a unique optimal trajectory law.
\begin{proposition}
\label{prop:bellman_value_optimality}
Let $\beta>0$ and $r\in B_b^T(\cS\times\cA)$. Then, the following are equivalent:
\begin{enumerate}
    \item $\pi\in\Pi^\star_{\mathsf{EV}}(r)$.
    \item $\bP^{\pi} = \bP^{\pistar_r}$.
    \item $\pi_t(\;\cdot \mid s_t)=\pistar_{t,r}(\;\cdot \mid s_t), \, \bP^{\pistar}\text{-a.s.},\, \forall t.$
    \item $p^{\pi}_t(a_t \mid s_t)=\pstar_{t,r}(a_t \mid s_t), \, \bP^{\pistar}\text{-a.s.},\, \forall t.$
\end{enumerate}
\end{proposition}
\begin{proof}
    \emph{1. $\iff$ 2.}:
    By Corollary~\ref{cor:soft_suboptimality}, we have
    \begin{equation}
        J^\star(r) - J(r,\pi) = \beta \DKL\!\br*{\bP^{\pi}, \bP^{\pistar_r}}.
    \end{equation}
    Since $\DKL(P, Q)$ is nonnegative and equal to zero if and only if $P = Q$, this implies the equivalence of \emph{1.} and \emph{2.}

    \emph{2. $\iff$ 3.}:
    By the chain rule of relative entropy, we have
    \begin{equation}
        \DKL\!\br*{\bP^{\pistar_r}, \bP^{\pi}} = \sum_{t=1}^T\E^{\pistar_r}\bs*{\DKL(\pistar_{t,r}(\;\cdot \mid s_t), \pi_t(\;\cdot \mid s_t))},
    \end{equation}
    which by the same reasoning as above implies the equivalence of \emph{2.} and \emph{3.}

    \emph{3. $\iff$ 4.}: Suppose first that
    \begin{equation}
        \pi_t(\;\cdot \mid s_t)=\pistar_{t,r}(\;\cdot \mid s_t), \quad \bP^{\pistar}\text{-a.s.},\, \forall t.
    \end{equation}
    By $\lambda$-a.e. uniqueness of densities,
    \begin{equation}
        p_t^\pi(\;\cdot \mid s_t)
        =
        p^\star_{t,r}(\;\cdot \mid s_t),
        \quad
        \lambda\text{-a.e.},\ \bP^{\pistar}\text{-a.s.},\ \forall t.
    \end{equation}
    Since $\pistar_{t,r}(\;\cdot \mid s_t)\ll\lambda$, this equality also holds
    $\pistar_{t,r}(\;\cdot \mid s_t)$-a.e. Therefore,
    \begin{equation}
        p_t^\pi(a_t \mid s_t)
        =
        p^\star_{t,r}(a_t \mid s_t),
        \quad
        \bP^{\pistar}\text{-a.s.},\ \forall t .
    \end{equation}
    Thus \emph{3.} implies \emph{4.}

    Conversely, suppose that
    \begin{equation}
        p_t^\pi(a_t \mid s_t)
        =
        \pstar_{t,r}(a_t \mid s_t),
        \quad
        \bP^{\pistar}\text{-a.s.},\ \forall t .
    \end{equation}
    Then,
    \begin{equation}
        p_t^\pi(\;\cdot \mid s_t)
        =
        \pstar_{t,r}(\;\cdot \mid s_t),
        \quad
        \pistar_{t,r}(\;\cdot \mid s_t)\text{-a.e.},
        \ \bP^{\pistar}\text{-a.s.},\ \forall t .
    \end{equation}
    Since $\lambda\ll\pistar_{t,r}(\;\cdot \mid s_t)$, the equality
    also holds $\lambda$-a.e., implying that the two densities define the same measure,
    \begin{equation}
        \pi_t(\;\cdot \mid s_t)
        =
        \pistar_{t,r}(\;\cdot \mid s_t),
        \quad
        \bP^{\pistar}\text{-a.s.},\ \forall t .
    \end{equation}
    Thus \emph{4.} implies \emph{3.}
\end{proof}

\subsection{Return Decomposition}\label{app:sec:return_decomp}

In this subsection, we decompose the centered regularized return into martingale
differences associated with action and transition randomness. Fix
$\beta\geq0$, a reward $r$, and a policy $\pi$. When $\beta>0$, we assume
$\pi\ll\lambda$ and when $\beta=0$, we omit all logarithmic terms below.

For a trajectory $\tau=(s_1,a_1,\ldots,s_T,a_T)$, define the regularized return and advantage function as
\begin{align}
    G_{r}^{\pi}(\tau)
    &\defeq
    \sum_{t=1}^T \br*{ r_t(s_t, a_t) - \beta \log p^{\pi}_t(a_t \mid s_t)}, \\
    A^{\pi}_{t,r}(s, a) &\defeq Q^{\pi}_{t,r}(s, a) - V^{\pi}_{t,r}(s) - \beta \log p^{\pi}_t (a \mid s),\\
    \delta_{t,r}^{\pi} (s_t, a_t, s_{t+1})
    &\defeq
    V_{t+1,r}^{\pi}(s_{t+1})
    -
    (\bP_t V_{t+1,r}^{\pi})(s_t,a_t),
    \qquad t=0,\dots,T-1,
\end{align}
where $s_0,a_0$ are dummy variables and $\bP_0(\;\cdot \mid s_0,a_0) = \bP_0$ is the initial state distribution.

\begin{lemma}\label{lem:return_decomp}
    Let either $\pi\ll \lambda$ or $\beta=0$. It holds that
\begin{equation}
    G_{r}^{\pi}(\tau) - J(r, \pi)
    =
    \sum_{t=1}^T A_{t, r}^{\pi}(s_t, a_t)
    +
    \sum_{t=0}^{T-1} \delta_{t, r}^{\pi} (s_t, a_t, s_{t+1}).
\end{equation}
Moreover, for any pair of policies $\pi, \pi'$,
\begin{equation}
    \E^{\pi}\bs*{A_{t, r}^{\pi}(s_t, a_t) \mid s_t} = 0,
    \qquad
    \E^{\pi'}\bs*{\delta_{t, r}^{\pi}(s_t, a_t, s_{t+1}) \mid s_t,a_t} = 0,
\end{equation}
and the family
\begin{equation}
    A_{1, r}^{\pi},\hdots,A_{T, r}^{\pi},\delta_{0, r}^{\pi'},\hdots,\delta_{T-1, r}^{\pi'}
\end{equation}
is pairwise orthogonal in $L^2(\bP^{\pi})$.
\end{lemma}
For the unregularized return, the above decomposition is presented, for example, by
\citet{pan2024skill}. More generally, it is a standard decomposition of the centered return into martingale differences with respect to the natural filtration of the trajectory \citep[see Section 3.1]{boucheron2013}. The proof follows from a simple telescoping argument.
\begin{proof}
From the definition of the advantage it follows that
\begin{align}
    r_t - \beta \log p^{\pi}_t
    &=
    A_{t, r}^{\pi}
    +
    V_{t,r}^{\pi}
    -
    \bP_t V_{t+1,r}^{\pi}.
\end{align}
Adding and subtracting $V_{t+1,r}^{\pi}$ gives
\begin{equation}
    r_t - \beta \log p^{\pi}_t
    =
    A_{t, r}^{\pi}
    +
    V_{t,r}^{\pi}
    -
    V_{t+1,r}^{\pi}
    +
    \delta_{t, r}^{\pi}.
\end{equation}
Summing over $t=1,\dots,T$ telescopes to
\begin{equation}
    \sum_{t=1}^T \br*{r_t - \beta \log p^{\pi}_t}
    =
    V_{1,r}^{\pi}
    +
    \sum_{t=1}^T A_{t, r}^{\pi}
    +
    \sum_{t=1}^{T-1} \delta_{t, r}^{\pi}.
\end{equation}
Since
\begin{equation}
    \delta_0^{\pi}
    =
    V_{1,r}^{\pi} - J(r, \pi),
\end{equation}
this proves the decomposition.

The conditional mean-zero identities follow directly from the definitions, and pairwise orthogonality is then immediate from the tower property.
\end{proof}

Next, we list three direct consequences of Lemma~\ref{lem:return_decomp}. In particular, the first two corollaries are the regularized performance difference \citep{kakade2002approximately} and soft suboptimality \citep{mei2020global}, while the third shows that the return variance decomposes into an action and dynamics variance term.

\begin{corollary}\label{cor:performance_difference}
For any two policies $\pi, \pi'$,
\begin{equation}
    J(r, \pi) - J^{0}(r, \pi')
    =
    - \sum_{t=1}^T \E^{\pi'}\bs*{A_{t, r}^{\pi}}
    - \beta \sum_{t=1}^T \E^{\pi'}\bs*{\log p^{\pi}_t}.
\end{equation}
\end{corollary}

\begin{corollary}\label{cor:soft_suboptimality}
    For any policy $\pi$ and reward $r$, we have
    \begin{equation}
        \Jstar(r) - J(r, \pi)
        =
        \beta \DKL\br*{\bP^{\pi}, \bP^{\pistar_r}} = \beta \sum_{t=1}^T \E^{\pi}\bs*{\DKL\br*{\pi_t(\;\cdot \mid s_t), \pistar_{t, r}(\;\cdot \mid s_t)}}.
    \end{equation}
\end{corollary}
\begin{corollary}\label{cor:variance_return_decomp}
The variance of the (regularized) return decomposes as
\begin{equation}
    \V^\pi\bs*{G_r^{\pi}}
    =
    \underbrace{\sum_{t=1}^T \E^{\pi}\bs*{\br*{A_{t, r}^{\pi}}^2}}_{\text{action variance}}
    +
    \underbrace{\sum_{t=0}^{T-1} \E^{\pi}\bs*{\br*{\delta_{t, r}^{\pi}}^2}}_{\text{dynamics variance}}.
\end{equation}
\end{corollary}

While the terms $\delta_{t,r}^{\pi}$ do not affect expected values, Corollary~\ref{cor:variance_return_decomp} shows that they do contribute to the return variance.

\subsection{Derivatives of the Optimal Value}\label{app:sec:der_optimal_value}
In this subsection, we first derive directional derivatives of the regularized optimal value with
respect to the reward, and then specialize to linear reward
parametrizations, where we derive the first three derivatives of the optimal value with respect to reward parameters.

\paragraph{Directional Derivatives in Reward Space}

Fix $\beta>0$ and $r,h\in B_b^T(\cSA)$. For $\varepsilon\in\R$, let
\begin{equation}
    r^\varepsilon
    \defeq
    r+\varepsilon h,
    \qquad
    \pi^\varepsilon
    \defeq
    \pistar_{r^\varepsilon},
    \qquad
    V_t^\varepsilon
    \defeq
    V_{t,r^\varepsilon}^\star,
    \qquad
    Q_t^\varepsilon
    \defeq
    Q_{t,r^\varepsilon}^\star,
    \qquad
    J^\varepsilon
    \defeq
    \Jstar(r^\varepsilon).
\end{equation}
We write $\dot V_t^\varepsilon$, $\dot Q_t^\varepsilon$, and
$\dot J^\varepsilon$ for their derivatives with respect to
$\varepsilon$. They can be computed as follows.

\begin{lemma}
\label{lem:first_der}
For every $\varepsilon\in\R$ and $t=1,\ldots,T$,
\begin{equation}
    \dot V_t^\varepsilon
    =
    V_{t,h}^{\pi^\varepsilon,0},
    \qquad
    \dot Q_t^\varepsilon
    =
    Q_{t,h}^{\pi^\varepsilon,0},
    \qquad
    \dot J^\varepsilon
    =
    J^0(h,\pi^\varepsilon)
    =
    \ip*{h,\mu^{\pi^\varepsilon}}.
\end{equation}
\end{lemma}
\begin{proof}
    For the derivative of $Q^{\varepsilon}_T = r_T + \varepsilon h_T$, we obtain $\dot{Q}^{\varepsilon}_T = h_T = Q_{T, h}^{\pi^{\varepsilon}, 0}$. Furthermore, by \eqref{eq:explicit_opt_value}
    \begin{align}
        \dot{V}_T^{\varepsilon}(s)
        &= \dfrac{\diff}{\diff\varepsilon}
        \beta \log \int_{\cA} e^{Q^{\varepsilon}_T(s,\cdot)/\beta}\diff\lambda
        \nonumber\\
        &= \int_{\cA} \dot{Q}^{\varepsilon}_T(s,\cdot)
        \dfrac{e^{Q^{\varepsilon}_T(s,\cdot)/\beta}}
        {\int_{\cA} e^{Q^{\varepsilon}_T(s,\cdot)/\beta}\diff\lambda}
        \diff\lambda
        = V_{T,h}^{\pi^{\varepsilon},0}(s).
    \end{align}
    and
    \begin{align}
        \dot{Q}^{\varepsilon}_{T-1}(s, a)
        &= \dfrac{\diff}{\diff\varepsilon}\br*{r_{T-1}(s,a)
        + \varepsilon h_{T-1}(s,a)
        + \br*{\bP_{T-1} V_{T}^{\varepsilon}}(s,a)} \nonumber\\
        &= h_{T-1}(s,a)
        + \br*{\bP_{T-1} \dot{V}_{T}^{\varepsilon}}(s,a)
        = Q_{T-1,h}^{\pi^{\varepsilon},0}(s,a).
    \end{align}
    Here, boundedness of derivatives allows us to interchange differentiation and integration. It follows then via backward induction that $\dot{V}_t^{\varepsilon} = V_{t, h}^{\pi^{\varepsilon}, 0} $ and $\dot{Q}_t^{\varepsilon} = Q_{t, h}^{\pi^{\varepsilon}, 0}$. Finally, $\dot{J}^{\varepsilon} = \bP_0 \dot{V}_1^{\varepsilon} = \bP_0 V_{1, h}^{\pi^{\varepsilon}, 0} = J^0\br*{h, \pi^{\varepsilon}}$.
\end{proof}

\paragraph{Derivatives with Respect to Parameters}
We now specialize to the linear reward model
\begin{equation}\label{eq:lin_reward_appendix}
    r_{t,\theta}(s,a) \defeq \ip*{\theta,\phi_t(s,a)}, \qquad \theta\in\R^d,
\end{equation} where $\phi_t:\cSA\to\R^d$ is bounded and measurable for every $t=1,\ldots,T$. We write
\begin{equation}
    \Jstar(\theta) \defeq \Jstar(r_\theta), \qquad \pistar_\theta \defeq \pistar_{r_\theta}.
\end{equation} For a policy $\pi$, define the vector-valued feature advantage componentwise by
\begin{equation}
    \left[ A_{t,\phi}^{\pi,0} \right]_i \defeq A_{t,\phi_i}^{\pi,0}.
\end{equation} For $\theta,\xi\in\R^d$, let
\begin{equation}
    Z_\phi^\theta(\tau) \defeq \sum_{t=1}^T A_{t,\phi}^{\pistar_\theta,0}(s_t,a_t), \qquad Z_\xi^\theta(\tau) \defeq \ip*{\xi,Z_\phi^\theta(\tau)}.
\end{equation}
By linearity,
\begin{equation}
    Z_\xi^\theta(\tau) = \sum_{t=1}^T A_{t,r_\xi}^{\pistar_\theta,0}(s_t,a_t).
\end{equation}
The first three derivatives of the optimal value $\Jstar$ are given as follows.
\begin{lemma}[Parameter derivatives]
\label{lem:param_der}
Fix $\beta>0$ and $\theta\in\R^d$. Then, for all
$\xi,\zeta,\omega\in\R^d$,
\begin{align}
    D\Jstar(\theta)[\xi]
    &=
    J^0(\xi,\pistar_\theta)
    =
    \ip*{r_\xi,\mu^{\pistar_\theta}},\\
    D^2\Jstar(\theta)[\xi,\zeta]
    &=
    \beta^{-1}
    \E^{\pistar_\theta}
    \left[
    Z_\xi^\theta Z_\zeta^\theta
    \right],\\
    D^3\Jstar(\theta)[\xi,\zeta,\omega]
    &=
    \beta^{-2}
    \E^{\pistar_\theta}
    \left[
    Z_\xi^\theta Z_\zeta^\theta Z_\omega^\theta
    \right].
\end{align}
Moreover, the trajectory score satisfies
\begin{equation}\label{eq:trajectory_score}
    D_\theta
    \log\ptraj^{\pistar_\theta}(\tau)[\xi]
    =
    \beta^{-1}Z_\xi^\theta(\tau).
\end{equation}
\end{lemma}

\begin{proof}
\textit{First derivative:}
Apply Lemma~\ref{lem:first_der} with
$r^\varepsilon=r_{\theta+\varepsilon\xi}
=r_\theta+\varepsilon r_\xi$ to obtain
\begin{equation}
    D_\theta\Jstar(\theta)[\xi]
    =
    \left.
    \dfrac{\diff}{\diff\varepsilon}
    \Jstar(r^\varepsilon)
    \right|_{\varepsilon=0}
    =
    J^0(\xi,\pistar_\theta).
\end{equation}

\textit{Second derivative:}
Let $\pi^\varepsilon\defeq\pistar_{\theta+\varepsilon\xi}$,
$p_t^\varepsilon\defeq p_t^{\pi^\varepsilon}$, and
$\ptraj^\varepsilon\defeq\ptraj^{\pi^\varepsilon}$. For
$\eta\in\R^d$, define
$Z_\eta^\varepsilon\defeq Z_\eta^{\theta+\varepsilon\xi}$.
Since
$\log\ptraj^\varepsilon(\tau)
=\sum_{t=1}^T\log p_t^\varepsilon(a_t \mid s_t)$,
Lemma~\ref{lem:first_der} gives
\begin{equation}\label{eq:trajectory_score_path}
    \dfrac{\diff}{\diff\varepsilon}
    \log\ptraj^\varepsilon
    =
    \beta^{-1}Z_\xi^\varepsilon,
\end{equation}
where we used that
\begin{equation}
    \dfrac{\diff}{\diff\varepsilon}
    \log p_t^\varepsilon(a \mid s)
    =
    \beta^{-1}
    \left(
    Q_{t,r_\xi}^{\pi^\varepsilon,0}(s,a)
    -
    V_{t,r_\xi}^{\pi^\varepsilon,0}(s)
    \right).
\end{equation}

Let $G_\zeta\defeq\sum_{t=1}^T r_{t,\zeta}$. By the
log-derivative identity,
\begin{equation}
    \dfrac{\diff}{\diff\varepsilon}
    J^0(\zeta,\pi^\varepsilon)
    =
    \int
    G_\zeta
    \dfrac{\diff}{\diff\varepsilon}
    \ptraj^\varepsilon
    \diff\lambdatraj
    =
    \E^{\pi^\varepsilon}
    \left[
    G_\zeta
    \dfrac{\diff}{\diff\varepsilon}
    \log\ptraj^\varepsilon
    \right].
\end{equation}
Together with \eqref{eq:trajectory_score_path}, this yields
\begin{equation}
    \dfrac{\diff}{\diff\varepsilon}
    J^0(\zeta,\pi^\varepsilon)
    =
    \beta^{-1}
    \E^{\pi^\varepsilon}
    \left[
    G_\zeta Z_\xi^\varepsilon
    \right].
\end{equation}

By the return decomposition in Lemma~\ref{lem:return_decomp},
\begin{equation}
    G_\zeta
    =
    J_\zeta^\varepsilon
    +
    Z_\zeta^\varepsilon
    +
    M_\zeta^\varepsilon,
\end{equation}
where $J_\zeta^\varepsilon\defeq J^0(\zeta,\pi^\varepsilon)$ and
$M_\zeta^\varepsilon
\defeq\sum_{t=0}^{T-1}
\delta_{t,r_\zeta}^{\pi^\varepsilon,0}$.
Since $\E^{\pi^\varepsilon}[Z_\xi^\varepsilon]=0$ and the
action-advantage and dynamics-residual terms are orthogonal,
\begin{equation}
    \dfrac{\diff}{\diff\varepsilon}
    J^0(\zeta,\pi^\varepsilon)
    =
    \beta^{-1}
    \E^{\pi^\varepsilon}
    \left[
    Z_\zeta^\varepsilon Z_\xi^\varepsilon
    \right].
\end{equation}
Evaluating at $\varepsilon=0$ proves
\begin{equation}
    D_\theta^2\Jstar(\theta)[\xi,\zeta]
    =
    \beta^{-1}
    \E^{\pistar_\theta}
    \left[
    Z_\xi^\theta Z_\zeta^\theta
    \right].
\end{equation}
Moreover, evaluating \eqref{eq:trajectory_score_path} at
$\varepsilon=0$ gives the score identity
\begin{equation}
    D_\theta
    \log\ptraj^{\pistar_\theta}(\tau)[\xi]
    =
    \beta^{-1}Z_\xi^\theta(\tau).
\end{equation}

\textit{Third derivative:}
Using the notation introduced above,
\begin{equation}
    D_\theta^3\Jstar(\theta)[\xi,\zeta,\omega]
    =
    \left.
    \dfrac{\diff}{\diff\varepsilon}
    \beta^{-1}
    \E^{\pi^\varepsilon}
    \left[
    Z_\zeta^\varepsilon Z_\omega^\varepsilon
    \right]
    \right|_{\varepsilon=0}.
\end{equation}
By the log-derivative identity and the product rule,
\begin{equation}
    \dfrac{\diff}{\diff\varepsilon}
    \beta^{-1}
    \E^{\pi^\varepsilon}
    \left[
    Z_\zeta^\varepsilon Z_\omega^\varepsilon
    \right]
    =
    \beta^{-1}
    \E^{\pi^\varepsilon}
    \left[
    \beta^{-1}
    Z_\xi^\varepsilon Z_\zeta^\varepsilon Z_\omega^\varepsilon
    +
    \dot Z_\zeta^\varepsilon Z_\omega^\varepsilon
    +
    Z_\zeta^\varepsilon\dot Z_\omega^\varepsilon
    \right].
\end{equation}
Differentiating the return decomposition gives
$\dot Z_\zeta^\varepsilon
=-\dot J_\zeta^\varepsilon-\dot M_\zeta^\varepsilon$, where
$\dot J_\zeta^\varepsilon$ is deterministic and
\begin{equation}
    \dot M_\zeta^\varepsilon
    =
    \sum_{t=0}^{T-1}
    \dot\delta_{t,\zeta}^\varepsilon,
    \qquad
    \dot\delta_{t,\zeta}^\varepsilon
    =
    \dot V_{t+1,\zeta}^\varepsilon(s_{t+1})
    -
    \left(
    \bP_t\dot V_{t+1,\zeta}^\varepsilon
    \right)(s_t,a_t).
\end{equation}
Since $\E^{\pi^\varepsilon}[Z_\omega^\varepsilon]=0$ and $
    \E^{\pi^\varepsilon}\bs{
    \dot\delta_{t,\zeta}^\varepsilon
 \mid s_t,a_t}
    =
    0,
$
the same action--dynamics orthogonality as above gives
\begin{equation}
    \E^{\pi^\varepsilon}
    \left[
    \dot Z_\zeta^\varepsilon Z_\omega^\varepsilon
    \right]
    =
    \E^{\pi^\varepsilon}
    \left[
    Z_\zeta^\varepsilon\dot Z_\omega^\varepsilon
    \right]
    =
    0.
\end{equation}
Evaluating at $\varepsilon=0$ therefore yields
\begin{equation}
    D_\theta^3\Jstar(\theta)[\xi,\zeta,\omega]
    =
    \beta^{-2}
    \E^{\pistar_\theta}
    \left[
    Z_\xi^\theta Z_\zeta^\theta Z_\omega^\theta
    \right].
\end{equation}
\end{proof}

We next record two consequences of Lemma~\ref{lem:param_der}.

\begin{corollary}[Bregman divergence and trajectory KL]
\label{cor:Bregman}
For $\theta,\theta'\in\R^d$, define
\begin{equation}
    D_{\Jstar}(\theta,\theta')
    \defeq
    \Jstar(\theta)
    -
    \Jstar(\theta')
    -
    \ip*{
    \theta-\theta',
    \nabla\Jstar(\theta')
    }.
\end{equation}
Then
\begin{equation}
    D_{\Jstar}(\theta,\theta')
    =
    \Jstar(\theta)
    -
    J(\theta,\pistar_{\theta'})
    =
    \beta
    \DKL\left(
    \bP^{\pistar_{\theta'}},
    \bP^{\pistar_\theta}
    \right).
\end{equation}
\end{corollary}

\begin{proof}
By Lemma~\ref{lem:param_der},
\begin{equation}
    \ip*{
    \theta-\theta',
    \nabla\Jstar(\theta')
    }
    =
    \ip*{
    r_\theta-r_{\theta'},
    \mu^{\pistar_{\theta'}}
    }.
\end{equation}
Since
$\Jstar(\theta') + \ip*{
    r_\theta-r_{\theta'},
    \mu^{\pistar_{\theta'}}} = J(\theta,\pistar_{\theta'})$, this proves the
first equality. The second equality follows from Corollary~\ref{cor:soft_suboptimality}.
\end{proof}

\begin{corollary}[Fisher information and Hessian]
\label{cor:fisher_hessian}
For $\theta\in\R^d$, define the Fisher information matrix of the
trajectory law $\bP^{\pistar_\theta}$ by
\begin{equation}
    \mathsf{I}(\theta)
    \defeq
    \E^{\pistar_\theta}
    \left[
    \nabla_\theta
    \log\ptraj^{\pistar_\theta}(\tau)
    \br*{\nabla_\theta
    \log\ptraj^{\pistar_\theta}(\tau)}^\top
    \right].
\end{equation}
Then
\begin{equation}
    \mathsf{H}(\theta)
    \defeq
    \nabla^2\Jstar(\theta)
    =
    \beta^{-1}
    \E^{\pistar_\theta}
    \left[
    Z_\phi^\theta
    \left(
    Z_\phi^\theta
    \right)^\top
    \right]
    =
    \beta \mathsf{I}(\theta).
\end{equation}
Moreover,
\begin{equation}
    \mathsf{H}(\theta)
    =
    \beta^{-1}
    \sum_{t=1}^T
    \E^{\pistar_\theta}
    \left[
    A_{t,\phi}^{\pistar_\theta,0}
    \left(
    A_{t,\phi}^{\pistar_\theta,0}
    \right)^\top
    \right].
\end{equation}
\end{corollary}

\begin{proof}
The first identity follows from the second-derivative formula in
Lemma~\ref{lem:param_der}. The score identity
\eqref{eq:trajectory_score} gives
\begin{equation}
    \mathsf{I}(\theta)
    =
    \beta^{-2}
    \E^{\pistar_\theta}
    \left[
    Z_\phi^\theta
    \left(
    Z_\phi^\theta
    \right)^\top
    \right],
\end{equation}
and therefore $\mathsf{H}(\theta)=\beta \mathsf{I}(\theta)$. Finally, the
action-advantage terms are martingale differences and are orthogonal
across time. Hence,
\begin{equation}
    \E^{\pistar_\theta}
    \left[
    Z_\phi^\theta
    \left(
    Z_\phi^\theta
    \right)^\top
    \right]
    =
    \sum_{t=1}^T
    \E^{\pistar_\theta}
    \left[
    A_{t,\phi}^{\pistar_\theta,0}
    \left(
    A_{t,\phi}^{\pistar_\theta,0}
    \right)^\top
    \right].
\end{equation}
\end{proof}

\section{Identifiability and Potential Shaping}\label{app:sec:identifiability_and_potential_shaping}

This section extends reward-identifiability results for entropy-regularized IRL in tabular MDPs \citep{cao2021identifiability, shehab2024learning}, to our Borel state and action space setting. The main result is that for a fixed initial distribution and transition kernel, rewards are identifiable only up to potential-shaping transformations and modifications on $\bP^{\pistar}$-null sets.

For a policy $\pi$, define the null subspace \begin{equation} \cN^{\pi} \defeq \bc*{ n \in B_b^T(\cSA): n_t(s_t, a_t) = 0 \quad \bP^{\pi}\text{-a.s. for every }t }, \end{equation} and the subspace of potential-shaping transformations \citep{ng1999policy} \begin{equation} \cU \defeq \bc*{ u \in B_b^T(\cSA) \,:\, \exists\, \psi_1,\hdots,\psi_T \in B_b(\cS),\; \psi_{T+1}=0,\; u_t = \psi_t - \bP_t \psi_{t+1} }. \end{equation} Finally, let \begin{equation} \cU^{\pi} \defeq \cU + \cN^{\pi}. \end{equation} For $\beta>0$, we write $\cU^{\pistar}$ for this space under any soft-optimal trajectory law, since Proposition~\ref{prop:common_support} shows that all such laws have the same null sets.

The next proposition characterizes $\cU$ as the class of rewards whose
unregularized advantage is zero for every policy, whereas $\cU^\pi$ is the class of rewards whose unregularized advantage under $\pi$ is zero
$\bP^{\pi}$-almost surely.
\begin{proposition}
\label{prop:pot_shaping_equiv}
For $u\in B_b^T(\cSA)$, the following are equivalent:
\begin{enumerate}
    \item We have $u\in\cU$.
    \item There exist $\psi_1,\hdots,\psi_T\in B_b(\cS)$ with
    $\psi_{T+1}=0$ such that, for every policy $\pi$,
    \begin{equation}
        V_{t,u}^{\pi,0}
        =
        Q_{t,u}^{\pi,0}
        =
        \psi_t,
        \qquad
        \forall t .
    \end{equation}
    \item For every policy $\pi$,
    \begin{equation}
        A_{t,u}^{\pi,0}=0,
        \qquad
        \forall t.
    \end{equation}
\end{enumerate}

Moreover, for a fixed policy $\pi$ and $h\in B_b^T(\cSA)$, the following are
equivalent:
\begin{enumerate}[resume]
    \item We have $h\in\cU^\pi$.
    \item There exist $\psi_1,\hdots,\psi_T\in B_b(\cS)$ with
    $\psi_{T+1}=0$ such that
    \begin{equation}
        V_{t,h}^{\pi,0}(s_t)
        =
        Q_{t,h}^{\pi,0}(s_t,a_t)
        =
        \psi_t(s_t),
        \qquad
        \bP^{\pi}\text{-a.s.}, \quad \forall t.
    \end{equation}
    \item We have
    \begin{equation}
        A_{t,h}^{\pi,0}(s_t,a_t)=0,
        \qquad
        \bP^{\pi}\text{-a.s.}, \quad \forall t.
    \end{equation}
\end{enumerate}
\end{proposition}
\begin{proof}
For the first set of equivalences, suppose first that $u\in\cU$, so that
\begin{equation}
    u_t=\psi_t-\bP_t\psi_{t+1}
\end{equation}
for some $\psi_1,\hdots,\psi_T\in B_b(\cS)$ with $\psi_{T+1}=0$. Backward
induction yields, for every policy $\pi$,
\begin{equation}
    V_{t,u}^{\pi,0}=Q_{t,u}^{\pi,0}=\psi_t,
\end{equation}
and hence also $A_{t,u}^{\pi,0}=0$.

Conversely, if $A_{t,u}^{\pi,0}=0$ for every $t$ and every policy $\pi$, then
\begin{equation}
    0
    = Q_{t,u}^{\pi,0} - V_{t,u}^{\pi,0} =
    u_t+\bP_tV_{t+1,u}^{\pi,0}-V_{t,u}^{\pi,0}.
\end{equation}
Thus, with $\psi_t=V_{t,u}^{\pi,0}$, we have $Q_{t,u}^{\pi,0} = \psi_t$ and
\begin{equation}
    u_t=\psi_t-\bP_t\psi_{t+1} \in \cU.
\end{equation}

For the second set of equivalences, suppose $h\in\cU^\pi$. Then
$h=u+n$ with $u\in\cU$ and $n\in\cN^\pi$. By the first part,
\begin{equation}
    V_{t,u}^{\pi,0}=Q_{t,u}^{\pi,0}=\psi_t .
\end{equation}
Moreover, since $n_t(s_t,a_t)=0$ $\bP^{\pi}$-a.s. for every $t$, the return
$G_{t,n}\defeq\sum_{k=t}^T n_k(s_k,a_k)$ satisfies
$G_{t,n}=0$ $\bP^{\pi}$-a.s. Hence, by the tower property,
\begin{equation}
    \E^{\pi}\!\bs*{\abs*{V_{t,n}^{\pi,0}(s_t)}}
    =
    \E^{\pi}\!\bs*{
        \abs*{\E^{\pi}\!\bs*{G_{t,n} \mid s_t}}
    }
    \leq
    \E^{\pi}\!\bs*{\E^{\pi}\!\bs*{\abs*{G_{t,n}} \mid s_t}}
    =
    0,
\end{equation}
and similarly $\E^{\pi}\!\bs*{\abs*{Q_{t,n}^{\pi,0}(s_t, a_t)}} \leq 0$. So,
\begin{equation}
    V_{t,n}^{\pi,0}(s_t)
    =
    Q_{t,n}^{\pi,0}(s_t,a_t)
    =
    0,
    \qquad
    \bP^{\pi}\text{-a.s.},
\end{equation}
which yields
\begin{equation}
    V_{t,h}^{\pi,0}(s_t)
    =
    Q_{t,h}^{\pi,0}(s_t,a_t)
    =
    \psi_t(s_t),
    \qquad
    \bP^{\pi}\text{-a.s.},
\end{equation}
and $A_{t,h}^{\pi,0}(s_t,a_t)=0$ $\bP^{\pi}$-a.s.

Conversely, if $A_{t,h}^{\pi,0}(s_t,a_t)=0$ $\bP^{\pi}$-a.s., then
\begin{equation}
    h_t(s_t,a_t)
    =
    V_{t,h}^{\pi,0}(s_t)
    -
    \bP_tV_{t+1,h}^{\pi,0}(s_t,a_t),
    \qquad
    \bP^{\pi}\text{-a.s.}
\end{equation}
Define
\begin{equation}
    u_t
    \defeq
    V_{t,h}^{\pi,0}
    -
    \bP_tV_{t+1,h}^{\pi,0}.
\end{equation}
Then $u\in\cU$ and $h-u\in\cN^\pi$, hence $h\in\cU^\pi$.
\end{proof}

The following theorem is the main identifiability statement: reward
transformations leave the optimal trajectory law invariant exactly when they
belong to $\cU^{\pistar}$. Equivalently, these are precisely the transformations
whose unregularized advantage under the optimal policy vanishes almost surely,
a characterization used in Corollary~\ref{cor:fisher_kernel_shaping}.

\begin{theorem}\label{thm:pot_shaping}
Let $\beta>0$ and let $r,h\in B_b^T(\cS\times\cA)$. Then, the following
are equivalent:
\begin{enumerate}
    \item $h\in\cU^{\pistar}$.
    \item $\bP^{\pistar_r}=\bP^{\pistar_{r+h}}$.
    \item $A_{t,h}^{\pistar_r, 0}(s_t,a_t) = 0, \; \bP^{\pistar}$-a.s., $\; \forall t.$
\end{enumerate}
Moreover, for the soft-optimal densities in
\eqref{eq:explicit_opt_policy},
\begin{equation}\label{eq:policy_identifiability}
    \pstar_r=\pstar_{r+h}
    \quad\iff\quad
    h\in\cU.
\end{equation}
\end{theorem}

\begin{proof}
    The equivalence of \emph{1.} and \emph{3.} follows directly from Proposition~\ref{prop:pot_shaping_equiv}.

    \emph{2. $\implies$ 1.}: Suppose that $\bP^{\pistar_r} = \bP^{\pistar_{r + h}}$. Then, by Proposition~\ref{prop:bellman_value_optimality} we have
    \begin{equation}
        \pstar_{t,r + h}(a_t \mid s_t) = \pstar_{t,r}(a_t \mid s_t), \quad \bP^{\pistar}\text{-a.s.}, \; \forall t.
    \end{equation}
    Therefore, for $\psi_t = \Vstar_{t, r + h} - \Vstar_{t, r}$ with $\psi_{T+1} = 0$, we get
    \begin{equation}\label{eq:log_ratio_soft_opt}
        0 = \beta \log \br*{\dfrac{\pstar_{t,r + h}(a_t \mid s_t)}{\pstar_{t,r}(a_t \mid s_t)}} = h_t(s_t, a_t) + \br*{\bP_t \psi_{t+1}}(s_t, a_t) - \psi_t(s_t), \quad \bP^{\pistar}\text{-a.s.}, \; \forall t,
    \end{equation}
    and so $h \in \cU^{\pistar}$.

    \emph{1. $\implies$ 2.}: Let $h \in \cU^{\pistar}$. Then, there exist $u \in \cU$ and $n \in \cN^{\pistar}$ such that $h = u + n$. For the expected value of a policy $\pi \ll \lambda$, we have
    \begin{align}
        J(r + h, \pi) = J(r, \pi) + \underbrace{\sum_{t=1}^T \E^{\pi}\bs*{ u_t}}_{=\bP_0 \psi_1} + \underbrace{\sum_{t=1}^T\E^{\pi}\bs*{ n_t}}_{=0} = J(r, \pi) + \text{constant}.
    \end{align}
    Since $\pistar_{r+h}$ maximizes $J(r + h, \cdot)$, it must therefore also maximize $J(r, \cdot)$. That is, $\pistar_{r+h} \in \Pi^\star_{\mathsf{EV}}(r)$, which by Proposition~\ref{prop:bellman_value_optimality} implies that $\bP^{\pistar_r} = \bP^{\pistar_{r + h}}$.

    Finally, if $h\in\cU$, backward induction yields
    \begin{equation}
        V^\star_{t,r+h}
        =
        V^\star_{t,r}+\psi_t,
        \qquad
        Q^\star_{t,r+h}
        =
        Q^\star_{t,r}+\psi_t,
    \end{equation}
    and hence $\pstar_r=\pstar_{r+h}$. Conversely, if the soft-optimal densities are equal, then \eqref{eq:log_ratio_soft_opt} holds pointwise, and hence $h\in\cU$.
\end{proof}

We now specialize the above identifiability result to finite-dimensional linear reward classes.
Recall the linear reward parametrization \eqref{eq:lin_reward_appendix},
$r_{t,\theta}(s,a)=\ip{\theta,\phi_t(s,a)}$. In this case, the unidentifiable
parameter directions are exactly those whose induced reward perturbation lies
in $\cU^{\pistar}$; equivalently, they are the kernel of the Hessian of the
optimal value.
\begin{corollary}\label{cor:fisher_kernel_shaping}
Let
\begin{equation}
    \Phi:\R^d \to B_b^T(\cS\times\cA),
    \qquad
    \Phi(\theta) \defeq r_\theta,
\end{equation}
denote the linear reward parametrization map. Then, for every $\theta \in \R^d$, we have
\begin{equation}
    \ker \mathsf{H}(\theta) = \Phi^{-1}(\cU^{\pistar}) = \bc*{\xi\in\R^d: \bP^{\pistar_{\theta}} = \bP^{\pistar_{\theta + \xi}}},
    \qquad
    \ima \mathsf{H}(\theta) = \br*{\Phi^{-1}(\cU^{\pistar})}^\perp.
\end{equation}
In particular, both $\ker \mathsf{H}(\theta)$ and $\ima \mathsf{H}(\theta)$ are independent of $\theta$.
\end{corollary}
\begin{proof}
By Lemma~\ref{lem:param_der}, we have
\begin{equation}
    \xi^\top \mathsf{H}(\theta) \xi
    =
    \beta^{-1}\E^{\pi_\theta^{\star}}\bs*{\br*{Z^{\theta}_{\xi}}^2} = \beta^{-1}\sum_{t=1}^T \E^{\pi_\theta^{\star}}\bs*{\br*{A_{t, \xi}^{\pistar_{\theta}, 0}}^2}.
\end{equation}
Hence, $\xi \in \ker \mathsf{H}(\theta)$ if and only if\footnote{Recall that by Proposition~\ref{prop:common_support}, the null sets of $\bP^{\pistar_{\theta}}$ are independent of $\theta$.} $A_{t, \xi}^{\pistar_{\theta}, 0}=0\; \bP^{\pistar}$-a.s., which by Theorem~\ref{thm:pot_shaping} is equivalent to $r_\xi=\Phi(\xi)\in\cU^{\pistar}$, \ie, $\xi \in \Phi^{-1}(\cU^{\pistar})$, and $\bP^{\pistar_{\theta}} = \bP^{\pistar_{\theta + \xi}}$. Since $\mathsf{H}(\theta)$ is symmetric, $\ima \mathsf{H}(\theta) = (\ker \mathsf{H}(\theta))^\perp$.
\end{proof}

\section{Structural Equivalences}\label{app:sec:structural_equivalences}
In this section, we provide the formal structural equivalence results from Section~\ref{sec:equivalences}. We first define equivalence of optimization problems.

\begin{definition}[Equivalence of optimization problems]\label{def:equivalence}
Let
\begin{equation}
    \mathsf{A}\br*{\cX}
    :\quad
    \operatorname*{\min}_{x\in\cX} f(x),
    \qquad
    \mathsf{B}\br*{\cY}
    :\quad
    \operatorname*{\min}_{y\in\cY} g(y),
\end{equation}
be two minimization problems. We say that
$ \mathsf{A}$ and $\mathsf{B}$ are equivalent if there exists a bijection
$\Lambda:\cX\to\cY$ and a strictly increasing function $h:\R\to\R$ such
that
\begin{equation}
    g(\Lambda(x))
    =
    h(f(x))
    \qquad
    \text{for all } x\in\cX .
\end{equation}
Consequently, whenever the argmin is nonempty,
\begin{align}
    \Lambda\!\left(
        \argmin_{x\in\cX} f(x)
    \right)
    &=
    \argmin_{y\in\cY} g(y).
\end{align}
\end{definition}

In the following, we establish two equivalences between Min-Max-IRL and MLE. The first one is between Min-Max-IRL over a set of rewards $\cR\subseteq B_b^T(\cSA)$ and MLE over the set of soft-optimal densities $\cP^{\star}(\cR) = \bc{\pstar_r: r \in \cR}$. Since several rewards can induce the same soft-optimal density, this map is not injective on $\cR$, and we pass to a quotient. By Theorem~\ref{thm:pot_shaping}, $\pstar_r = \pstar_{r'}$ if and only if $r-r' \in\cU$, where $\cU$ is the space of potential shaping transformations. We therefore use the equivalence classes $[r]_{\cU} \defeq r + \cU$, $r\in B_b^T(\cSA)$, which form the quotient space $B_b^T(\cSA) / \cU$, and consider the subset induced by $\cR$,
\begin{equation}
    \bs*{\cR}_{\cU} \defeq \bc*{[r]_{\cU}: r \in \cR}.
\end{equation}

The second equivalence is between Min-Max-IRL over $\beta \log \cP$ and MLE over $\cP$. Here, $\cP$ denotes a class of \emph{policy densities}, that is, a class such that any $p\in\cP$ defines a Markov policy $\pi_p$ (\ie, a stochastic kernel) via $\pi_{t, p}(\diff a \mid s) = p_t(a \mid s)\, \lambda(\diff a)$. We say that $\cP$ has \emph{bounded log-densities} if $\br{\beta \log p_t}_{t=1}^T \in B_b^T(\cSA)$ for any $p\in\cP$; then $\beta\log\cP \subseteq B_b^T(\cSA)$ as required for rewards in our setup.

To formally define these equivalences, we need the following bijections.
\begin{proposition}\label{prop:bijection}
    Let $\cR \subseteq B_b^T(\cSA)$ be a reward class and $\cP$ a policy density class with bounded log-densities. Then:
    \begin{enumerate}
        \item $\Lambda: \bs*{\cR}_{\cU} \to \cP^{\star}(\cR)$, defined by $\Lambda\br*{[r]_{\cU}} = \pstar_r$, is a bijection with inverse $\Lambda^{-1}\br*{p} = [\beta \log p]_{\cU}$.
        \item $\Gamma: \beta\log\cP \to \cP$, defined by $\Gamma(r) = \pstar_r$, is a bijection with inverse $\Gamma^{-1}(p) = \beta \log p$.
    \end{enumerate}
    Here $\beta \log p = \br{\beta \log p_t}_{t=1}^T$ is the myopic reward given by $\br*{\beta \log p}_t(s,a) = \beta \log p_t(a \mid s)$.
\end{proposition}
\begin{proof}
    \emph{Part 1.} By Theorem~\ref{thm:pot_shaping}, $\pstar_r = \pstar_{r'}$ if and only if $r' \in [r]_{\cU}$, so $\Lambda$ is well-defined on equivalence classes and injective. It is surjective by the definition of $\cP^{\star}(\cR)$. For the inverse, let $p = \pstar_r$. By \eqref{eq:explicit_opt_policy},
    \begin{align}
        r_t(s, a) - \beta \log \pstar_{t,r}(a \mid s)
        &= r_t(s, a) - \br*{ r_t(s, a) + \br*{\bP_t \Vstar_{t+1,r}}(s,a) - \Vstar_{t,r}(s) }\\
        &= \Vstar_{t,r}(s) - \br{\bP_t \Vstar_{t+1,r}}(s,a),
    \end{align}
    which implies that $r- \beta \log \pstar_r \in \cU$. So $\beta\log p \in B_b^T(\cSA)$ and $\beta \log p \in [r]_{\cU}$, \ie, $\Lambda^{-1}\br*{p} = [\beta \log p]_{\cU}$.

    \emph{Part 2.} Clearly, $\Gamma^{-1}(p) = \beta \log p$ is injective and surjective by definition of $\beta \log \cP$. It remains to show that $\Gamma^{-1}(p) = \beta \log p$ and $\Gamma(r) = \pstar_r$ are inverses of each other. Let $r= \beta \log p$ for $p\in\cP$. By the Bellman optimality equations \eqref{eq:explicit_opt_value} and backward induction, if $\Vstar_{t+1,r} = 0$, then we have
    \begin{align}
        \Qstar_{t,r}(s,a)
        &=r_t(s,a)+\br{\bP_t\Vstar_{t+1,r}}(s,a)
        =\beta\log p_t(a\mid s),\nonumber\\
        \Vstar_{t,r}(s)
        &=\beta\log\int_{\cA}p_t(a \mid s)\,\lambda(\diff a)=0.
        \label{eq:backward_induction_equivalence}
    \end{align}
    As $\Vstar_{T+1,r} = 0$, \eqref{eq:backward_induction_equivalence} holds for all $t=1,\hdots, T$. Therefore,
    \begin{equation*}
        \pstar_{t,r}(a \mid s)
        = e^{\beta^{-1} \br{\Qstar_{t,r}(s,a)-\Vstar_{t,r}(s)}}
        = p_t(a \mid s),
    \end{equation*}
    implying that $\Gamma \circ \Gamma^{-1}(p) = p$.
\end{proof}
Part~2 of the proof above shows why we refer to $\beta \log p$ as myopic rewards: the policy $\pi_p$ is soft-optimal for $r = \beta\log p$ with $\Vstar_{t,r} = 0$.

We are now ready to state the formal equivalence result. For clarity, we write $[r]\defeq [r]_{\cU}$ in the following theorem.

\begin{theorem}\label{thm:equivalences_formal}
    Let $\beta > 0$. Consider a reward class $\cR \subseteq B_b^T(\cSA)$ and a policy density class $\cP$ with bounded log-densities. Let $\cP^{\star}(\cR)$ and $\beta \log \cP$ be defined as in \eqref{eq:soft_opt_densities} and \eqref{eq:myopic_rewards}, and write $\LIRL([r])\defeq\LIRL(r)$ and $\LhatIRL([r])\defeq\LhatIRL(r)$ in Part~1. Then we have the following equivalences between optimization problems (see Definition~\ref{def:equivalence}):
    \begin{enumerate}
        \item \begin{enumerate}[label=(\alph*)]
            \item The problems
            \begin{equation*}
                \min_{[r]\in\bs*{\cR}_{\cU}} \LIRL([r])
                \quad\text{and}\quad
                \min_{p\in\cP^{\star}(\cR)} \LMLE(p)
            \end{equation*}
            are equivalent.
            \item Let $\br{\bP_t}_{t=0}^{T-1}$ be deterministic. Then, the problems
            \begin{equation*}
                \min_{[r]\in\bs*{\cR}_{\cU}} \LhatIRL([r])
                \quad\text{and}\quad
                \min_{p\in\cP^{\star}(\cR)} \LhatMLE(p)
            \end{equation*}
            are equivalent.
        \end{enumerate}
        \item \begin{enumerate}[label=(\alph*)]
            \item $\min_{r\in \beta\log\cP} \LIRL(r)$ and $\min_{p\in\cP} \LMLE(p)$ are equivalent.
            \item $\min_{r\in \beta\log\cP} \LhatIRL(r)$ and $\min_{p\in\cP} \LhatMLE(p)$ are equivalent.
        \end{enumerate}
    \end{enumerate}
\end{theorem}
\begin{proof}
\emph{Part 1.} Since $\pstar_r$ is soft-optimal, we have $A^{\pstar_r}_{t,r} = 0$, so Lemma~\ref{lem:return_decomp} gives for any trajectory $(s_1, a_1, \hdots, s_T, a_T)$:
\begin{equation}
    \Jstar(r) - \sum_{t=1}^T r_t(s_t, a_t) + \sum_{t=0}^{T-1} \delta^{\star}_{t,r}(s_t, a_t, s_{t+1}) = -\beta \sum_{t=1}^T \log \pstar_{t,r}(a_t \mid s_t).
\end{equation}
Taking the expectation with respect to $\E^{\piE}$ and $\Ehat^{\piE}$ yields
\begin{equation}\label{eq:equivalence_losses}
    \LIRL(r) = \beta \LMLE(\pstar_r), \quad \LhatIRL(r) + \sum_{t=0}^{T-1} \Ehat^{\piE}\bs*{\delta^{\star}_{t,r}} = \beta \LhatMLE(\pstar_r),
\end{equation}
where in the first identity we used that $\E^{\piE}\bs{\delta^{\star}_{t,r}} = \E^{\piE}\bs*{\E^{\piE}\bs{\delta^{\star}_{t,r} \mid s_t, a_t}} = 0$. If $\br{\bP_t}_{t=0}^{T-1}$ are deterministic, then $\delta^{\star}_{t,r}(s_t, a_t, s_{t+1}) = \Vstar_{t+1, r}(s_{t+1}) - \br*{\bP_t \Vstar_{t+1, r}}(s_t, a_t) = 0$ everywhere, so also $\LhatIRL(r) = \beta \LhatMLE(\pstar_r)$.

By Theorem~\ref{thm:pot_shaping} both $[r]\mapsto \LMLE(\pstar_r)$ and $[r]\mapsto \LhatMLE(\pstar_r)$ are well-defined on $\bs*{\cR}_{\cU}$, hence \eqref{eq:equivalence_losses} implies that so is $\LIRL$, and $\LhatIRL$ if $\br{\bP_t}_{t=0}^{T-1}$ are deterministic. With the bijection $\Lambda$ of Proposition~\ref{prop:bijection}, we have $\LMLE(\Lambda([r])) = \beta^{-1}\LIRL([r])$, which is Definition~\ref{def:equivalence} with $h(x)=\beta^{-1}x$ (strictly increasing since $\beta>0$), establishing (a). In the deterministic setting also $\LhatMLE(\Lambda([r])) = \beta^{-1}\LhatIRL([r])$, establishing (b).

\emph{Part 2.} By Proposition~\ref{prop:bijection}, $\Gamma:\beta\log\cP \to \cP$, $\Gamma(r)=\pstar_r$, is a bijection with $\Gamma^{-1}(p)=\beta \log p$. As shown in the proof of Proposition~\ref{prop:bijection}, we have $\Vstar_{t,r}= 0$ for $r = \beta\log p$. Hence $\delta^{\star}_{t,r}= 0$, so \eqref{eq:equivalence_losses} gives for every $r\in\beta\log\cP$ that $\LIRL(r) = \beta \LMLE(\Gamma(r))$ and $\LhatIRL(r) = \beta \LhatMLE(\Gamma(r))$, establishing (a) and (b).
\end{proof}
The following result shows that Min-Max-IRL is convex, while MLE-IRL is nonconvex in the stochastic case.
\begin{restatable}{proposition}{convexity}\label{prop:convexity}
    Let $\beta > 0$.
    \begin{enumerate}
    \item The map $r\mapsto\LhatIRL(r)$ is convex.
    \item Let $\br{\bP_t}_{t=1}^{T-1}$ be deterministic, then $r\mapsto\LhatMLE(\pstar_r)$ is convex.
    \item There exists an MDP such that $r\mapsto\LhatMLE(\pstar_r)$ is nonquasiconvex.
\end{enumerate}
\end{restatable}
\begin{proof}
\emph{Part 1.} The Min-Max-IRL loss $\LhatIRL$ is a pointwise maximum of affine functions and therefore convex.

\emph{Part 2.} Let $\Vstar_r$ denote the soft-optimal value. If only the initial distribution $\bP_0$ is nondeterministic, then we have
\begin{align}
    \beta\LhatMLE(\pstar_r) &= \dfrac{1}{n}\sum_{i=1}^n \sum_{t=1}^T \br*{\Vstar_{t, r}(s_t^i) - r_t(s_t^i, a_t^i) - \Vstar_{t+1, r}(s_{t+1}^i)}\\
    &= \hat{\bP}_{0}\Vstar_{1, r} - \sum_{t=1}^T \ip*{r_t, \muhat^{\piE}_{t}},
\end{align}
which equals the Min-Max-IRL loss $\LhatIRL(r)$ for the initial distribution $\bP_0 = \hat{\bP}_{0}$.

\emph{Part 3.} We want to show that $f(r)\defeq \LhatMLE(\pstar_r)$ fails to be quasiconvex in general. To this end, we construct an example where $f\br{\frac{r + r'}{2}}> \max\bc{f(r), f(r')}$. We consider the following MDP with horizon $T=2$, state and action spaces $\mathcal S=\{x,y\}$ and $\mathcal A=\{a,b\}$, and regularization parameter $\beta = 1$. At $t=1$ the MDP starts in $s_1 = x$, and evolves as follows:
\begin{equation}
    \bP_1(y \mid x,a)=1,\quad \bP_1(x \mid x,b)=\tfrac{1}{2},\quad \bP_1(y \mid x,b)=\tfrac{1}{2}, \quad \bP_1(\;\cdot \mid y,\cdot)= \text{arbitrary}.
\end{equation}
We consider a reward $r_{\theta}$ parametrized by $\theta\defeq (\theta_x,\theta_y)$ as follows: At $t=1$, we have $r_1(\cdot,\cdot)=0$, and at $t=2$:
\begin{equation}
    r_2(x,a)=\theta_x,\quad r_2(x,b)=0,\qquad r_2(y,a)=\theta_y,\quad r_2(y,b)=0.
\end{equation}
At $t=2$, this yields the optimal values
\begin{equation}
    V_2^{\star}(x)=\log(1+e^{\theta_x}),\qquad V_2^{\star}(y)=\log(1+e^{\theta_y}),
\end{equation}
and at $t=1$ in state $x$,
\begin{equation}
    Q_1^{\star}(x,a)=V_2^{\star}(y),\qquad Q_1^{\star}(x,b)=\tfrac{1}{2}V_2^{\star}(x)+\tfrac{1}{2}V_2^{\star}(y).
\end{equation}
Now, consider the trajectory $\tau=(x,b,y,a)$. Using  $-\log\pi_1^{\star}(b \mid x)=\log\!\big(1+\exp(Q_1^{\star}(x,a)-Q_1^{\star}(x,b))\big)$ and
$-\log\pi_2^{\star}(a \mid y)=\log(1+e^{-\theta_y})$, the dataset consisting of this single trajectory has the negative log-likelihood
\begin{equation}
    f(r_{\theta}) = \log\br*{1+\exp\br*{\tfrac{1}{2}(\log(1+e^{\theta_y})-\log(1+e^{\theta_x}))}}
+\log(1+e^{-\theta_y}).
\end{equation}
Let $r$ be parametrized by $\theta=(2,4)$ and $r'$ by $\theta'=(-4,2)$, with midpoint
$\tfrac{1}{2}(\theta+\theta')=(-1,3)$. A direct evaluation gives
\begin{equation}
f(r)\approx 1.2919,\qquad
f(r')\approx 1.4802,\qquad
f(\tfrac{r+r'}{2})\approx 1.6431,
\end{equation}
hence $f(\tfrac{r+r'}{2})>\max\bc{f(r),f(r')}$, violating quasiconvexity.
\end{proof}

\section{Fast-Rate Upper Bounds for Min-Max-IRL}\label{app:sec:irl_upper_bounds}

We first show how the cumulative-advantage
bound controls the trajectory density ratios and relative Hessian
variation (pseudo-self-concordance). We then establish global and localized fast-rate
guarantees. Finally, we decompose the effective dimension and provide the concentration inequality used in the analysis.

\subsection{Cumulative Advantage and Density Ratios}
\label{app:sec:advantage_and_density_ratio_bound}

Recall the cumulative advantage bound \eqref{eq:advantage_bound}. Using the notation
\begin{equation}
    Z_\phi^\theta(\tau)
    =
    \sum_{t=1}^T
    A_{t,\phi}^{\pistar_\theta,0}(s_t,a_t),
\end{equation}
we require the bound to hold only almost surely in the uniform sense
\begin{equation}\label{eq:cumulative_advantage_bound}
    \sup_{\theta\in\R^d}\norm*{Z_\phi^\theta(\tau)}
    \leq
    B_{A_\phi},
    \quad
    \bP^{\pistar}\text{-a.s.}
\end{equation}
Under Assumption~\ref{ass:linear_model}, this condition holds with $B_{A_\phi}=2TB_\phi$, since
\begin{equation}
    \norm*{Z_\phi^\theta(\tau)}
    \leq
    \sum_{t=1}^T
    \norm*{A_{t,\phi}^{\pistar_\theta,0}(s_t,a_t)}
    \leq
    2TB_\phi.
\end{equation}
If the transition dynamics are deterministic, the return
decomposition (Lemma~\ref{lem:return_decomp}) instead gives
\begin{equation}
    Z_\phi^\theta
    =
    \phi(\tau)
    - \phi(\pistar_{\theta}), \quad \forall \theta\in\R^d, \quad \bP^{\pistar}\text{-a.s.},
\end{equation}
and hence one may take $B_{A_\phi}=2B_\phi$.

The cumulative-advantage bound yields the following density ratio bound.
\begin{proposition}\label{prop:density_ratio_bound}
    It holds that
    \begin{equation}
        \abs*{\log
    \br*{\dfrac{
        \ptraj^{\pistar_\theta}(\tau)
    }{
        \ptraj^{\pistar_{\theta'}}(\tau)
    }}} \leq \beta^{-1} B_{A_{\phi}} \norm*{\theta - \theta'}, \quad \forall \theta, \theta'\in\R^d, \quad \bP^{\pistar}\text{-a.s.}
    \end{equation}
\end{proposition}
\begin{proof}
    By \eqref{eq:cumulative_advantage_bound}, there exists a measurable
    set $\cT_0\subseteq(\cSA)^T$ such that
    \begin{equation}
        \bP^{\pistar_\theta}(\cT_0)=1,
        \qquad
        \forall\theta\in\R^d,
    \end{equation}
    and
    \begin{equation}
        \sup_{\theta\in\R^d}
        \norm*{Z_\phi^\theta(\tau)}
        \leq B_{A_\phi},
        \qquad
        \forall\tau\in\cT_0.
    \end{equation}
    Fix $\tau\in\cT_0$ and recall from the proof of Lemma~\ref{lem:param_der} that we have the score identity
    \begin{equation}
        D_{\theta} \br*{\log \ptraj^{\pistar_{\theta}}(\tau)}[\xi] = \beta^{-1} Z_{\xi}^{\theta}(\tau) =
        \beta^{-1} \ip*{\xi, Z_{\phi}^{\theta}(\tau)}, \quad \forall \xi\in\R^d.
    \end{equation}
    Hence, it follows that
    \begin{equation}
        \abs*{\log
    \br*{\dfrac{
        \ptraj^{\pistar_\theta}(\tau)
    }{
        \ptraj^{\pistar_{\theta'}}(\tau)
    }}} = \abs*{ \beta^{-1} \int_0^1 \ip*{\theta - \theta', Z_{\phi}^{\theta' + \alpha(\theta - \theta')}(\tau)} \diff \alpha} \leq \beta^{-1} B_{A_{\phi}} \norm*{\theta - \theta'}, \quad \forall \theta, \theta'\in\R^d.
    \end{equation}
    Since $\cT_0$ has full measure under every soft-optimal trajectory
law, the claim holds $\bP^{\pistar}$-a.s.
\end{proof}

\subsection{Pseudo-Self-Concordance}
\begin{proposition}[Pseudo-Self-Concordance]\label{prop:self_concordance}
    For any $\theta, \xi,\zeta\in\R^d$, we have
    \begin{equation}
        \abs*{ D^3 \Jstar (\theta) [\xi, \xi, \zeta] } \leq \beta^{-1} B_{A_{\phi}} \norm*{\zeta} D^2 \Jstar (\theta) [\xi, \xi] .
    \end{equation}
\end{proposition}
\begin{proof}
    By Lemma~\ref{lem:param_der} and \eqref{eq:cumulative_advantage_bound}, we have
    \begin{equation}
        \abs*{ D^3 \Jstar (\theta) [\xi, \xi, \zeta] } = \beta^{-2}\abs*{\E^{\pi_\theta^{\star}}\bs*{\br*{Z^{\theta}_{\xi}}^2 Z^{\theta}_{\zeta}}} \leq \beta^{-1}B_{A_{\phi}}\norm*{\zeta} D^2 \Jstar (\theta) [\xi, \xi].
    \end{equation}
\end{proof}

Recall that $\mathsf{H}(\theta)=\nabla^2\Jstar(\theta)$. The above
pseudo-self-concordance property controls the relative variation of
the Hessian along line segments. The following result shows that this yields local equivalences between squared Hessian norms and Bregman divergences.

\begin{lemma}\label{lem:self_concordance_global_consequences}
    Suppose that Assumption~\ref{ass:full_rank} holds, so that
$\mathsf{H}(\theta)\succ0$ for every $\theta\in\R^d$. Fix $\theta_0, \theta_1 \in \R^d$, set $\Delta \defeq \theta_1 - \theta_0$ and
$\theta_\alpha \defeq \theta_0 + \alpha\Delta$ for $\alpha \in [0,1]$, and define
\begin{equation}
    S \defeq \beta^{-1} B_{A_\phi}\, \norm*{\Delta}.
\end{equation}
\begin{enumerate}
    \item For all $\alpha \in [0,1]$,
    \begin{equation}\label{eq:hessian_sandwich}
        e^{-\alpha S}\, \mathsf{H}(\theta_0) \;\preceq\; \mathsf{H}(\theta_\alpha) \;\preceq\; e^{\alpha S}\, \mathsf{H}(\theta_0).
    \end{equation}

    \item Let $\psi(x) \defeq (e^x - x - 1)/x^2$. Then
    \begin{equation}
        \psi(-S)\, \norm*{\Delta}_{\mathsf{H}(\theta_0)}^2
        \;\leq\; D_{\Jstar}(\theta_1, \theta_0)
        \;\leq\; \psi(S)\, \norm*{\Delta}_{\mathsf{H}(\theta_0)}^2.
    \end{equation}

    \item Let $\chi(x) \defeq (e^x - 1)/x$. Then
    \begin{align}
        \chi(-S)\, \norm*{\Delta}_{\mathsf{H}(\theta_0)}^2
        &\;\leq\; D_{\Jstar}(\theta_1, \theta_0)
        + D_{\Jstar}(\theta_0, \theta_1) \nonumber\\
        &= \ip*{\Delta, \, \nabla\Jstar(\theta_1) - \nabla\Jstar(\theta_0)}
        \;\leq\; \chi(S)\, \norm*{\Delta}_{\mathsf{H}(\theta_0)}^2.
    \end{align}
\end{enumerate}
Moreover, $\chi(-S) \ge (1+S)^{-1}$, so $\chi(-S)^{-1} \leq 1 + S$.
\end{lemma}

\begin{proof}
\emph{Part 1.} Fix $\xi \neq 0$ and set $g(\alpha) \defeq D^2\Jstar(\theta_\alpha)[\xi,\xi] = \xi^\top \mathsf{H}(\theta_\alpha)\xi$. Then $g'(\alpha) = D^3\Jstar(\theta_\alpha)[\xi,\xi,\Delta]$, and pseudo-self-concordance (Proposition~\ref{prop:self_concordance}) gives
\begin{equation}
    \abs*{\dfrac{\diff}{\diff\alpha}\log g(\alpha)}
    = \abs*{\dfrac{g'(\alpha)}{g(\alpha)}}
    \leq \beta^{-1} B_{A_\phi}\,\norm*{\Delta} = S.
\end{equation}
Integrating from $0$ to $\alpha$ yields $-\alpha S \leq \log(g(\alpha)/g(0)) \leq \alpha S$, hence \eqref{eq:hessian_sandwich} as $\xi$ was arbitrary.

\emph{Part 2.} By Taylor's theorem with integral remainder and \eqref{eq:hessian_sandwich},
\begin{align}
    \int_0^1 (1-\alpha)e^{-\alpha S}\diff\alpha\,
    \norm*{\Delta}_{\mathsf{H}(\theta_0)}^2
    &\leq D_{\Jstar}(\theta_1,\theta_0) \nonumber\\
    &= \int_0^1(1-\alpha)
    \norm*{\Delta}_{\mathsf{H}(\theta_\alpha)}^2\diff\alpha \nonumber\\
    &\leq \int_0^1(1-\alpha)e^{\alpha S}\diff\alpha\,
    \norm*{\Delta}_{\mathsf{H}(\theta_0)}^2.
\end{align}
The result then follows from
\begin{equation}
    \int_0^1 (1-\alpha) e^{-\alpha S} \diff \alpha = \psi(-S),  \quad \text{and} \quad  \int_0^1 (1-\alpha) e^{\alpha S} \diff \alpha = \psi(S).
\end{equation}

\emph{Part 3.} By the fundamental theorem of calculus,
\begin{equation}
    D_{\Jstar}(\theta_1, \theta_0) + D_{\Jstar}(\theta_0, \theta_1)
    = \ip*{\nabla\Jstar(\theta_1) - \nabla\Jstar(\theta_0),\, \Delta}
    = \int_0^1 \norm*{\Delta}_{\mathsf{H}(\theta_\alpha)}^2 \diff\alpha.
\end{equation}
Sandwiching using \eqref{eq:hessian_sandwich} and computing $\int_0^1 e^{\pm\alpha S}\diff\alpha = \chi(\pm S)$ gives the result.

\textit{Final inequality.} For $S \geq 0$, we have
\begin{equation}
    \chi(-S) = \dfrac{1 - e^{-S}}{S} = \dfrac{1}{S}\br*{1 - e^{-S}} \geq \dfrac{1}{S}\br*{1 - \dfrac{1}{1+S}} = \dfrac{1}{S}\dfrac{S}{1+S} = \dfrac{1}{1+S},
\end{equation}
where we used $e^{-S} \leq 1/(1+S)$.
\end{proof}

\begin{corollary}\label{cor:self_concordance}
    Suppose that Assumption~\ref{ass:full_rank} holds, so that
$\mathsf{H}(\theta)\succ0$ for every $\theta\in\R^d$.
Let $\theta_0, \theta_1 \in \R^d$, and set $\Delta = \theta_1 - \theta_0$ and $\mathsf{H}_0 \defeq \mathsf{H}(\theta_0)$. If
\begin{equation}
    \norm*{\Delta}_{\mathsf{H}_0} \leq \rho_0 \defeq \dfrac{\beta\sqrt{\lambda_{\min}(\mathsf{H}_0)}}{B_{A_\phi}},
\end{equation}
then, we have:
\begin{enumerate}[label=\arabic*)]
    \item (Density ratio bound)
    \begin{equation}
        \abs*{\log
    \br*{\dfrac{
        \ptraj^{\pistar_{\theta_0}}(\tau)
    }{
        \ptraj^{\pistar_{\theta_1}}(\tau)
    }}}
        \leq 1, \quad \bP^{\pistar}\text{-a.s.}
    \end{equation}
    \item (Hessian sandwich)
    \begin{equation}\label{eq:hessian_local_sandwich}
        e^{-1}\, \mathsf{H}_0 \;\preceq\; \mathsf{H}(\theta_1) \;\preceq\; e\, \mathsf{H}_0.
    \end{equation}

    \item (Bregman bounds)
    \begin{equation}\label{eq:bregman_local}
        e^{-1}\,\norm*{\Delta}_{\mathsf{H}_0}^2
        \;\leq\; D_{\Jstar}(\theta_1, \theta_0) = \beta \DKL(\bP^{\pistar_{\theta_0}}, \bP^{\pistar_{\theta_1}})
        \;\leq\; (e - 2)\,\norm*{\Delta}_{\mathsf{H}_0}^2.
    \end{equation}

    \item (Symmetric Bregman bounds)
    \begin{equation}\label{eq:sym_bregman_local}
        (1 - e^{-1})\,\norm*{\Delta}_{\mathsf{H}_0}^2
        \;\leq\; \ip*{\Delta, \, \nabla\Jstar(\theta_1) - \nabla\Jstar(\theta_0)}
        \;\leq\; (e - 1)\,\norm*{\Delta}_{\mathsf{H}_0}^2.
    \end{equation}

    \item (Hellinger-KL equivalence)
    \begin{equation}
        \DHel^2(\bP^{\pistar_{\theta_0}}, \bP^{\pistar_{\theta_1}}) \leq \DKL(\bP^{\pistar_{\theta_0}}, \bP^{\pistar_{\theta_1}}) \leq 3  \DHel^2(\bP^{\pistar_{\theta_0}}, \bP^{\pistar_{\theta_1}}).
    \end{equation}
\end{enumerate}
Consequently, we have the equivalences
\begin{equation}
    \DHel^2(\bP^{\pistar_{\theta_0}}, \bP^{\pistar_{\theta_1}}) \asymp \DKL(\bP^{\pistar_{\theta_0}}, \bP^{\pistar_{\theta_1}}) \asymp \DKL(\bP^{\pistar_{\theta_1}}, \bP^{\pistar_{\theta_0}}) \asymp \beta^{-1} \norm*{\Delta}_{\mathsf{H}_0}^2.
\end{equation}
\end{corollary}
\begin{proof}
    We have
    \begin{equation}
        \beta^{-1}B_{A_\phi}\norm{\Delta} \leq \beta^{-1} \lambdamin(\mathsf{H}_0)^{-1/2} B_{A_\phi}\norm{\Delta}_{\mathsf{H}_0} = \rho_0^{-1}\norm{\Delta}_{\mathsf{H}_0}
        \leq 1.
    \end{equation}

    Part~1 then follows from Proposition~\ref{prop:density_ratio_bound}. For Parts~2--4, apply Lemma~\ref{lem:self_concordance_global_consequences} with $S = \beta^{-1}B_{A_\phi}\norm{\Delta} \leq 1$, and use that $e^x, \chi(x), \psi(x)$ are increasing over $\R$. Finally, Part~5 uses Part~1 together with \citet[Lemma 5]{birge1998minimum}, which shows that $\DKL(P, Q)\leq (2 + \log B)\DHel^2(P, Q)$ if $P \leq B Q$.
\end{proof}

The following proposition will be helpful for the localization step in the fast rate proof.
\begin{proposition}\label{prop:increasing_function}    Let $\rho>0$. The function
    \begin{equation}
        f(x) = x \cdot \chi(-\rho^{-1}x) = \rho (1-e^{-\rho^{-1}x}),
    \end{equation}
    is a strictly increasing function from $[0, \infty)$ to $[0, \rho)$, and its inverse is given by
    \begin{equation}
        f^{-1}(y) = - \rho \log(1-\rho^{-1}y).
    \end{equation}
\end{proposition}
\begin{proof}
The claim follows from
$f'(x)=e^{-\rho^{-1}x}>0$, $f(0)=0$,
$\lim_{x\to\infty}f(x)=\rho$, and direct inversion.
\end{proof}

\subsection{Global Fast Rates}
\begin{theorem}\label{thm:global_fast_rate}
    Let $\beta>0$, let Assumptions~\ref{ass:linear_model} and~\ref{ass:full_rank} hold, and assume the bounded-advantage condition. Let $\thetastar$ and $\thetahat$ denote the population and empirical risk minimizers in \eqref{eq:irl_risk_minimizers} and define $\pi_{\star} = \pistar_{\thetastar}, \pihat \defeq \pistar_{\thetahat}$. Furthermore, let
    \begin{equation}
        \bar{\varepsilon}_n\br*{\delta} \defeq \dfrac{4 \dstar \log\br*{2\delta^{-1}}}{n} + \dfrac{32 B_{\phi}^2 \log^2\br*{2\delta^{-1}}}{\lambdastar n^2}.
    \end{equation}
    Then, with probability at least $1-\delta$,
    \begin{align}
        \DKL\br*{\bP^{\piE}, \bP^{\pihat}} &\leq \min_{\theta \in\Theta} \DKL\br*{\bP^{\piE}, \bP^{\pistar_{\theta}}} + \beta^{-1} \chi(-S)^{-1} \bar{\varepsilon}_n(\delta),\\
        \DKL\br*{\bP^{\pi_{\star}}, \bP^{\pihat}} + \DKL\br*{\bP^{\pihat}, \bP^{\pi_{\star}}}&\leq \beta^{-1} \chi(-S)^{-1} \bar{\varepsilon}_n(\delta),\\
        \norm*{\thetahat - \thetastar}_{\Hstar}^2 &\leq \chi(-S)^{-2} \bar{\varepsilon}_n(\delta),
    \end{align}
    where $S = 2 \beta^{-1} B_{A_{\phi}} B_{\theta}$ and $\chi$ is defined in Lemma~\ref{lem:self_concordance_global_consequences} and satisfies
    $\chi(-S)^{-1} \leq 1+S$.
\end{theorem}

The proof of Theorem~\ref{thm:global_fast_rate} leverages the
pseudo-self-concordance of the Min-Max-IRL loss
(Proposition~\ref{prop:self_concordance}). The argument follows
\citet{ostrovskii2021finite, liu2022confidence}, with modifications
for our setting: a bounded parameter set, and a vector Bernstein
inequality in place of sub-Gaussian concentration.

\begin{proof}[Proof of Theorem~\ref{thm:global_fast_rate}]
    \textit{Setup:}
    We denote the population and empirical risk as
    \begin{equation}
        L(\theta) \defeq \LIRL(\theta) = \Jstar(\theta) - \ip*{\theta, \phi\br*{\piE}}, \quad L_n(\theta) \defeq \LhatIRL(\theta) = \Jstar(\theta) - \ip*{\theta, \phihat\br*{\piE}}.
    \end{equation}
    Furthermore, define $\Delta^{\theta}_n \defeq \thetahat - \thetastar$ and $\Delta^{\phi}_n \defeq \phihat\br*{\piE} - \phi\br*{\piE}$, as well as,
    \begin{equation}
        \rho_n \defeq \norm*{\Delta^{\theta}_n}_{\Hstar}, \quad \eta_n \defeq \norm*{\Delta^{\phi}_n}_{\Hstar^{-1}}.
    \end{equation}
    We then have
    \begin{equation}
        L_n(\theta) = L(\theta) - \ip*{\theta, \Delta^{\phi}_n} \quad \text{and} \quad \nabla L_n(\theta) = \nabla L(\theta) - \Delta^{\phi}_n,
    \end{equation}
    and by optimality also
    \begin{equation}\label{eq:opt_cond_global_rate}
        \ip{\nabla L_n(\thetahat), \; \theta - \thetahat} \geq 0, \quad \ip{\nabla L(\thetastar), \; \theta - \thetastar} \geq 0, \quad \forall \theta \in \Theta.
    \end{equation}

    \textit{Step 1}: From the two first-order optimality conditions in Equation~\eqref{eq:opt_cond_global_rate}, it follows that
    \begin{align}\label{eq:global_rate_step1}
        \ip*{\Delta^{\theta}_n, \nabla \Jstar(\thetahat) - \nabla \Jstar(\thetastar)}
         & = \ip*{\Delta^{\theta}_n, \nabla L_n(\thetahat) + \phihat(\piE)} - \ip*{\Delta^{\theta}_n, \nabla L(\thetastar) + \phi(\piE)} \\
         & = \underbrace{\ip*{\Delta^{\theta}_n, \nabla L_n(\thetahat) - \nabla L(\thetastar)}}_{\leq 0} + \ip*{\Delta^{\theta}_n, \Delta^{\phi}_n} \leq \rho_n \eta_n.
    \end{align}
    Furthermore, define
    \begin{equation}
        S_n
        \defeq
        \beta^{-1}B_{A_\phi}\norm*{\Delta_n^\theta}.
    \end{equation}
    Since $\thetahat,\thetastar\in\Theta$, we have $S_n
        \leq S \defeq
        2\beta^{-1}B_{A_\phi}B_\theta$.
    Applying Lemma~\ref{lem:self_concordance_global_consequences}
    with $\theta_0=\thetastar$ and $\theta_1=\thetahat$ gives
    \begin{equation}
        \rho_n^2\chi(-S_n)
        \leq
        \ip*{
            \Delta_n^\theta,
            \nabla\Jstar(\thetahat)-\nabla\Jstar(\thetastar)
        }.
    \end{equation}
    Since $x\mapsto\chi(-x)$ is decreasing and
    $S_n\leq S$, it follows that
    \begin{equation}\label{eq:global_rate_step2}
        \rho_n^2\chi(-S)
        \leq
        \ip*{
            \Delta_n^\theta,
            \nabla\Jstar(\thetahat)-\nabla\Jstar(\thetastar)
        }.
    \end{equation}
    Hence, combining \eqref{eq:global_rate_step1} and \eqref{eq:global_rate_step2} it follows that
    \begin{equation}\label{eq:global_rate_step3}
        \rho_n \leq \chi(-S)^{-1} \eta_n.
    \end{equation}

    \textit{Step 2}:
    For the excess risk, we have by optimality of $\thetahat$ and \eqref{eq:global_rate_step3} that
    \begin{align}
        L(\thetahat) - L(\thetastar) &= \ip*{\thetahat, \Delta^{\phi}_n} + L_n(\thetahat) - \ip*{\thetastar, \Delta^{\phi}_n} - L_n(\thetastar)\\
        &= \ip*{\Delta^{\theta}_n, \Delta^{\phi}_n} + L_n(\thetahat) - L_n(\thetastar)\\
        &\leq \ip*{\Delta^{\theta}_n, \Delta^{\phi}_n} \leq \chi(-S)^{-1} \eta_n^2.\label{eq:global_excess_risk_bound_eta}
    \end{align}
    By Corollary~\ref{cor:soft_suboptimality}, we have $L(\theta) = \beta \br*{\DKL\br*{\bP^{\piE}, \bP^{\pistar_{\theta}}} + H(\piE)}$, so the above excess risk bound implies the KL bound
    \begin{equation}
        \DKL\br*{\bP^{\piE}, \bP^{\pihat}} \leq \DKL\br*{\bP^{\piE}, \bP^{\pi_{\star}}} + \beta^{-1} \chi(-S)^{-1} \eta_n^2,
    \end{equation}
    where in case that $H(\piE) = -\infty$, both sides equal $+\infty$ and the inequality holds trivially.
    Furthermore, we have also by Corollary~\ref{cor:Bregman} that
    \begin{align}
        \beta \br*{\DKL\br*{\bP^{\pi_{\star}}, \bP^{\pihat}}
        + \DKL\br*{\bP^{\pihat}, \bP^{\pi_{\star}}}}
        &= D_{\Jstar}(\thetahat, \thetastar)
        + D_{\Jstar}(\thetastar, \thetahat) \nonumber\\
        &= \ip*{\Delta^{\theta}_n,
        \nabla \Jstar(\thetahat) - \nabla \Jstar(\thetastar)}.
    \end{align}
    Consequently, Equation~\eqref{eq:global_rate_step1} implies that
    \begin{equation}
        \DKL\br*{\bP^{\pi_{\star}}, \bP^{\pihat}} + \DKL\br*{\bP^{\pihat}, \bP^{\pi_{\star}}} \leq \beta^{-1} \rho_n \eta_n \leq \beta^{-1} \chi(-S)^{-1} \eta_n^2.
    \end{equation}
    Lastly, for $\rho_n^2$, Equation~\eqref{eq:global_rate_step3} yields
    \begin{equation}
        \rho_n^2 \leq \chi(-S)^{-2} \eta_n^2.
    \end{equation}

    \textit{Step 3 (Concentration):} It remains to show that $\eta_n^2 \leq \bar{\varepsilon}_n(\delta)$ with high probability.
    By Proposition~\ref{prop:concentration}, with probability at least \(1-\delta\),
    \begin{equation}\label{eq:global_eta_concentration}
        \eta_n
        \leq
        \sqrt{\frac{2d_{\star}\log(2\delta^{-1})}{n}}
        +
        \frac{4B_{\phi}\log(2\delta^{-1})}{\sqrt{\lambda_{\star}}\,n}.
    \end{equation}
    Using \((a+b)^2 \leq 2a^2+2b^2\), this implies
    \begin{equation}
        \eta_n^2
        \leq \bar{\varepsilon}_n(\delta) =
        \frac{4d_{\star}\log(2\delta^{-1})}{n}
        +
        \frac{32B_{\phi}^2\log^2(2\delta^{-1})}{\lambda_{\star}n^2},
    \end{equation}
    with probability at least $1-\delta$, concluding the proof.
\end{proof}
\begin{remark}[Identifiability and quotient-space formulation]
\label{rem:identifiability_quotient}
Assumption~\ref{ass:full_rank} is used to bound
\begin{equation}
    \ip*{\Delta_n^\theta,\Delta_n^\phi}
    \leq
    \norm*{\Delta_n^\theta}_{\Hstar}
    \norm*{\Delta_n^\phi}_{\Hstar^{-1}}.
\end{equation}
For a positive semidefinite matrix $A$,
\begin{equation}
    A^{1/2}(A^\dagger)^{1/2}
    =
    \proj_{\ima A}.
\end{equation}
Hence, if either $x$ or $y$ lies in $\ima A$,
\begin{equation}
    \ip*{x,y}
    \leq
    \norm*{x}_A\norm*{y}_{A^\dagger}.
\end{equation}

Thus, the full rank assumption can be avoided by projecting onto the identifiable
subspace $\ima \Hstar$. In particular, let $\Pi\defeq\proj_{\ima\Hstar}$ and consider the loss
\begin{equation}
    \ell^{\mathsf{MM}}_{\Pi}(\theta;\tau)\defeq\ell^{\mathsf{MM}}(\Pi\theta;\tau).
\end{equation}
Then, the above proof applies with $\Hstar^{-1}$ replaced by
$\Hstar^\dagger$, $\dstar=\tr(\SigmaE\Hstar^\dagger)$, and
$\lambdastar$ replaced by the smallest positive eigenvalue of
$\Hstar$. Since $\ima\Hstar$ is isomorphic to the quotient space
$\R^d/\ker\Hstar$, this can be seen equivalently as an optimization problem over
$\R^d/\ker\Hstar$. The projection $\Pi\theta$ ensures that all members of
the same equivalence class are assigned the same loss.
\end{remark}

\subsection{Localized Fast Rates}\label{app:sec:proof_local_fast_rates}
\irlfastrate*
The proof of Theorem~\ref{thm:irl_fast_rate} follows from Theorem~\ref{thm:global_fast_rate} and an additional localization step.
\begin{proof}[Proof of Theorem~\ref{thm:irl_fast_rate}]
    Consider the same setup and definitions as in the proof of Theorem~\ref{thm:global_fast_rate}. Let $\rho_{\star} = \tfrac{\beta \sqrt{\lambda_{\star}}}{B_{A_{\phi}}}$ and define the event
    \begin{equation}
        E \defeq \bc*{\eta_n \leq \rho_{\star}\br*{1-e^{-1}}}.
    \end{equation}
    Applying Lemma~\ref{lem:self_concordance_global_consequences} with $S = \beta^{-1} B_{A_{\phi}} \norm{\Delta^{\theta}_n} \leq B_{A_{\phi}} \br*{\beta \sqrt{\lambdastar}}^{-1}  \rho_n = \rho_n / \rho_{\star}$, it follows that
    \begin{equation}
        \rho_n^2 \chi(-\rho_n / \rho_{\star}) \leq \rho_n^2 \chi(-S) \leq \ip*{\Delta^{\theta}_n, \nabla \Jstar(\thetahat) - \nabla \Jstar(\thetastar)} \leq \ip*{\Delta^{\theta}_n, \Delta^{\phi}_n} \leq \rho_n \eta_n,
    \end{equation}
    where we used that $x\mapsto \chi(-x)$ is decreasing. Hence, by Proposition~\ref{prop:increasing_function},
    \begin{equation}
        \rho_n \leq - \rho_{\star} \log(1-\eta_n / \rho_{\star}), \quad \text{if } \eta_n < \rho_{\star}.
    \end{equation}
    In particular, on the event $E$ we obtain $\rho_n \leq \rho_{\star}$. As $x\mapsto \chi(-x)$ is decreasing, this implies that
    \begin{equation}
        \chi(-S) \geq \chi(- \rho_n / \rho_{\star}) \geq \chi(-1) = (1- e^{-1}).
    \end{equation}
    Define $c := 1-e^{-1}$. The event $E$ holds with probability at least $1-\delta$, if
    \begin{equation}
        \sqrt{\frac{2d_{\star}\log(\delta^{-1})}{n}}
        +
        \frac{4B_{\phi}\log(2\delta^{-1})}{\sqrt{\lambda_{\star}}\,n}
        \leq
        c \rho_{\star}.
    \end{equation}
    Let $M_{\star}\defeq\max\bc*{B_{A_{\phi}}d_{\star}/\beta,B_{\phi}}$.
    It suffices to choose
    \begin{align}\label{eq:n_condition_sufficient_Rstar}
        n
        &\geq
        \max\bc*{
            \frac{8 d_{\star}\log(2\delta^{-1})}{c^2  \rho_{\star}^2},
            \frac{8 B_{\phi}\log(2\delta^{-1})}{c \rho_{\star} \sqrt{\lambda_{\star}}}
        } \nonumber\\
        &= \max\bc*{
            \frac{8 B_{A_{\phi}}^2 d_{\star}\log(2\delta^{-1})}{c^2\beta^2 \lambda_{\star}},
            \frac{8 B_{A_{\phi}} B_{\phi} \log(2\delta^{-1})}{c\beta \lambda_{\star}}}
        \nonumber\\
        &= \Omega\br*{
            \frac{B_{A_{\phi}} \log(\delta^{-1})}{\beta \lambda_{\star}}
            M_{\star}}.
    \end{align}
    The parameter estimation bound in Part~1 and the excess KL risk bound in Part~2 then follow from Theorem~\ref{thm:global_fast_rate}, while the equivalences in Part~3 follow from Corollary~\ref{cor:self_concordance}.

\end{proof}

\subsection{Effective Dimension}

The preceding bounds are governed by the effective dimension
$\dstar=\tr(\SigmaE\Hstar^{-1})$. The following proposition shows how this quantity can be bounded.

\begin{proposition}
\label{prop:dstar_decomp}
For a policy $\pi$, define the vector-valued dynamics residual
componentwise by
\begin{equation}
    \left[
    \delta_{t,\phi}^{\pi,0}
    \right]_i
    \defeq
    \delta_{t,\phi_i}^{\pi,0},
\end{equation}
and let
\begin{equation}
    \Sigma_{\mathrm{act}}^\pi
    \defeq
    \sum_{t=1}^T
    \E^{\pi}
    \left[
    A_{t,\phi}^{\pi,0}
    \left(
    A_{t,\phi}^{\pi,0}
    \right)^\top
    \right],
    \qquad
    \Sigma_{\mathrm{dyn}}^\pi
    \defeq
    \sum_{t=0}^{T-1}
    \E^{\pi}
    \left[
    \delta_{t,\phi}^{\pi,0}
    \left(
    \delta_{t,\phi}^{\pi,0}
    \right)^\top
    \right].
\end{equation}
Then, for every policy $\pi$,
\begin{equation}\label{eq:sigma_decomp}
    \operatorname{Cov}_{\tau\sim\bP^{\pi}}
    \left[
    \phi(\tau)
    \right]
    =
    \Sigma_{\mathrm{act}}^\pi
    +
    \Sigma_{\mathrm{dyn}}^\pi.
\end{equation}

In particular, for
$\SigmaE=\operatorname{Cov}_{\tau\sim\bP^{\piE}}[\phi(\tau)]$,
\begin{equation}\label{eq:dstar_decomp}
    \dstar
    =
    \tr\left(
    \Sigma_{\mathrm{act}}^{\piE}\Hstar^{-1}
    \right)
    +
    \tr\left(
    \Sigma_{\mathrm{dyn}}^{\piE}\Hstar^{-1}
    \right).
\end{equation}
Consequently:
\begin{enumerate}
    \item If $\piE=\pistar_{\thetastar}$, then
    \begin{equation}
        \dstar
        =
        \beta d
        +
        \tr\left(
        \Sigma_{\mathrm{dyn}}^{\piE}\Hstar^{-1}
        \right).
    \end{equation}
    In particular, if $(\bP_t)_{t=0}^{T-1}$ are deterministic, then
    $\dstar=\beta d$.
    \item In general,
    \begin{equation}
        \dstar
        \leq
        \frac{B_\phi^2}{\lambdastar}.
    \end{equation}
\end{enumerate}
\end{proposition}

\begin{proof}
Applying Lemma~\ref{lem:return_decomp} with $\beta=0$ to each
coordinate reward $\phi_i$ and stacking the resulting identities gives
\begin{equation}
    \phi(\tau)-\phi(\pi)
    =
    \sum_{t=1}^T
    A_{t,\phi}^{\pi,0}(s_t,a_t)
    +
    \sum_{t=0}^{T-1}
    \delta_{t,\phi}^{\pi,0}(s_t,a_t,s_{t+1}).
\end{equation}
The action-advantage and dynamics-residual terms are pairwise
orthogonal martingale differences. Taking second moments therefore
gives \eqref{eq:sigma_decomp}. Setting $\pi=\piE$ and using linearity
of the trace yields \eqref{eq:dstar_decomp}.

If $\piE=\pistar_{\thetastar}$, Corollary~\ref{cor:fisher_hessian}
implies
$\Sigma_{\mathrm{act}}^{\piE}=\beta\Hstar$. Hence,
\begin{equation}
    \tr\left(
    \Sigma_{\mathrm{act}}^{\piE}\Hstar^{-1}
    \right)
    =
    \beta d.
\end{equation}
If $(\bP_t)_{t=0}^{T-1}$ are deterministic, then
$\Sigma_{\mathrm{dyn}}^{\piE}=0$, proving~1.

Finally, since $\norm*{\phi(\tau)}\leq B_\phi$, we have
$\tr(\SigmaE)\leq B_\phi^2$. Together with
$\Hstar\succeq\lambdastar I_d$, this gives
\begin{equation}
    \dstar
    =
    \tr\left(
    \SigmaE\Hstar^{-1}
    \right)
    \leq
    \lambdastar^{-1}\tr(\SigmaE)
    \leq
    \frac{B_\phi^2}{\lambdastar}.
\end{equation}
\end{proof}

\begin{remark}[Effective dimension of MLE-IRL]
\label{rem:mle_effective_dimension}
Recall from the discussion of classical asymptotic theory in
Section~\ref{sec:asymptotics} that the effective dimension associated
with a loss $\ell(\theta;\tau)$ is
\begin{equation}
    d_\star
    =
    \tr\left(\mathsf{H}_\star^{-1}\mathsf{G}_\star\right),
    \qquad
    \mathsf{H}_\star
    \defeq
    \nabla^2L(\thetastar),
    \qquad
    \mathsf{G}_\star
    \defeq
    \operatorname{Cov}_{\tau\sim\bP^{\piE}}
    \left[
        \nabla\ell(\thetastar;\tau)
    \right].
\end{equation}
For MLE-IRL,
\begin{equation}
    \ell^{\mathsf{MLE}}(\theta;\tau)
    \defeq
    -\log\ptraj^{\pistar_\theta}(\tau).
\end{equation}
By the population-risk equivalence in \eqref{eq:risk_equivalence} and
the score identity in Lemma~\ref{lem:param_der},
\begin{equation}
    \mathsf{H}_\star^{\mathsf{MLE}}
    =
    \beta^{-1}\Hstar
    =
    \beta^{-2}
    \E^{\pi_\star}
    \left[
        Z_\phi^{\thetastar}
        \left(Z_\phi^{\thetastar}\right)^\top
    \right],
    \qquad
    \mathsf{G}_\star^{\mathsf{MLE}}
    =
    \beta^{-2}
    \operatorname{Cov}_{\tau\sim\bP^{\piE}}
    \left[
        Z_\phi^{\thetastar}
    \right],
\end{equation}
where $\Hstar$ is the Hessian of the Min-Max-IRL loss. In the
well-specified setting, $\piE=\pi_\star$ and
$\E^{\pi_\star}[Z_\phi^{\thetastar}]=0$, so the information matrix
equality
\begin{equation}
    \mathsf{G}_\star^{\mathsf{MLE}}
    =
    \mathsf{H}_\star^{\mathsf{MLE}}
\end{equation}
holds, and hence
\begin{equation}
    d_\star^{\mathsf{MLE}}=d.
\end{equation}

In comparison, Proposition~\ref{prop:dstar_decomp} shows that, in the
well-specified setting,
\begin{equation}
    d_\star^{\mathsf{MM}}
    =
    \beta d
    +
    \tr\left(
        \Sigma_{\mathrm{dyn}}^{\piE}
        \Hstar^{-1}
    \right).
\end{equation}
Consequently,
\begin{equation}
    \beta^{-1}d_\star^{\mathsf{MM}}
    \geq
    d_\star^{\mathsf{MLE}}
    =
    d,
\end{equation}
with equality under deterministic dynamics. This agrees with
Theorem~\ref{thm:equivalences}, since in that case the MLE-IRL and
Min-Max-IRL losses agree up to the factor $\beta^{-1}$. The same
factor appears when converting Min-Max-IRL excess risk into KL
divergence in Theorem~\ref{thm:irl_fast_rate}.

Finally, unlike the Min-Max-IRL loss, the MLE-IRL sample Hessian
generally depends on $\tau$. Thus, under stochastic dynamics, an
analogous nonasymptotic parameter bound would additionally require
concentration of the empirical Hessian; see, for example,
\citet{ostrovskii2021finite,liu2022confidence}. We leave this extension
to future work.
\end{remark}

\subsection{Concentration}
\begin{proposition}\label{prop:concentration}
Assume that $\bP^{\piE}$-a.e. $\norm{\sum_{t=1}^T \phi_t} \leq B_{\phi}$. Let
\begin{equation}
    \SigmaE \defeq \operatorname{Cov}\br*{\sum_{t=1}^T \phi_t} \quad \text{and} \quad d_{\star} \defeq \tr\br*{\SigmaE \Hstar^{-1}}.
\end{equation}
Then, with probability at least $1-\delta$,
\begin{align}
    \norm*{\phihat(\piE)-\phi\br*{\piE}}_{\Hstar^{-1}} &\leq \sqrt{\frac{2d_{\star}\log\br*{2\delta^{-1}}}{n}} + \dfrac{4B_{\phi}\log\br*{2\delta^{-1}}}{\sqrt{\lambda_{\star} }n}.
\end{align}
\end{proposition}
\begin{proof}
    Let $X_i \defeq \Hstar^{-1/2}\br{\sum_{t=1}^T \phi_t(s_t^i, a_t^i) - \phi(\piE)}$. Then, we have
    \begin{equation}
        \norm*{\dfrac{1}{n}\sum_{i=1}^n X_i} = \norm*{\phihat(\piE)-\phi\br*{\piE}}_{\Hstar^{-1}},
    \end{equation}
    as well as
    \begin{equation}
        \E^{\piE}\bs*{\norm*{X_i}^2} = \tr\br*{ \Hstar^{-1} \SigmaE } = \dstar, \quad \text{and} \quad \norm*{X_i} \leq \dfrac{2B_{\phi}}{\sqrt{\lambdastar}}.
    \end{equation}
    The result then follows by Lemma~\ref{lem:vector_bernstein}.
\end{proof}

\begin{lemma}[Vector Bernstein inequality, \citealp{pinelis1986remarks}]\label{lem:vector_bernstein}
Let $X_1,\dots,X_n$ be independent, mean-zero random vectors in $\R^d$. Assume that for all $i$ we have $\mathbb{E} \norm*{X_i}^2\leq \sigma^2$ and $\norm*{X_i} \leq b$ almost surely.
Then, with probability at least $1-\delta$,
\begin{align}
    \norm*{\dfrac{1}{n}\sum_{i=1}^n X_i} &\leq \sqrt{\dfrac{2\sigma^2 \log\br*{2\delta^{-1}}}{n}} + \dfrac{2b\log\br*{2\delta^{-1}}}{n}.
\end{align}
\end{lemma}

\section{Minimax Lower Bound}\label{app:sec:lower_bound}
We now provide the proof of the minimax lower bound in Section~\ref{sec:lower_bound}. The proof is a classic application of Fano's method \citep[Section 15.3]{wainwright2019high}.
\irllowerbound*

\begin{proof}
    \textit{Step 1}: Let $\mathsf{H}_0 = \mathsf{H}(\theta_0)$ and define the Dikin ellipsoid
    \begin{equation}
        \Theta_{\rho}(\theta_0) = \bc*{\theta \in \R^d : \norm*{\theta - \theta_0}_{\mathsf{H}_0} \leq \rho},
    \end{equation}
    of radius $\rho>0$. Since $U$ is a neighborhood of $\theta_0$, choose $\rho_U>0$ such that $\Theta_{\rho_U}(\theta_0) \subseteq U$ and define $\overline{\rho} \defeq \min\bc{\rho_0, \rho_U}$, where $\rho_0 \defeq \beta \sqrt{\lambda_{\min}(\mathsf{H}_0)}/B_{A_\phi}$. For every $\rho \leq \overline{\rho}$, Corollary~\ref{cor:self_concordance} implies that, for all $\theta \in \Theta_{\rho}(\theta_0)$, we have
    \begin{equation}\label{eq:lb_pf_1}
        e^{-1} \mathsf{H}_0 \preceq \mathsf{H}(\theta), \quad \text{and} \quad \DKL(\bP^{\pistar_{\theta}}, \bP^{\pistar_{\theta_0}}) \leq \beta^{-1}(e-1) \norm*{\theta - \theta_0}_{\mathsf{H}_0}^2.
    \end{equation}
    Now, define
    \begin{equation}\label{eq:lb_pf_2}
        \delta_n^2 \defeq \dfrac{\beta \log 2}{1024(e-1)}\dfrac{d}{n}, \quad \rho_n \defeq 16 \delta_n.
    \end{equation}
    To ensure that both $\Theta_{\rho_n}(\theta_0)\subseteq U$ and \eqref{eq:lb_pf_1} hold within $\Theta_{\rho_n}(\theta_0)$, we require $\rho_n\leq \overline{\rho}$, which holds whenever
    \begin{equation}
        n \geq \dfrac{\beta \log 2}{4(e-1)}\dfrac{d}{\overline{\rho}^2}.
    \end{equation}

    \textit{Step 2}: We now construct a $2\delta_n$-packing of $\Theta_{\rho_n}(\theta_0)$ in the $\norm{\cdot}_{\mathsf{H}_0}$-norm. To this end, consider the reparametrization $\theta(u) \defeq \theta_0 + \mathsf{H}_0^{-1/2} u$ for $u \in \R^d$. Then, $\norm{\theta(u) - \theta_0}_{\mathsf{H}_0} = \norm{u}_2$ and $\Theta_{\rho_n}(\theta_0) = \theta(B^d_{\rho_n})$ for the Euclidean ball $B^d_{\rho_n}$ of radius $\rho_n$. By a standard volumetric packing argument (see \citealp[Lemma 5.7]{wainwright2019high}), there exists a subset $\bc*{u^1, \hdots, u^M} \subseteq B^d_{\rho_n}$ such that $\norm{u^i - u^j}_2\geq 2\delta_n$ for $i\neq j$ and $M \geq 8^d$. Define $\theta^j \defeq \theta(u^j)$ for $j=1, \hdots, M$. Then, each $\theta^j \in \Theta_{\rho_n}(\theta_0)\subseteq U$, and
    \begin{equation}\label{eq:sep_prop}
        \norm*{\theta^i - \theta^j}_{\mathsf{H}_0} \geq 2 \delta_n, \quad i\neq j.
    \end{equation}

    \textit{Step 3}:
    Now consider the following $M$-ary testing problem. Let
    $J$ be an index sampled uniformly from $\bc{1,\hdots, M}$ and, conditionally on $J=j$, let the dataset $Z = \DE$ be sampled from
    \begin{equation}
        \bP^j
        \defeq
        \br*{\bP^{\pistar_{\theta^j}}}^{\otimes n}.
    \end{equation}
    Let $Q_{Z,J}, Q_Z, Q_J$ denote the joint and marginal laws of $Z$ and $J$, respectively. Then,
    \begin{equation}
        Q_{Z \mid J=j} = \bP^j, \quad \text{and} \quad Q_{Z} = \overline{\bP}\defeq \tfrac{1}{M} \sum_{j=1}^M \bP^j.
    \end{equation}

    Let $\hat{J}$ be any possibly randomized decoder of $J$ from $Z$. By Fano's inequality \citep[Equation~15.31]{wainwright2019high},
    \begin{equation}\label{eq:lb_pf_3}
        \Pr \br*{\hat{J} \neq J} \geq 1 - \frac{I(J;Z)+\log 2}{\log M},
    \end{equation}
    where $I(J;Z) \defeq \DKL\br*{Q_{J,Z}, Q_J\otimes Q_Z}$ denotes the mutual information. By the chain rule of relative entropy,
    \begin{equation}
        I(J;Z) = \DKL\br*{Q_{J,Z}, Q_J\otimes Q_Z} = \E_{J\sim Q_J} \DKL\br*{Q_{Z \mid J}, Q_Z} = \dfrac{1}{M} \sum_{j=1}^M \DKL(\bP^j, \bar{\bP}).
    \end{equation}
    Define $\bP^0 \defeq \br*{\bP^{\pistar_{\theta_0}}}^{\otimes n}$; then the chain rule of Radon--Nikodym derivatives yields
    \begin{align}
        \dfrac{1}{M}\sum_{j=1}^M \DKL\br*{\bP^j, \bar{\bP}} &\stackrel{(i)}{=} \dfrac{1}{M}\sum_{j=1}^M  \DKL\br*{\bP^j, \bP^0} - \dfrac{1}{M}\sum_{j=1}^M\E_{\bP^j} \log \br*{ \dfrac{\diff \overline{\bP}}{\diff \bP^0} }  \\
        &= \dfrac{1}{M}\sum_{j=1}^M \DKL\br*{\bP^j, \bP^0} -  \DKL(\overline{\bP}, \bP^0) \\
        &\leq \dfrac{1}{M}\sum_{j=1}^M \DKL\br*{\bP^j, \bP^0}.
    \end{align}
    Using product additivity of KL, the local KL bound \eqref{eq:lb_pf_1}, and the definition of $\delta_n, \rho_n$ in \eqref{eq:lb_pf_2}, we have
    \begin{align}
        I(J; Z) &\leq \dfrac{n}{M}\sum_{j=1}^M \DKL\br*{\bP^{\pistar_{\theta^j}}, \bP^{\pistar_{\theta_0}}}\\
        &\stackrel{(ii)}{\leq} \dfrac{n (e-1)}{\beta M} \sum_{j=1}^M \norm*{\theta^j - \theta_0}_{\mathsf{H}_0}^2\\
        &\stackrel{(iii)}{\leq} \dfrac{n (e-1) }{\beta}\br*{16 \delta_n}^2 = \dfrac{\log 2}{4}d.
    \end{align}
    Since $M\geq 8^d$, it holds that $\log M \geq d \log 8 = 3d \log 2$, and combining with \eqref{eq:lb_pf_3},
    \begin{equation}\label{eq:lb_pf_4}
        \Pr \br*{\hat{J} \neq J} \geq 1 - \frac{d\log 2 / 4 + \log 2}{3d \log 2} = 1 - \dfrac{1}{12} - \dfrac{1}{3d} \geq \dfrac{7}{12} \geq \dfrac{1}{2}.
    \end{equation}

    \textit{Step 4}: We now convert this testing lower bound into an estimation lower bound. Let $\thetahat$ be any possibly randomized estimator. It induces the nearest-neighbor decoder
    \begin{equation}
        \hat{J} \in \argmin_{1\leq j \leq M} \; \norm*{\thetahat - \theta^j}_{\mathsf{H}_0}.
    \end{equation}
    By the separation property \eqref{eq:sep_prop}, if $J = j$ and
    \begin{equation}
        \norm*{\thetahat - \theta^j}_{\mathsf{H}_0} < \delta_n,
    \end{equation}
    then $\hat{J} = j$. Hence,
    \begin{equation}
        \bc*{\hat{J} \neq J} \subseteq \bc*{\norm*{\thetahat - \theta^J}^2_{\mathsf{H}_0} \geq \delta_n^2}.
    \end{equation}
    Using \eqref{eq:lb_pf_4}, we get
    \begin{equation}\label{eq:lb_pf_5}
        \dfrac{1}{M} \sum_{j=1}^M \Pr\nolimits_{\theta^j} \br*{\norm*{\thetahat - \theta^j}^2_{\mathsf{H}_0} \geq \delta_n^2} = \Pr \br*{\norm*{\thetahat - \theta^J}^2_{\mathsf{H}_0} \geq \delta_n^2} \geq \dfrac{1}{2}.
    \end{equation}
    Here, $\Pr$ denotes probability under the joint law of $(J,Z)$ and any internal randomness of
    $\thetahat$, while $\Pr\nolimits_{\theta^j}$ denotes the corresponding law
    conditioned on $J=j$. Inequality \eqref{eq:lb_pf_5} implies that there exists $j_{\star}\in \bc{1, \hdots, M}$ such that
    \begin{equation}
        \Pr\nolimits_{\theta^{j_{\star}}} \br*{\norm*{\thetahat - \theta^{j_{\star}}}^2_{\mathsf{H}_0}\geq \delta_n^2} \geq \dfrac{1}{2}.
    \end{equation}

    Since $\theta^{j_{\star}}\in U$, we get the local minimax lower bound
    \begin{equation}
        \sup_{\thetaE\in U} \; \Pr\nolimits_{\thetaE} \br*{\norm*{\thetahat - \thetaE}^2_{\mathsf{H}_0}\geq \delta_n^2} \geq \dfrac{1}{2}.
    \end{equation}
    Using the Hessian comparison \eqref{eq:lb_pf_1} and the definition of $\delta_n$ \eqref{eq:lb_pf_2}, we have
    \begin{equation}
        \sup_{\thetaE\in U} \; \Pr\nolimits_{\thetaE} \br*{\norm*{\thetahat - \thetaE}^2_{\mathsf{H}(\thetaE)}\geq e^{-1} \delta_n^2} \geq \dfrac{1}{2},
    \end{equation}
    with
    \begin{equation}
        e^{-1} \delta_n^2 = \beta \dfrac{\log 2}{1024 e (e-1)}\dfrac{d}{n}.
    \end{equation}
    Thus the theorem holds with
    \begin{equation}
        c = \dfrac{\log 2}{1024 e (e-1)}.
    \end{equation}
\end{proof}

\section{Comparison of Theorem~\ref{thm:irl_fast_rate} with MLE-based guarantees}
\label{app:sec:mle_comparison}
In light of the equivalence between Min-Max-IRL and MLE-IRL established in
Theorem~\ref{thm:equivalences}, we compare the fast-rate guarantee obtained in
Theorem~\ref{thm:irl_fast_rate} with a direct analysis of MLE-IRL via the analysis of MLE behavioral cloning by \citet{foster2024behavior}, and its misspecified extension by
\citet{rohatgi2025computational}. To this end, we first establish a general MLE guarantee, which slightly improves on the misspecification term in \citet[Theorem 4.2]{rohatgi2025computational}, and then apply it to the trajectory densities induced by the class of soft-optimal densities.

\subsection{General MLE Guarantee}
Throughout this section, $m$ denotes a fixed base measure on $\cZ$. We use uppercase and lowercase letters for probability measures and their
densities with respect to $m$, that is, $p=\diff P/\diff m$. We begin with two definitions and a concentration inequality.

\begin{definition}[R\'enyi divergence, \citealp{van2014renyi}]
\label{def:renyi}
Let $P$ and $Q$ be probability measures with densities $p$ and $q$ with
respect to a common dominating measure $m$. For
$\alpha\in(0,1)\cup(1,\infty)$, the R\'enyi divergence of order $\alpha$
from $P$ to $Q$ is
\begin{equation}
    D_\alpha(P,Q)
    \defeq
    \frac{1}{\alpha-1}
    \log
    \int p^\alpha q^{1-\alpha}\diff m.
\end{equation}
We set $D_\alpha(P,Q)=\infty$ for $\alpha>1$ whenever $P\not\ll Q$, and
define the limiting orders by
\begin{align}
    D_1(P,Q)
    &\defeq
    \lim_{\alpha\to1}D_\alpha(P,Q)
    =
    \DKL(P, Q),
    \\
    D_\infty(P,Q)
    &\defeq
    \lim_{\alpha\to\infty}D_\alpha(P,Q)
    =
    \log
    \esssup_Q
    \frac{\diff P}{\diff Q}.
\end{align}
The map $\alpha\mapsto D_\alpha(P,Q)$ is nondecreasing
\citep[Theorem~3]{van2014renyi}, and
\begin{equation}
    \DHel^2(P,Q)
    \leq
    D_{1/2}(P,Q).
\end{equation}
\end{definition}

\begin{definition}[One-sided log-covering number]
\label{def:log_cover}
Let $\cF$ be a class of probability densities with respect to $m$. For
$\varepsilon>0$, let $\Nlog(\cF,\varepsilon)$ denote the smallest
cardinality of a subset $\cF_\varepsilon\subseteq\cF$ for which there exists
an $m$-null set $N_\varepsilon$ such that, for every $p\in\cF$, there is
some $\widetilde p\in\cF_\varepsilon$ satisfying
\begin{equation}\label{eq:pointwise_cover}
    p(z)
    \leq
    e^\varepsilon\widetilde p(z),
    \qquad
    z\in\cZ\setminus N_\varepsilon.
\end{equation}
\end{definition}
Condition~\eqref{eq:pointwise_cover} implies
$D_\infty(P,\widetilde P)\leq\varepsilon$, but allows the inequality \eqref{eq:pointwise_cover} to be evaluated at
data-dependent densities such as $\hat p$.

\begin{lemma}[\citealp{foster2021statistical}]
\label{lem:exp_martingale}
Let $(X_i)_{i\in\bN}$ be adapted to a filtration
$(\cF_i)_{i\in\bN_0}$ and satisfy $X_i\in(-\infty,\infty]$. Then, for any
$\delta\in(0,1)$, with probability at least $1-\delta$, simultaneously for
all $n\in\bN$,
\begin{equation}
    \sum_{i=1}^n
    -\log
    \E\bs*{e^{-X_i}\mid\cF_{i-1}}
    \leq
    \sum_{i=1}^n X_i
    +
    \log\br*{\delta^{-1}}.
\end{equation}
\end{lemma}

The following result follows the proof strategy of
\citet{foster2024behavior} with an additional
R\'enyi-divergence argument to control the misspecification term.

\begin{theorem}[General MLE guarantee]
\label{thm:gen_mle_guarantee}
Let $\cF$ be a class of probability densities with respect to a measure
$m$ on $\cZ$, and let $Z^1,\hdots,Z^n$ be i.i.d.\ from a probability
measure $P_0$ with density $p_0$. Fix $p_\star\in\cF$, with corresponding
probability measure $P_\star$, and let $\hat p\in\cF$, with corresponding
probability measure $\hat P$, satisfy
\begin{equation}
    \hat L_n(\hat p)
    \leq
    \inf_{p\in\cF}\hat L_n(p)
    +
    \varepsilonopt,
    \qquad
    \hat L_n(p)
    \defeq
    -\frac{1}{n}
    \sum_{i=1}^n
    \log p(Z^i).
\end{equation}
Then, with probability at least $1-\delta$,
\begin{align}
    \DHel^2(P_0,\hat P)
    &\leq
    2
    \inf_{\eta>0}
    \bc*{
        D_{1+\eta}(P_0,P_\star)
        +
        \frac{\log\br*{2\delta^{-1}}}{\eta n}
    }
    \nonumber\\
    &\quad+
    \inf_{\varepsilon>0}
    \bc*{
        \frac{
            4\log\br*{
                2\Nlog(\cF,\varepsilon)\delta^{-1}
            }
        }{n}
        +
        4\varepsilon
    }
    +
    2\varepsilonopt.
    \label{eq:mle_renyi_envelope}
\end{align}
If, in addition,
\begin{equation}
    D_\infty(P_0,P_\star)
    =
    \log B
    <
    \infty,
\end{equation}
then, on the same event,
\begin{align}
    \DHel^2(P_0,\hat P)
    &\leq
    4\DKL(P_0, P_\star)
    +
    \frac{
        2(1+\log B)\log\br*{2\delta^{-1}}
    }{n}
    \nonumber\\
    &\quad+
    \inf_{\varepsilon>0}
    \bc*{
        \frac{
            4\log\br*{
                2\Nlog(\cF,\varepsilon)\delta^{-1}
            }
        }{n}
        +
        4\varepsilon
    }
    +
    2\varepsilonopt
    \label{eq:mle_kl_form}
    \\
    &\leq
    4(2+\log B)\DHel^2(P_0,P_\star)
    +
    \frac{
        2(1+\log B)\log\br*{2\delta^{-1}}
    }{n}
    \nonumber\\
    &\quad+
    \inf_{\varepsilon>0}
    \bc*{
        \frac{
            4\log\br*{
                2\Nlog(\cF,\varepsilon)\delta^{-1}
            }
        }{n}
        +
        4\varepsilon
    }
    +
    2\varepsilonopt.
    \label{eq:mle_hel_form}
\end{align}
\end{theorem}

\begin{proof}
We follow the argument of \citet{foster2024behavior}, but handle the misspecification term slightly differently from \citet[Theorem 4.2]{rohatgi2025computational}.

Fix $\varepsilon,\eta>0$ such that
$\Nlog(\cF,\varepsilon)<\infty$ and
$D_{1+\eta}(P_0,P_\star)<\infty$. Let
$\cF_\varepsilon$ be a minimal log-cover in the sense of
Definition~\ref{def:log_cover}, with common exceptional set
$N_\varepsilon$, and choose $\widetilde p\in\cF_\varepsilon$ such that
\begin{equation}
    \hat p
    \leq
    e^\varepsilon\widetilde p
    \qquad
    \text{on }
    \cZ\setminus N_\varepsilon.
\end{equation}
Since $P_0\ll m$, all sample points lie outside $N_\varepsilon$ almost
surely. Also, since $\cF_\varepsilon$ is finite and its elements,
$p_0$, and $p_\star$ are densities, they are finite at all sample points
almost surely. Moreover, $p_0(Z^i)>0$ almost surely, and
$D_{1+\eta}(P_0,P_\star)<\infty$ implies
$P_0\ll P_\star$ and hence $p_\star(Z^i)>0$ almost surely. We work
throughout on the intersection of these probability-one events. We then have
\begin{equation}
    \hat L_n(\widetilde p)
    \leq
    \hat L_n(\hat p)+\varepsilon
    \leq
    \hat L_n(p_\star)+\varepsilonopt+\varepsilon
    <
    \infty.
\end{equation}
Consequently, all likelihood ratios below are well defined and no expression of the form
$\infty-\infty$ occurs.

By monotonicity of the R\'enyi divergence,
\begin{equation}\label{eq:mle_cover_distance}
    \DHel^2(\hat P,\widetilde P)
    \leq
    D_\infty(\hat P,\widetilde P)
    \leq
    \varepsilon,
\end{equation}
and since $\DHel$ is a metric,
\begin{equation}\label{eq:mle_triangle}
    \DHel^2(P_0,\hat P)
    \leq
    2\DHel^2(P_0,\widetilde P)
    +
    2\varepsilon.
\end{equation}

Next, we apply Lemma~\ref{lem:exp_martingale} for each $p\in\cF_\varepsilon$, with
\begin{equation}
    X_i
    =
    \frac{1}{2}
    \log
    \frac{p_0(Z^i)}{p(Z^i)}.
\end{equation}
A union bound over $\cF_\varepsilon$ yields, with probability at least
$1-\delta/2$, for all $p\in\cF_\varepsilon$,
\begin{equation}\label{eq:mle_cover_concentration}
    D_{1/2}(P_0,P)
    \leq
    \hat L_n(p)-\hat L_n(p_0)
    +
    \frac{
        2\log\br*{
            2\Nlog(\cF,\varepsilon)\delta^{-1}
        }
    }{n}.
\end{equation}
Evaluating at $\widetilde p$ and using
$\DHel^2\leq D_{1/2}$ yields
\begin{equation}\label{eq:mle_cover_bound}
    \DHel^2(P_0,\widetilde P)
    \leq
    \hat L_n(\widetilde p)-\hat L_n(p_0)
    +
    \frac{
        2\log\br*{
            2\Nlog(\cF,\varepsilon)\delta^{-1}
        }
    }{n}.
\end{equation}

The covering relation and approximate optimality of $\hat p$ imply
\begin{equation}\label{eq:mle_empirical_comparison}
    \hat L_n(\widetilde p)-\hat L_n(p_0)
    \leq
    \varepsilon
    +
    \varepsilonopt
    +
    \hat L_n(p_\star)-\hat L_n(p_0).
\end{equation}
Up to this point our proof followed \citet[Theorem C.1]{foster2024behavior}. To control the misspecification, we apply
Lemma~\ref{lem:exp_martingale} once more with
\begin{equation}
    X_i
    =
    -\eta
    \log
    \frac{p_0(Z^i)}{p_\star(Z^i)}.
\end{equation}
With probability at least $1-\delta/2$,
\begin{align}
    \hat L_n(p_\star)-\hat L_n(p_0)
    &\leq
    \frac{1}{\eta}
    \log
    \E_{P_0}
    \bs*{
        \br*{\frac{p_0}{p_\star}}^\eta
    }
    +
    \frac{\log\br*{2\delta^{-1}}}{\eta n}
    \nonumber\\
    &=
    D_{1+\eta}(P_0,P_\star)
    +
    \frac{\log\br*{2\delta^{-1}}}{\eta n}.
    \label{eq:mle_comparator_bound}
\end{align}

Combining
\eqref{eq:mle_triangle},
\eqref{eq:mle_cover_bound},
\eqref{eq:mle_empirical_comparison}, and
\eqref{eq:mle_comparator_bound}, and taking a union bound, yields
\begin{align}
    \DHel^2(P_0,\hat P)
    &\leq
    2D_{1+\eta}(P_0,P_\star)
    +
    \frac{2\log\br*{2\delta^{-1}}}{\eta n}
    \nonumber\\
    &\quad+
    \frac{
        4\log\br*{
            2\Nlog(\cF,\varepsilon)\delta^{-1}
        }
    }{n}
    +
    4\varepsilon
    +
    2\varepsilonopt.
    \label{eq:mle_fixed_parameters}
\end{align}
Taking the infima yields \eqref{eq:mle_renyi_envelope}.

Now suppose that
$D_\infty(P_0,P_\star)=\log B<\infty$. Choosing
\begin{equation}
    \eta
    =
    \frac{1}{1+\log B}
\end{equation}
and applying Lemma~\ref{lem:renyi_kl_bound} below gives
\begin{equation}
    D_{1+\eta}(P_0,P_\star)
    \leq
    2\DKL(P_0, P_\star),
\end{equation}
which proves \eqref{eq:mle_kl_form}. Finally,
\eqref{eq:mle_hel_form} follows from \citet[Lemma~5]{birge1998minimum}, which shows that the R\'enyi-$\infty$ bound yields
\begin{equation}
    \DKL(P_0, P_\star)
    \leq
    (2+\log B)\DHel^2(P_0,P_\star).
\end{equation}
\end{proof}

\subsection{Application to MLE-IRL}

We now apply Theorem~\ref{thm:gen_mle_guarantee} to the trajectory
densities induced by the linear reward class.

\begin{restatable}[MLE-IRL guarantee]{corollary}{mlelinear}
\label{cor:mle_linear}
Let Assumptions~\ref{ass:linear_model} and
\ref{ass:full_rank} hold, and let
\begin{equation}
    \thetahat^{\mathsf{MLE}}
    \in
    \argmin_{\theta\in\Theta}
    \LhatMLE(\pstar_\theta),
    \qquad
    \pihat^{\mathsf{MLE}}
    \defeq
    \pistar_{\thetahat^{\mathsf{MLE}}}.
\end{equation}
Then, with probability at least $1-\delta$,
\begin{align}
    \DHel^2\br*{
        \bP^{\piE},
        \bP^{\pihat^{\mathsf{MLE}}}
    }
    &\lesssim
    \DKL\br*{
        \bP^{\piE},
        \bP^{\pistar_{\thetastar}}
    }
    \nonumber\\
    &\quad+
    \frac{
        \br*{
            1+
            D_\infty\br*{
                \bP^{\piE},
                \bP^{\pistar_{\thetastar}}
            }
        }
        \log\br*{\delta^{-1}}
    }{n}
    \nonumber\\
    &\quad+
    \frac{d}{n}
    \log\br*{
        e+
        \frac{
            B_\theta B_{A_\phi}n
        }{
            \beta d
        }
    }.
    \label{eq:mle_linear_bound}
\end{align}
\end{restatable}

\begin{proof}
Let $p_\theta$ denote the density of
$\bP^{\pistar_\theta}$ with respect to the common trajectory base measure,
and define
\begin{equation}
    \cF_\Theta
    \defeq
    \bc*{p_\theta:\theta\in\Theta}.
\end{equation}
Apply \eqref{eq:mle_kl_form} with
\begin{equation}
    P_0
    =
    \bP^{\piE},
    \qquad
    P_\star
    =
    \bP^{\pistar_{\thetastar}},
    \qquad
    \hat P
    =
    \bP^{\pihat^{\mathsf{MLE}}},
    \qquad
    \varepsilonopt
    =
    0.
\end{equation}

It remains to bound the log-covering number of $\cF_\Theta$.
By Proposition~\ref{prop:density_ratio_bound}, there exists a common
null set outside which, simultaneously for all
$\theta,\theta'\in\Theta$,
\begin{equation}
    p_\theta(\tau)
    \leq
    \exp\br*{
        \beta^{-1}B_{A_\phi}
        \norm*{\theta-\theta'}_2
    }
    p_{\theta'}(\tau).
\end{equation}
Consequently, every Euclidean
$\beta\varepsilon/B_{A_\phi}$-cover of $\Theta$ induces an
$\varepsilon$-log-cover of $\cF_\Theta$ with the same common exceptional
set. Since $\Theta$ is contained in a $d$-dimensional Euclidean ball of
radius $B_\theta$ \citep[Corollary~4.2.13]{vershynin2018high},
\begin{equation}
    \Nlog(\cF_\Theta,\varepsilon)
    \leq
    \br*{
        1+
        \frac{2B_\theta B_{A_\phi}}
        {\beta\varepsilon}
    }^d.
\end{equation}
Substituting this bound into \eqref{eq:mle_kl_form} and choosing
$\varepsilon=d/n$ yields
\begin{align}
    \DHel^2\br*{
        \bP^{\piE},
        \bP^{\pihat^{\mathsf{MLE}}}
    }
    &\lesssim
    \DKL\br*{
        \bP^{\piE},
        \bP^{\pistar_{\thetastar}}
    }
    \nonumber\\
    &\quad+
    \frac{
        \br*{
            1+
            D_\infty\br*{
                \bP^{\piE},
                \bP^{\pistar_{\thetastar}}
            }
        }
        \log\br*{\delta^{-1}}
    }{n}
    \nonumber\\
    &\quad+
    \frac{d}{n}
    \log\br*{
        e+
        \frac{B_\theta B_{A_\phi}n}{\beta d}
    }.
\end{align}
\end{proof}

Compared to Theorem~\ref{thm:irl_fast_rate}, Corollary~\ref{cor:mle_linear} controls only the squared Hellinger
distance, and under misspecification it also depends on
\begin{equation}
    D_\infty\br*{
        \bP^{\piE},
        \bP^{\pistar_{\thetastar}}
    },
\end{equation}
which can be infinite even when the corresponding KL divergence is finite. Conversely, Corollary~\ref{cor:mle_linear} holds for every
sample size and requires no burn-in. Moreover, the general MLE
guarantee in Theorem~\ref{thm:gen_mle_guarantee} does not require
linear rewards or convexity of the model class. In the well-specified deterministic setting, where MLE-IRL and Min-Max-IRL are equivalent, both results achieve the same $\widetilde{\cO}(d/n)$ rate for the squared Hellinger distance. 

\subsection{R\'enyi--KL Bound}

\begin{lemma}
\label{lem:renyi_kl_bound}
Suppose $D_\infty(P,Q)<\infty$ and let
\begin{equation}
    \eta
    \defeq
    \frac{1}{1+D_\infty(P,Q)}.
\end{equation}
Then,
\begin{equation}
    D_{1+\eta}(P,Q)
    \leq
    2\DKL(P, Q).
\end{equation}
\end{lemma}

\begin{proof}
Let
\begin{equation}
    B
    \defeq
    \exp\br*{D_\infty(P,Q)}
    =
    \esssup_Q
    \frac{\diff P}{\diff Q}.
\end{equation}
If $B=1$, then $P=Q$ and the result is immediate. Suppose therefore that
$B>1$, and set
\begin{equation}
    \eta
    =
    \frac{1}{1+\log B}.
\end{equation}
By \citet[Theorem~35(b)]{sason2016f},
\begin{equation}\label{eq:sason_thm35b}
    D_{1+\eta}(P,Q)
    \leq
    \frac{1}{\eta}
    \log\br*{
        1+
        \frac{\eta}{
            \kappa_{1+\eta}(B)
        }
        \DKL(P, Q)
    },
\end{equation}
where
\begin{equation}
    \kappa_\alpha(t)
    \defeq
    \frac{
        (\alpha-1)(t\log t-t+1)
    }{
        t^\alpha-1-\alpha(t-1)
    },
    \qquad
    t>1.
\end{equation}
Using $\log(1+x)\leq x$, it remains to show that
$\kappa_{1+\eta}(B)\geq1/2$.

Write $u\defeq\log B$ and $x\defeq u/(1+u)$. After substitution,
$\kappa_{1+\eta}(B)\geq1/2$ is equivalent to
\begin{equation}
    F(u)
    \defeq
    3ue^u+1
    -
    (1+u)e^{u(2+u)/(1+u)}
    \geq
    0.
\end{equation}
We have $F(0)=0$, and direct differentiation gives
\begin{equation}
    F'(u)
    =
    \frac{e^u}{
        (1+u)(1-x)^2
    }
    \br*{
        3-
        e^x(x^2-3x+3)
    }.
\end{equation}
To bound the final factor, define
\begin{equation}
    g(x)
    \defeq
    x+\log(x^2-3x+3)-\log 3.
\end{equation}
For $x\in[0,1)$,
\begin{equation}
    g'(x)
    =
    \frac{x(x-1)}{x^2-3x+3}
    \leq
    0,
\end{equation}
and $g(0)=0$. Hence
$e^x(x^2-3x+3)\leq3$, so $F'(u)\geq0$ and therefore $F(u)\geq0$.
Thus, $\kappa_{1+\eta}(B)\geq1/2$, and
\eqref{eq:sason_thm35b} yields
\begin{equation}
    D_{1+\eta}(P,Q)
    \leq
    2\DKL(P, Q).
\end{equation}
\end{proof}

\section{Performance Gap Bounds}
\label{app:sec:performance_gap_bounds}

Besides trajectory-level divergences, imitation quality is commonly
measured through the performance gap under an unknown test reward
$\rtest$,
\begin{equation}
    \gap(\pihat)
    \defeq
    \ip*{\rtest,\mu^{\piE}-\mu^{\pihat}}.
\end{equation}
If $\sum_{t=1}^T \rtest_t \in [0,B]$,
\citet{foster2024behavior} show that
\begin{equation}\label{eq:gap-hellinger}
    \gap(\pihat)
    \lesssim
    \sqrt{
        \V_{\mathsf{act}}^{\piE}
        \DHel^2\br*{\bP^{\piE},\bP^{\pihat}}
    }
    +
    \widetilde{\cO}(B)
    \DHel^2\br*{\bP^{\piE},\bP^{\pihat}},
\end{equation}
where
\begin{equation}
    \V_{\mathsf{act}}^{\piE}
    \defeq
    \sum_{t=1}^T
    \E^{\piE}
    \bs*{
        \br*{
            A_{t,\rtest}^{\piE,0}(s_t,a_t)
        }^2
    }.
\end{equation}
Thus, trajectory-level Hellinger or KL guarantees imply performance
guarantees. Under misspecification, such bounds may be
vacuous if the expert trajectory law cannot be approximated well by
the induced class of trajectory laws.

If $\rtest\in\cR$, the gap is controlled directly by the
integral probability metric (IPM)\footnote{For a function class $\cG$, the IPM between two
probability measures $P$ and $Q$ is usually defined as
$\sup_{f\in\cG} \abs{\ip*{f, P-Q}}$ \citep{muller1997integral}. For convenience, we define it here without absolute value and directly for vectors of occupancy measures. With the usual definition, if $\cR$ is symmetric, we have
$D_{\cR}\!\br*{\mu^{\pi},\mu^{\pi'}} =
D_{\cG_{\cR}}(\bP^{\pi},\bP^{\pi'})$ for $\cG_{\cR} = \bc*{\sum_{t=1}^T r_t(s_t,a_t) : r\in\cR}$.} induced by $\cR$,
\begin{equation}\label{eq:ipm}
    \gap(\pihat)
    \leq
    D_{\cR}\br*{\mu^{\piE},\mu^{\pihat}}
    \defeq
    \sup_{r\in\cR}
    \ip*{r,\mu^{\piE}-\mu^{\pihat}},
\end{equation}
without requiring a Hellinger guarantee. The max-min dual of \eqref{eq:min_max_irl} is, up to the regularization, exactly minimizing the empirical version of this IPM,
\begin{equation}\label{eq:max_min_irl}\tag{Max-Min-IRL}
    \min_{\pi}
        D_{\cR}\br*{\muhat_n^{\piE},\mu^\pi}
        -
        \beta H(\pi).
\end{equation}
The above formulation is classical in imitation learning and appears in many algorithms 
\citep{abbeel2004apprenticeship,syed2007game,
syed2008apprenticeship,swamy2021moments,shani2022online}. In particular, Sion's min-max theorem
\citep{sion1958general} ensures that for compact convex $\cR$, this
problem admits a saddle point, which for $\beta>0$ ensures that the recovered policy is equivalent to a policy recovered via \eqref{eq:min_max_irl}. We have the following guarantee.

\begin{restatable}{theorem}{irlguaranteeipm}
\label{thm:irl_guarantee_ipm}
Let $\beta\geq 0$, let $\rtest\in\cR$, and let $\pihat$ be a minimizer in \eqref{eq:max_min_irl}. Define
\begin{equation}
    \Delta_n(\cR)
    \defeq
    \sup_{r\in\cR}
    \abs*{
        \ip*{
            r,\muhat_n^{\piE}-\mu^{\piE}
        }
    },
    \qquad
    \Delta_H
    \defeq
    \sup_{\pi,\pi'}
    \br*{
        H(\pi)-H(\pi')
    }.
\end{equation}
Then,
\begin{equation}\label{eq:ipm_oracle}
    \gap(\pihat)
    \leq
    \beta\Delta_H
    +
    2\Delta_n(\cR).
\end{equation}
In particular, under Assumption~\ref{ass:linear_model}, with
probability at least $1-\delta$,
\begin{equation}\label{eq:ipm_linear}
    \gap(\pihat)
    \leq
    \beta\Delta_H
    +
    2B_\theta
    \left(
        \sqrt{
            \frac{
                2\tr(\SigmaE)
                \log\br*{2\delta^{-1}}
            }{n}
        }
        +
        \frac{
            4B_\phi
            \log\br*{2\delta^{-1}}
        }{n}
    \right).
\end{equation}
\end{restatable}

\begin{proof}
Since $\rtest\in\cR$ and $\piE$ is feasible,
\begin{align}
    \gap(\pihat)
    &\leq
    D_{\cR}\br*{\mu^{\piE},\mu^{\pihat}}
    \\
    &\leq
    D_{\cR}\br*{\muhat_n^{\piE},\mu^{\pihat}}
    +
    \Delta_n(\cR)
    \\
    &\leq
    D_{\cR}\br*{\muhat_n^{\piE},\mu^{\piE}}
    +
    \beta\br*{H(\pihat)-H(\piE)}
    +
    \Delta_n(\cR)
    \\
    &\leq
    2\Delta_n(\cR)+\beta\Delta_H,
\end{align}
where the third inequality follows from the optimality of $\pihat$ in
\eqref{eq:max_min_irl}.

Under Assumption~\ref{ass:linear_model},
\begin{equation}
    \Delta_n(\cR)
    =
    B_\theta
    \norm*{
        \phihat_n(\piE)-\phi(\piE)
    }_2.
\end{equation}
Applying the vector Bernstein inequality of
Lemma~\ref{lem:vector_bernstein} with variance
$\tr(\SigmaE)$ and almost-sure bound $2B_\phi$ yields
\eqref{eq:ipm_linear}.
\end{proof}

Thus, the Max-Min-IRL objective directly controls the performance gap
for every test reward in $\cR$. This guarantee is well known in the
unregularized setting
\citep{syed2007game,swamy2021moments,shani2022online}. The result above
extends it to entropy-regularized IRL up to the additional bias
$\beta\Delta_H$. If $\Delta_H<\infty$, choosing
$\beta=\cO(n^{-1/2})$ preserves the usual $\cO(n^{-1/2})$ performance-gap
rate, while setting $\beta=0$ removes this bias. The result
readily extends to other reward classes whenever
$\Delta_n(\cR)$ can be bounded with high probability. Similarly, we may restrict minimization to a policy class $\Pi$ by introducing the misspecification error
$\inf_{\pi\in\Pi}D_{\cR}(\mu^{\piE},\mu^\pi)$.

Because the bound requires only estimation of the induced IPM,
it can remain informative when the expert trajectory law is not learnable
in Hellinger distance. In such a setting, \citet{simchowitz2025pitfalls} show that dynamics-agnostic
offline algorithms returning smooth Markov policies with state-independent
stochasticity, such as BC, suffer exponential-in-horizon compounding
error in the worst case. In contrast, Theorem~\ref{thm:irl_guarantee_ipm} shows that IRL, which is dynamics-aware, does not suffer from such compounding as long as $\rtest\in\cR$ and $\Delta_n(\cR)$ can be controlled appropriately.

\section{From Min-Max-IRL to No-Regret Learning in Games}
\label{app:sec:computation}

We discuss how minimizing the Min-Max-IRL loss can be reduced to no-regret learning. Define the saddle-point objective
\begin{equation}
    \Lhat(r, \pi)
    \defeq
    \ip*{r, \mu^{\pi} - \muhat^{\piE}}
    +
    \beta H(\pi).
\end{equation}
Then $\LhatIRL(r) = \max_{\pi} \Lhat(r, \pi)$, so minimizing $\LhatIRL$ over a reward class $\cR$ amounts to solving the min-max game
\begin{equation}
    \min_{r \in \cR} \max_{\pi} \Lhat(r, \pi).
\end{equation}
For reward and policy sequences $r_1, \hdots, r_K$ and $\pi_1, \hdots, \pi_K$, define the external regrets of the reward and policy players by
\begin{align}
    \operatorname{Reg}^{\mathsf{r}}(K)
    &\defeq
    -
    \min_{r\in\cR}
    \sum_{k=1}^K
    \br*{
        \Lhat(r, \pi_k)
        -
        \Lhat(r_k, \pi_k)
    }
    =
    -
    \min_{r\in\cR}
    \sum_{k=1}^K
    \ip*{
        r-r_k,
        \mu^{\pi_k}-\muhat^{\piE}
    },
    \\
    \operatorname{Reg}^{\mathsf{\pi}}(K)
    &\defeq
    \max_{\pi}
    \sum_{k=1}^K
    \br*{
        \Lhat(r_k, \pi)
        -
        \Lhat(r_k, \pi_k)
    }
    =
    \max_{\pi}
    \sum_{k=1}^K
    \br*{
        J(r_k, \pi)
        -
        J(r_k, \pi_k)
    }.
\end{align}

\begin{proposition}
\label{prop:no_regret_reduction}
Suppose that $\cR$ is convex and define $\rbar = K^{-1}\sum_{k=1}^K r_k \in \cR$. Then,
\begin{equation}\label{eq:no_regret_reduction}
    \LhatIRL(\rbar)
    -
    \min_{r\in\cR}\LhatIRL(r)
    \;\leq\;
    \frac{
        \operatorname{Reg}^{\mathsf r}(K)
        +
        \operatorname{Reg}^{\mathsf \pi}(K)
    }{K}.
\end{equation}
\end{proposition}

\begin{proof}
\begin{align}
    &\LhatIRL(\rbar) - \min_{r\in\cR}\LhatIRL(r)\\
    =&
    \max_{\pi}
    \dfrac{1}{K}
    \sum_{k=1}^K
    \Lhat(r_k, \pi)
    -
    \min_{r\in\cR}
    \max_{\pi}
    \dfrac{1}{K}
    \sum_{k=1}^K
    \Lhat(r, \pi)
    \\
    \leq&
    \max_{\pi}
    \dfrac{1}{K}
    \sum_{k=1}^K
    \Lhat(r_k, \pi)
    -
    \min_{r\in\cR}
    \dfrac{1}{K}
    \sum_{k=1}^K
    \Lhat(r, \pi_k)
    \\
    =&
    \max_{\pi}
    \dfrac{1}{K}
    \sum_{k=1}^K
    \br*{
        \Lhat(r_k, \pi)
        -
        \Lhat(r_k, \pi_k)
    }
    -
    \min_{r\in\cR}
    \dfrac{1}{K}
    \sum_{k=1}^K
    \br*{
        \Lhat(r, \pi_k)
        -
        \Lhat(r_k, \pi_k)
    }
    \\
    =&
    \dfrac{1}{K}
    \br*{
        \operatorname{Reg}^{\mathsf{r}}(K)
        +
        \operatorname{Reg}^{\mathsf{\pi}}(K)
    }.
\end{align}
\end{proof}
The two regret terms can often be controlled using online learning methods. For the reward player, $\Lhat(r,\pi_k)$ is affine in $r$, so for our linearly parametrized and bounded reward class, projected online gradient descent yields $\operatorname{Reg}^{\mathsf r}(K)=\cO(K^{1/2})$. This follows from the classical regret bound of \citet{zinkevich2003online} and has been applied to IRL by \citet{schlaginhaufen2024towards}.

For the policy player, controlling $\operatorname{Reg}^{\mathsf \pi}(K)$ amounts to regularized online RL with adversarial rewards. One possibility is to use a probably approximately correct (PAC) RL oracle to approximately solve the regularized MDP for each reward $r_k$, since per-round near-optimality also controls external regret. This, however, introduces a nested RL loop. Ideally, we want sublinear regret guarantees for algorithms that perform only incremental policy updates at each iteration, in the spirit of gradient descent-ascent. Such guarantees have been obtained, for example, for mirror-descent policy optimization with optimistic exploration in finite-horizon tabular MDPs \citep{shani2022online}, and for optimistic regularized approximate dynamic programming in infinite-horizon discounted linear MDPs \citep{moulin2025optimistically}. Both obtain $\cO(K^{1/2})$ policy-regret guarantees and apply their algorithms to obtain guarantees for the policy, rather than the reward, in unregularized \eqref{eq:max_min_irl}, which admits an analogous regret decomposition. If both regret terms are of order $\cO(K^{1/2})$, Proposition~\ref{prop:no_regret_reduction} yields an optimization error of order $\cO(K^{-1/2})$.

\end{document}